\documentclass[11pt,twoside]{article} 

\usepackage{amsmath,amsfonts,bm}

\def\eqref#1{equation~\ref{#1}}

\def\1{\bm{1}}

\DeclareMathAlphabet{\mathsfit}{\encodingdefault}{\sfdefault}{m}{sl}
\SetMathAlphabet{\mathsfit}{bold}{\encodingdefault}{\sfdefault}{bx}{n}

\usepackage[utf8]{inputenc} 
\usepackage[T1]{fontenc}    
\usepackage{booktabs}       
\usepackage{amsfonts}       
\usepackage{nicefrac}       
\usepackage{microtype}      

\usepackage{subfigure}
\usepackage{epsf}
\usepackage{epsfig}
\usepackage{fancyhdr}
\usepackage{graphics}
\usepackage{graphicx}
\usepackage{psfrag}
\usepackage{fullpage}
\usepackage{pdfpages}
\usepackage{natbib}
\usepackage{url}
\usepackage[colorlinks,linkcolor=magenta,citecolor=blue, pagebackref=true]{hyperref}
\renewcommand*{\backrefalt}[4]{%
    \ifcase #1  {(Not cited.)}%
    \or         {(Cited on page~#2.)}%
    \else       {(Cited on pages~#2.)}%
    \fi}
\usepackage{color}

\usepackage{amsthm}
\usepackage{amsmath}
\usepackage{amssymb,bbm}
\usepackage{caption}
\usepackage{textcomp}
\usepackage{siunitx}
\usepackage{wrapfig}
\usepackage{multirow}
\usepackage{multicol}

\usepackage{caption}
\usepackage{algorithm}
\usepackage{amssymb}
\usepackage{graphicx}                        
\usepackage{tikz}
\definecolor{softblue}{RGB}{180,185,230}
\definecolor{softgreen}{RGB}{190,235,190}
\definecolor{softred}{RGB}{240,180,170}
\definecolor{boxgray}{RGB}{235,235,235}
\usetikzlibrary{positioning, arrows.meta}    
\usepackage{booktabs} 
\usepackage{multirow}  
\usepackage{graphicx}  
\usetikzlibrary{fit}

\newtheorem{remark}{Remark}
\usepackage{bbm}
\usepackage{algorithmic}
\usepackage{algorithm}
\usepackage{flafter}

\newtheorem{assumption}{Assumption}

\usepackage{amssymb}
\usepackage{mathtools}
\newtheorem{example}{Example} 
\newtheorem{theorem}{Theorem}
\newtheorem{lemma}{Lemma} 
\newtheorem{proposition}{Proposition} 
\newtheorem{corollary}{Corollary}
\newtheorem{definition}{Definition}

\usepackage[colorinlistoftodos]{todonotes}

\begin{document}

\begin{center}

{\bf{\LARGE{Sliced Orlicz-Wasserstein
}}}
  
\vspace*{.2in}
{\large{
\begin{tabular}{cc}
Binh Thuan Tran$^\diamond$&Khai Nguyen$^\dagger$
\end{tabular}
}}

\vspace*{.2in}

\begin{tabular}{c}
$^\diamond$LAMA, LIGM, Université Gustave Eiffel \\
$^\dagger$Department of Statistics and Data Science, Texas A\&M University
\end{tabular}

\vspace*{.2in}
\today

\vspace*{.2in}

\begin{abstract}
We propose sliced Orlicz-Wasserstein (SOW) distance which is a generalization of sliced Wasserstein (SW) distance. SOW replaces the $L^p$ norm in SW with a Luxemburg norm cost induced by an Orlicz function $\phi$. First, we prove that SOW distance is a metric on the space of measures with finite Orlicz norm, and show that it recovers the SW distance when the Orlicz function is $\phi(x)=x^p$. Next, we derive the topological properties of the SOW distance. In particular, we show that convergence under SOW implies weak convergence, and the converse is true under the compact support condition. We then present the theoretical results for estimating the SOW distance. We derive sample complexity for both the distance itself and the powered functional of the distance, and prove their minimax optimality. In addition, we discuss the computational algorithm for approximating the SOW distance by Monte-Carlo estimation and bisection search, as well as the associated approximation error and computational complexity analysis. Our experimental results reveal the superior computational efficiency of SOW compared with Orlicz-Wasserstein (OW) distance. Also, in the experiments, we demonstrate the favorable flexibility of SOW distance over SW in detecting differences between distributions by comparing their performance in two-sample tests and evaluating generative models on image datasets\footnote{Code for this paper is published at  \url{https://github.com/khainb/SOW}.}.
\end{abstract}

\end{center}

\section{Introduction}
\label{sec:introduction}

Comparing probability measures is the most fundamental problem in statistics, machine learning, and data sciences. Among statistical distances, the Wasserstein distance~\cite{villani2003topics,villani2008optimal} has recently attracted a lot of attention due to its rich geometrical structure from optimal transport (OT). Within generative modeling, it has served to improve generative adversarial networks~\cite{arjovsky2017wasserstein}, flow-based approaches~\cite{lipman2023flow}, and drifting models~\cite{he2026sinkhorn}. It has likewise been employed for matching source and target distributions in domain adaptation~\cite{courty2017joint}. It also plays a significant role in computational biology~\cite{schiebinger2019optimal}, image processing~\cite{feydy2017optimal}, signal processing~\cite{kolouri2017optimal}, computer graphics~\cite{solomon2016entropic,solomon2015convolutional}, statistical inference~\cite{bernton2019approximate,bernton2019parameter}, and the measurement of dependence~\cite{catalano2021measuring,catalano2024wasserstein}, among numerous other domains.

Wasserstein distance of order $p$ aggregates transport cost by averaging the $p$-th power of the ground distance. Therefore, a single exponent controls how much weight is placed on small versus large displacements. Replacing this power function by an arbitrary Orlicz function (see Section~\ref{sec:background}) yields Orlicz-Wasserstein (OW) distance~\cite{sturm2011generalized,kell2017interpolation}, which measures the cost of a coupling through a Luxemburg-type norm. Since the power function is itself an Orlicz function, this family strictly generalizes Wasserstein distance. The choice of the Orlicz function controls how transported mass is weighted through the associated Luxemburg norm. In particular, as illustrated by \eqref{eq:Wphi-Fundamnetalfunction}, different growth regimes of $\phi$ induce different sensitivities to discrepancies carried by a small fraction of mass. This flexibility has already led to improved theoretical guarantees for parameter estimation in mixture models, where the mass a posteriori assigned to atoms far from the true support vanishes only logarithmically under Wasserstein distance, but at an almost polynomial rate under a suitable OW distance~\cite{guha2023excess}. Moreover, the work of \cite{altschuler2024faster} leverages $\operatorname{OW}$ to establish fast convergence of hypocoercive differential equations.

In this work, we focus on discussing the computational aspect of OW which is an underexplored area. As discussed in~\cite{guha2023excess}, computing OW requires using a bisection search where each iteration evaluates a Wasserstein distance at least once. Therefore, OW becomes even more computationally expensive than Wasserstein, which is already computationally demanding. For better context, the worst-case time complexity of Wasserstein distance is $\mathcal{O}(n^3 \log n)$~\cite{peyre2020computational} in the discrete case with at most $n$ atoms. \cite{guha2023excess} mitigate the problem by adapting entropic regularization~\cite{cuturi2013sinkhorn} to the case of OW, which reduces the computational complexity for each iteration to $\mathcal{O}(n^2)$.  However, this quadratic complexity is not scalable enough for a very large $n$. As a result, it is desirable to have a faster alternative.

We propose sliced Orlicz-Wasserstein (SOW) distance  which achieves a quasi-linear time complexity of $\mathcal{O}(n\log n)$ in the discrete setting. SOW is a sliced version of OW analogous distance to sliced Wasserstein~\cite{rabin2012wasserstein} (SW) distance. Similar to SW, SOW leverages the closed-form of one-dimensional OT \cite{cambanis1976inequalities} which accelerates computation of one-dimensional OW. In addition to the appealing computational complexity, SOW also enjoys a dimension-independent sample complexity as SW under the compact support assumption, in contrast to the curse of dimensionality problem of Wasserstein distance~\cite{fournier2015rate}. With the ability to sharpen control over tail behavior, SOW can be a suitable replacement for SW in settings where large displacements should be penalized more aggressively than the polynomial cost of SW allows, e.g., when the two measures differ mainly in their tails or in low-mass modes that SW is known to under-weight. More broadly, the choice of the Orlicz function makes SOW a family rather than a single metric. In particular, $\operatorname{SOW}$ recovers $\operatorname{SW}_p$ when $\phi(t)=t^p$, while more general choices of $\phi$ allow the sensitivity to small amounts of transported mass to depart from the fixed power-law behavior imposed by Wasserstein distances. SOW extends the family of sliced optimal transport metrics~\cite{nguyen2025introduction} and the family of sliced probability metrics~\cite{nadjahi2020statistical}. With existing broad applications of sliced optimal transport metrics, e.g., generative modeling~\cite{deshpande2018generative}, clustering~\cite{kolouri2018slicedgmm}, representation learning~\cite{kolouri2018sliced}, and many others, SOW can serve as a new choice in these applications. In summary, our contributions are threefold:


\begin{enumerate}

 \item  We establish a quantile representation for the one-dimensional $\operatorname{OW}$ distance and use it to introduce $\operatorname{SOW}$, obtained by aggregating projected one-dimensional $\operatorname{OW}$ distances. The proposed distance recovers $\operatorname{SW}$ as a special case.


 \item  Beyond basic metric properties, we develop the topological and statistical theory of $\operatorname{SOW}$. We show that $\operatorname{SOW}$ convergence implies weak convergence, with equivalence under compact support. Under the compact-support assumption, we derive dimension-independent sample complexities for both $\operatorname{SOW}$ and its powered functional, and prove that the resulting rates are minimax optimal. The powered-functional case requires additional perturbation and minimax lower-bound arguments beyond those used for standard $\operatorname{SW}$. We provide a proof roadmap in Appendix~\ref{app:technicalcontribution}, highlighting where the classical SW arguments break down and the additional ingredients needed in the general Orlicz setting. Besides, we also take the first step beyond compact support by identifying a necessary tail integrability condition for empirical $\operatorname{SOW}$ convergence.


 \item  We provide a practical approximation scheme based on Monte Carlo projections and bisection, together with approximation-error and computational-complexity guarantees. Experiments demonstrate substantially improved scalability over $\operatorname{OW}$ and show that suitable choices of the Orlicz function can increase sensitivity to low-mass or large-displacement discrepancies in two-sample testing and generative model evaluation.
\end{enumerate}

\textbf{Organization.} The remainder of the paper is organized as follows. We first review basic definitions of Wasserstein distance, SW distance, Orlicz space and OW distance, and fundamental functions in Section~\ref{sec:background}. We then discuss SOW, including its definitions, theoretical properties, and computations in Section~\ref{sec:SOW}.  Experiments and corresponding discussions are given in Section~\ref{sec:exps}. We conclude the paper and discuss some future works in Section~\ref{sec:conclusion}. Technical proofs are provided in the Appendices.

{\textbf{Notation.}} Throughout, $\|\cdot\|$ denotes the Euclidean norm. For $d\ge1$, we write
$\mathbb{S}^{d-1}\coloneqq
\left\{\theta\in\mathbb{R}^{d}:\|\theta\|=1\right\}$
for the unit sphere in $\mathbb{R}^{d}$ and let $\sigma$ denote the uniform probability measure on $\mathbb{S}^{d-1}$. Given $\theta\in\mathbb{S}^{d-1}$, let $\Pi^{\theta}:\mathbb{R}^{d}\to\mathbb{R}$ denote the projection $\Pi^{\theta}(x)=\theta^{\top}x$. For a measurable map $f$ and a measure $\mu$, the pushforward of $\mu$
through $f$ is denoted by $f_{\#}\mu\coloneqq\mu\circ f^{-1}$. For a Borel set $A$ in a Euclidean space, let $\mathcal{P}(A)$ denote the set of Borel probability measures on $A$. For $\mu\in\mathcal{P}(\mathbb{R})$, let $F_{\mu}$ denote its cumulative distribution function and for $u\in (0,1)$, define its quantile function by $F_{\mu}^{-1}(u)\coloneqq
\inf\left\{x\in\mathbb{R}:F_{\mu}(x)\ge u\right\}$. Weak convergence of $\mu_n$ to $\mu$ is denoted by
$\mu_n\rightharpoonup\mu$, and $\delta_x$ denotes the Dirac measure at $x$. For probability measures $\mu$ and $\nu$ on the same Euclidean space, $\Pi(\mu,\nu)$ denotes the set of their couplings. For two positive sequences $(a_n)_{n\ge1}$ and $(b_n)_{n\ge1}$, we write $a_n=\mathcal{O}(b_n)$, or equivalently $a_n\lesssim b_n$, if there exists a constant $C>0$, independent of $n$, such that $a_n\le Cb_n$ for every $n\ge1$. We write $a_n\asymp b_n$ if both $a_n\lesssim b_n$ and $b_n\lesssim a_n$. Subscripts indicate the parameters on which the implicit constants may depend. Increasing and decreasing are understood in the weak sense unless
explicitly qualified as strict.
\section{Background}
\label{sec:background}
We first review the Wasserstein and SW distances, and then recall the basic notions of Orlicz spaces and the OW distance used throughout the paper.

\textbf{Wasserstein distance.}
For $p\ge1$, let $\mathcal{P}_p(\mathbb{R}^d)$ denote the set of probability measures on $\mathbb{R}^d$ with finite $p$-th moment. The
$p$-Wasserstein distance~\cite{villani2008optimal} between $\mu,\nu\in\mathcal{P}_p(\mathbb{R}^d)$ is defined as follows:
\begin{equation}
\operatorname{W}_p^p(\mu,\nu)\coloneqq 
\inf_{\pi\in\Pi(\mu,\nu)}\int_{\mathbb{R}^d\times\mathbb{R}^d}\|x-y\|^p\mathrm{d}\pi(x,y).
\end{equation}

\textbf{SW distance.}
SW leverages the fact that 
the Wasserstein distance in one dimension admits an explicit quantile
representation. In particular, for $\mu,\nu\in\mathcal{P}_p(\mathbb{R})$, the
$p$-Wasserstein distance admits the quantile representation
$
\operatorname{W}_p^p(\mu,\nu)=\int_0^1\left|F_\mu^{-1}(u)-F_\nu^{-1}(u)\right|^p\mathrm{d}u
$ \cite[Theorem~2.10]{bobkov2019one}. For $\mu,\nu\in\mathcal P_p(\mathbb R^d)$, the
$p$-sliced Wasserstein distance~\cite{rabin2012wasserstein,bonneel2015sliced,nguyen2025introduction} is defined by
\begin{equation}
\operatorname{SW}^p_p(\mu,\nu)\coloneqq\int_{\mathbb S^{d-1}}\operatorname{W}_p^p\left(\Pi^\theta_{\#}\mu,\Pi^\theta_{\#}\nu \right)\mathrm d\sigma(\theta).
\end{equation}
Thus, SW compares high-dimensional probability measures through their one-dimensional linear projections, allowing each projected Wasserstein distance to be computed using the explicit one-dimensional quantile formula. In practice, the spherical integral is approximated using $L$ directions sampled independently from $\sigma$ which is the uniform distribution over $\mathbb{S}^{d-1}$.

\textbf{Orlicz spaces and OW distance.} We first recall Orlicz spaces and the Luxemburg norm. For further details on Orlicz spaces, we refer the reader to \cite{rao1991theory,rubshtein2016foundations}. We consider an Orlicz function\footnote{Definitions of Orlicz functions vary in the literature. For instance, \cite[Definition~13.1.1]{rubshtein2016foundations} allows Orlicz functions to take the value $+\infty$ and does not require strict monotonicity. Following the OW framework of \cite{sturm2011generalized}, we restrict attention to finite-valued, strictly increasing Orlicz functions.} $\phi:[0,+\infty)\to [0,+\infty)$ that is strictly increasing and convex, with $\phi(0)=0$. Given a measure space $\left(\Omega,m\right)$ and an Orlicz function $\phi$, the Orlicz space $L^{\phi}\left(\Omega,m\right)$ consists of measurable functions $f$ such that $\int_\Omega\phi(|f|/\lambda)\mathrm{d}m<\infty$ for some $\lambda>0$. It is equipped with the Luxemburg norm~\cite{luxemburg1955banach} defined by $\|f\|_{\phi}\coloneqq\inf\left\{\lambda>0:\int_\Omega\phi\left(\frac{|f|}{\lambda}\right)\mathrm dm\le 1\right\},$ where the underlying measure $m$ is omitted from the notation whenever it is clear from context. When $\phi(t)=t^p$ with $p\ge1$, one has
$L^\phi(\Omega,m)=L^p(\Omega,m)$, and the Luxemburg norm coincides with the usual $L^p$ norm~\cite[Section~13.4]{rubshtein2016foundations}. To define the corresponding transport distance, let
\begin{equation*}
\mathcal P_{\phi}(\mathbb R^d)\coloneqq\left\{
P\in\mathcal P(\mathbb R^d):\int_{\mathbb R^d}
\phi\left(\frac{\|x\|}{\lambda}\right)\mathrm dP(x)<\infty\text{ for some }\lambda>0
\right\}.
\end{equation*}
For $\mu,\nu\in \mathcal{P}_{\phi}\left(\mathbb{R}^d\right)$ and $\pi\in\Pi(\mu,\nu)$, the transport displacement is the function $(x,y)\mapsto\|x-y\|$ on the measure space $(\mathbb R^d\times\mathbb R^d,\pi)$. The OW
distance~\cite{sturm2011generalized,kell2017interpolation,guha2023excess} is obtained by minimizing the Luxemburg norm of this displacement function over all couplings:
\begin{equation}
\operatorname{W}_{\phi}(\mu,\nu)
\coloneqq\inf_{\pi\in\Pi(\mu,\nu)}
\inf\left\{\lambda>0:
\int_{\mathbb R^d\times\mathbb R^d}
\phi\left(\frac{\|x-y\|}{\lambda}\right)
\mathrm d\pi(x,y)
\le 1\right\}.
\end{equation}
When $\phi(t)=t^p$ for $p\ge1$, one has
$\mathcal P_{\phi}(\mathbb R^d)=\mathcal P_p(\mathbb R^d)$ and
$\operatorname{W}_{\phi}=\operatorname{W}_p$.

\textbf{Fundamental function.}
By \cite[Proposition~13.3.1]{rubshtein2016foundations}, when $m$ is a
probability measure, the fundamental function associated with the Luxemburg norm is given by
\begin{equation}\label{eq:fundamentalfunction}
\omega_\phi(u)\coloneqq\frac{1}{\phi^{-1}(1/u)},\qquad u\in(0,1], \qquad \omega_\phi(0)\coloneqq0.
\end{equation}
In the one-dimensional OW setting, this fundamental function admits a
direct transport interpretation. Indeed, for any $\Delta>0$ and
$u\in[0,1]$, a direct calculation gives
\begin{equation}\label{eq:Wphi-Fundamnetalfunction}
\operatorname{W}_{\phi}
\left(\delta_0,(1-u)\delta_0+u\delta_{\Delta}\right)=\Delta\,\omega_{\phi}(u).
\end{equation}
This identity separates mass and distance: $\omega_\phi(u)$ captures the effect of the transported mass $u$, while $\Delta$ represents the transport displacement.

\section{Sliced Orlicz-Wasserstein (SOW)}
\label{sec:SOW}
We now introduce $\operatorname{SOW}$ and establish its topological, statistical, and computational properties.

\subsection{Definition}
\label{subsec:defs}
 We begin by discussing the one-dimensional OW distance, which serves as the basic building block of the later sliced construction.
 
\textbf{One-dimensional OW distance.} Let $\mu,\nu\in\mathcal{P}_{\phi}(\mathbb{R})$. Since $\phi$ is strictly increasing and convex, for each $\lambda>0$, the cost function $(x,y)\mapsto\phi(|x-y|/\lambda)$ satisfies the conditions of \cite[Theorem~2]{cambanis1976inequalities}. Thus,
\begin{equation}
\inf_{\pi\in\Pi(\mu,\nu)}\int_{\mathbb{R}\times \mathbb{R}}
\phi\left(\frac{|x-y|}{\lambda}\right)\mathrm{d}\pi(x,y)=
\int_0^1\phi\left(\frac{|F_{\mu}^{-1}(u)-F_{\nu}^{-1}(u)|}{\lambda}\right)\mathrm{d}u.
\end{equation}
Therefore,
\begin{equation}\label{eq:OW-quantile}
\operatorname{W}_{\phi}(\mu,\nu) = \inf\left\{ \lambda>0: \int_0^1 \phi\left( \frac{|F_{\mu}^{-1}(u)-F_{\nu}^{-1}(u)|}{\lambda} \right) \mathrm{d}u \leq 1 \right\}. 
\end{equation}

Building on the one-dimensional quantile representation above, we
introduce the SOW distance.
\begin{definition}For $p\geq 1$ and $\mu,\nu\in\mathcal{P}_{\phi}(\mathbb{R}^{d})$,
we define the $\operatorname{SOW}$ distance by
\begin{equation}
\operatorname{SOW}_{\phi,p}^p(\mu,\nu)\coloneqq
\int_{\mathbb{S}^{d-1}}\operatorname{W}^p_{\phi}
\left(\Pi^{\theta}_{\#}\mu,
\Pi^{\theta}_{\#}\nu \right)
\mathrm{d}\sigma(\theta).
\end{equation}
\end{definition}
Since $\Pi^\theta_{\#}\mu$ and $\Pi^\theta_{\#}\nu$ belong to $\mathcal{P}_{\phi}(\mathbb{R})$ for every $\theta\in\mathbb{S}^{d-1}$, \eqref{eq:OW-quantile}
applies directly to each projected pair. The proposed distance recovers the classical
$\operatorname{SW}_p$  as a special case. Indeed, when $\phi(t)=t^p$ with $p\ge1$, we have
$\operatorname{W}_\phi=\operatorname{W}_p$, and therefore $\operatorname{SOW}_{\phi,p}=\operatorname{SW}_p$.

\subsection{{Topological Properties}} \label{sec:topological-properties}

We first establish the metric and comparison properties of the SOW
distance.
\begin{theorem}\label{thm:SOW-metric}
Let $\phi$ be an Orlicz function and let $p\geq1$. Then
$\operatorname{SOW}_{\phi,p}$ defines a metric on
$\mathcal{P}_{\phi}(\mathbb{R}^{d})$.
\end{theorem}
\begin{proposition}\label{prop:SOW-OW} Let $\phi$ be an Orlicz function and let $p\geq1$. Then 
\begin{equation*}
\operatorname{SOW}_{\phi,p}(\mu,\nu) \leq \operatorname{W}_{\phi}(\mu,\nu)  \quad \text{for all } \mu,\nu \in \mathcal{P}_{\phi}(\mathbb{R}^{d}) 
\end{equation*}
\end{proposition}
\begin{proposition}\label{prop:SOW-phi} Let $\phi_1$ and $\phi_2$ be Orlicz functions satisfying $\phi_1(t)\leq\phi_2(t)$ for every $t \ge 0$. Then, for every $p\geq1$ and every $\mu,\nu\in\mathcal{P}_{\phi_2}(\mathbb{R}^{d})$, 
$
\operatorname{SOW}_{\phi_1,p}(\mu,\nu) \leq \operatorname{SOW}_{\phi_2,p}(\mu,\nu).     
$
\end{proposition}
The proofs build on the arguments in \cite[Lemmas~3.3--3.4]{guha2023excess} and are given in Appendix~\ref{app:proof-SOW-metric-comparison}. We next study the convergence induced by $\operatorname{SOW}_{\phi,p}$. We first show that convergence in $\operatorname{SOW}_{\phi,p}$ is equivalent to uniform convergence of the projected OW distances. 
\begin{proposition} \label{prop:SOW-maxSOW} 
Let $p\in[1,\infty)$ and let $\phi$ be an Orlicz function. Let $(\mu_n)_{n\geq1}\subset\mathcal{P}_{\phi}(\mathbb{R}^{d})$, and let $\mu\in\mathcal{P}_{\phi}(\mathbb{R}^{d})$. Then $\operatorname{SOW}_{\phi,p}(\mu_n,\mu)\to0 \quad\Longleftrightarrow\quad \sup_{\theta\in\mathbb{S}^{d-1}} \operatorname{W}_{\phi} \left( \Pi^\theta_{\#}\mu_n, \Pi^\theta_{\#}\mu \right) \to0. $
\end{proposition}
The key point is that the projected OW distances vary continuously with the projection direction (see Lemma~\ref{lem:directional-lipschitz}). Moreover, SOW convergence yields a uniform bound on the Orlicz moments defined in \eqref{eq:momentorlicz}. Together, these properties prevent the projected distances from remaining large on arbitrarily small sets of directions. The complete proof is given in Appendix~\ref{app:proof-SOW-topo}. As a direct consequence, convergence in $\operatorname{SOW}_{\phi,p}$ implies weak convergence.
\begin{corollary} \label{cor:SOW-weak} 
Under the assumptions of Proposition~\ref{prop:SOW-maxSOW},
$\operatorname{SOW}_{\phi,p}(\mu_n,\mu)\to0 \quad\Longrightarrow\quad \mu_n \rightharpoonup \mu.$
\end{corollary}

\begin{remark}
The converse implication does not hold in general, even in dimension one.
Indeed, one can construct a sequence that converges weakly while its $\operatorname{OW}$ distance from the limit remains bounded away from zero. A concrete counterexample is given in Appendix~\ref{app:SOW-weak-counterexample}. However, the converse becomes valid under a common compact-support assumption. 
\end{remark}
\begin{proposition}\label{prop:SOW-compact-topology}
Let $p\in[1,\infty)$, let $\phi$ be an Orlicz function, and let $K\subset\mathbb{R}^{d}$ be compact. Suppose that $(\mu_n)_{n\geq1}$ and $\mu$ belong to
$\mathcal{P}(K)$. Then
$
\mu_n \rightharpoonup \mu
\quad\Longleftrightarrow\quad
\operatorname{SOW}_{\phi,p}(\mu_n,\mu)\to0.
$
\end{proposition}
The proof of this property is deferred to Appendix~\ref{app:proof-SOW-topo}.
\subsection{{Statistical Properties}}\label{sec:statistical-properties}
\label{subsec:stats}
We now study plug-in estimation of $\operatorname{SOW}_{\phi,p}$. We first establish a one-sample empirical convergence bound. We then study two related two-sample estimation problems: estimation of the distance $\operatorname{SOW}_{\phi,p}$ and estimation of the powered functional $\operatorname{SOW}_{\phi,p}^{p}$. Hereafter, $\phi$ denotes a fixed Orlicz function introduced in Section~\ref{sec:background}. For the statistical and minimax results, we also fix a compact set $K\subset\mathbb R^d$ with $D\coloneqq\operatorname{diam}(K)>0$. For $\mu,\nu \in \mathcal{P}(K)$, let $\widehat{\mu}_n$ and $\widehat{\nu}_m$ denote the empirical measures based on independent samples of size $n$ and $m$ from $\mu$ and $\nu$, respectively.
\begin{theorem}\label{thm:SOW-one-sample}
Let $p\in[1,\infty)$ and $\mu\in\mathcal{P}(K)$. Then
$$\left[\mathbb{E}
\operatorname{SOW}_{\phi,p}^{p}\left(
\widehat{\mu}_n,\mu\right)
\right]^{1/p}\lesssim_{p}D\,\omega_{\phi}\left(n^{-1/2}\right).
$$
\end{theorem}
The rate is governed by the fundamental function $\omega_{\phi}$ defined in \eqref{eq:fundamentalfunction} and has no explicit dependence on the ambient dimension $d$. The complete proof is deferred to Appendix~\ref{app:proof-SOW-statistics}.

We next consider the two-sample setting. The metric property of $\operatorname{SOW}_{\phi,p}$ and Theorem~\ref{thm:SOW-one-sample} yield the following bound. The complete proof of Corollary~\ref{cor:SOW-two-sample} is given in
Appendix~\ref{app:proof-SOW-statistics}.
\begin{corollary} \label{cor:SOW-two-sample}
Let $p\in[1,\infty)$ and let $\mu,\nu\in\mathcal{P}(K)$. Then
$$\mathbb{E} \left| \operatorname{SOW}_{\phi,p} \left( \widehat{\mu}_n,\widehat{\nu}_m \right) - \operatorname{SOW}_{\phi,p}(\mu,\nu) \right| \lesssim_p D \left[\omega_{\phi}\left(n^{-1/2}\right) + \omega_{\phi}\left(m^{-1/2}\right) \right].$$
\end{corollary}

We next turn to the powered functional
$\operatorname{SOW}_{\phi,p}^{p}$. Unlike the preceding distance bound, its analysis requires controlling the variation of the powered Luxemburg
norm under $L^1$ perturbations. To this end, we define
\begin{equation}\label{eq:psiphip}
\psi_{\phi,p}(u)\coloneqq \omega_\phi(u)^p
=\frac{1}{\phi^{-1}(1/u)^p}, \qquad u\in(0,1],
\end{equation}
with $\psi_{\phi,p}(0)\coloneqq0$. A direct calculation gives
$$\psi_{\phi,p}(u) = \operatorname{W}_\phi^p
\left(\delta_0,(1-u)\delta_0+u\delta_1\right), \quad u\in[0,1].
$$
Thus, $\psi_{\phi,p}(u)$ measures the powered OW response to transporting a mass fraction $u$ over a unit distance. Let $\psi_{\phi,p}^{\mathrm{cav}}$ denote the least concave majorant of $\psi_{\phi,p}$ on $[0,1]$, i.e., the smallest concave function on $[0,1]$ that dominates $\psi_{\phi,p}$. For the powered-functional estimation problem, we need to control how the powered Luxemburg norm changes when a quantile-gap profile is perturbed in $L^1$. To capture the worst-case variation, define
\begin{equation*}
\Omega_{\phi,p}(u) \coloneqq\sup_{\substack{0\le f,g\le1\\\|f-g\|_{L^1(0,1)}\le u}} \left|\|f\|_\phi^p-\|g\|_\phi^p\right|, \quad u \in [0,1].
\end{equation*}
Proposition~\ref{prop:luxemburg-L1-modulus} shows that $\Omega_{\phi,p}(u) \asymp_{p} \psi_{\phi,p}^{\mathrm{cav}}(u)$. Thus, $\psi_{\phi,p}^{\mathrm{cav}}$ gives the sharp perturbation scale underlying the following two-sample bound. The proof is provided in Appendix~\ref{app:proof-SOW-power-twosample}.

\begin{theorem}\label{thm:powerSOW-two-sample}
Let $p\in[1,\infty)$ and let $\mu,\nu\in\mathcal{P}(K)$. Then
$$\mathbb{E}\left|\operatorname{SOW}_{\phi,p}^{p}
\left(\widehat{\mu}_n,\widehat{\nu}_m
\right)-\operatorname{SOW}_{\phi,p}^{p}(\mu,\nu)
\right|\leq p2^pD^p\psi_{\phi,p}^{\mathrm{cav}}\left(\frac{1}{2}n^{-1/2}+\frac{1}{2}m^{-1/2}\right).
$$
\end{theorem}
\begin{remark}[Recovery of SW rates] \label{rem:connection-SW-rates}
When $\phi(t)=t^p$ with $p\ge1$, a direct computation yields
\begin{equation*}
\omega_\phi(u)=u^{1/p},\qquad\psi_{\phi,p}(u)=
\psi_{\phi,p}^{\mathrm{cav}}(u)=u,\qquad\operatorname{SOW}_{\phi,p}=\operatorname{SW}_p.   
\end{equation*}
Assuming $n\leq m$, Theorem~\ref{thm:SOW-one-sample} and
Corollary~\ref{cor:SOW-two-sample} give the rate $n^{-1/(2p)}$ for one-sample empirical convergence and two-sample estimation of $\operatorname{SW}_p$, respectively. These rates are consistent with those established for sliced Wasserstein distances in \cite{nadjahi2020statistical}. Moreover, Theorem~\ref{thm:powerSOW-two-sample} gives the rate $n^{-1/2}$ for the powered functional $\operatorname{SW}_p^p$, matching the parametric rate obtained under compact-support assumptions in \cite{nietert2022statistical}.
\end{remark}
\subsection{{Minimax Optimality}} \label{sec:minimax}
The preceding results provide upper bounds for the empirical plug-in estimators but do not establish their statistical optimality. We therefore study the corresponding minimax risks over $\mathcal{P}(K)$, taking the infimum over all estimators based on the observed samples, following the standard minimax framework of \cite[Chapter~2]{tsybakov2009introduction}. Matching lower bounds then show that the empirical plug-in estimators are minimax-rate optimal up to constant factors.

\begin{theorem}[Minimax optimality]
\label{thm:SOW-minimax}
Let $p\in[1,\infty)$. For one-sample distribution estimation,
$$\inf_{\widetilde{\mu}_n}\sup_{\mu\in\mathcal{P}(K)}
\left[\mathbb{E}\operatorname{SOW}^p_{\phi,p} \left(\widetilde{\mu}_n,\mu
\right)\right]^{1/p}\asymp_{d,p}D\omega_\phi\left(n^{-1/2}\right).
$$
For the two-sample problems, we have
$$\inf_{\widehat{T}}\sup_{\mu,\nu\in\mathcal{P}(K)}
\mathbb{E}\left|\widehat{T}-
\operatorname{SOW}_{\phi,p}(\mu,\nu)\right|\asymp_{d,p}D\left[
\omega_\phi\left(n^{-1/2}\right)+\omega_\phi\left(m^{-1/2}\right)
\right],
$$
and
$$\inf_{\widehat{T}}\sup_{\mu,\nu\in\mathcal{P}(K)}\mathbb{E}
\left|\widehat{T}-\operatorname{SOW}_{\phi,p}^{p}(\mu,\nu)
\right|\asymp_{d,p}D^p\psi_{\phi,p}^{\mathrm{cav}}\left(\frac{1}{2}n^{-1/2}+\frac{1}{2}m^{-1/2}\right).
$$
Here, the first infimum is taken over all estimators based on $n$ observations, while the last two are taken over all estimators based on the two independent samples of sizes $n$ and $m$.
\end{theorem}
The proof of Theorem~\ref{thm:SOW-minimax}, based on Le Cam's
two-point method
\cite{lecam1973convergence,le2012asymptotic}, is given in Appendix~\ref{app:proof-minimax}. Consequently, the empirical plug-in estimators studied in the preceding subsection are minimax-rate optimal over $\mathcal{P}(K)$. For fixed $d$, $p$, $\phi$, and $K$, the
sample-size dependence of these rates is tight up to constant factors.
\subsection{Computational Properties}
\label{subsec:computation}
In this subsection, we focus on the $\operatorname{SOW}$ between the empirical measures $\widehat{\mu}_n$ and $\widehat{\nu}_m$. The observed samples are held fixed throughout. We approximate the spherical integral by Monte Carlo sampling and compute the projected $\operatorname{OW}$ distances by bisection when no closed-form expression is available\footnote{See Appendix~\ref{app:closeform}.}.

\textbf{Monte Carlo approximation.}
For $\theta_1,\dots,\theta_L\overset{\mathrm{i.i.d.}}{\sim}\sigma$, we first approximate $\operatorname{SOW}^p_{\phi,p}\left(\widehat{\mu}_n,\widehat{\nu}_m\right)$ by
\begin{equation}\label{eq:MC-SOW}
\widehat{\operatorname{SOW}}_{\phi,p}^{p}(\widehat{\mu}_n,\widehat{\nu}_m,L)
\coloneqq \frac{1}{L}\sum_{l=1}^{L}
\operatorname{W}_{\phi}^{p}\left(\Pi^{\theta_l}_{\#}\widehat{\mu}_n,
\Pi^{\theta_l}_{\#}\widehat{\nu}_m\right).
\end{equation}

\textbf{Bisection for the one-dimensional distances.}
Fix $\theta\in\mathbb{S}^{d-1}$ and write
\begin{equation}\label{eq:projected-quantile-gap}
\Delta_{\theta}(u)\coloneqq
F^{-1}_{\Pi^{\theta}_{\#}\widehat{\mu}_n}(u)-F^{-1}_{\Pi^{\theta}_{\#}\widehat{\nu}_m}(u),
\qquad
\Psi_{\theta}(\lambda)\coloneqq
\int_{0}^{1}\phi\left(\frac{|\Delta_{\theta}(u)|}{\lambda}\right)\mathrm{d}u,
\quad \lambda>0 .
\end{equation}
By the quantile representation \eqref{eq:OW-quantile},
$\operatorname{W}_{\phi}(\Pi^{\theta}_{\#}\widehat{\mu}_n,\Pi^{\theta}_{\#}\widehat{\nu}_m)
=\inf\{\lambda>0:\Psi_{\theta}(\lambda)\leq1\}$. Since $\phi$ is
increasing, $\Psi_{\theta}$ is decreasing in $\lambda$, so the
sublevel set $\{\Psi_{\theta}\leq1\}$ is a half-line and the value of
$\operatorname{W}_{\phi}$ can be located by bisection. By Jensen's inequality and the monotonicity of $\phi$, we obtain the following initial bracket:
\begin{equation}\label{eq:bracket}
L_{\theta}\coloneqq \omega_{\phi}(1)\|\Delta_{\theta}\|_{L^{1}(0,1)}
\leq\operatorname{W}_{\phi}\left(\Pi^\theta_{\#}\widehat{\mu}_n,\Pi^\theta_{\#}\widehat{\nu}_m\right)\le
\omega_{\phi}(1)\max_{u \in (0,1)}\left|\Delta_{\theta}(u)\right|
\coloneqq U_\theta.
\end{equation}

Let $\widehat{\operatorname{W}}_{\phi}
(\Pi^{\theta}_{\#}\widehat{\mu}_n,\Pi^{\theta}_{\#}\widehat{\nu}_m,T)$ denote the upper endpoint of
the bracket produced by $T$ bisection steps started from $[L_{\theta},U_{\theta}]$ (see Algorithm~\ref{alg:SOW}). The resulting
estimator of $\operatorname{SOW}_{\phi,p}^{p}$ is
\begin{equation}\label{eq:MC-bisect-SOW}
\widehat{\operatorname{SOW}}_{\phi,p}^{p}(\widehat{\mu}_n,\widehat{\nu}_m,L,T)
\coloneqq\frac{1}{L}\sum_{l=1}^{L}
\widehat{\operatorname{W}}_{\phi}^{p}
\left(\Pi^{\theta_l}_{\#}\widehat{\mu}_n,\Pi^{\theta_l}_{\#}\widehat{\nu}_m,T\right).
\end{equation}

\begin{algorithm}[!t]
\caption{Computation of $\operatorname{SOW}_{\phi,p}^p\left(\widehat{\mu}_n,\widehat{\nu}_m\right)$}
\label{alg:SOW}
\begin{algorithmic}
\STATE \textbf{Input:} $\widehat{\mu}_n,\widehat{\nu}_m$, Orlicz function $\phi$, $p\geq1$, number of
projections $L$, number of bisection steps $T$.
\STATE Sample $\theta_1,\dots,\theta_L\overset{\mathrm{i.i.d.}}{\sim}\sigma$.
\STATE Project and sort to obtain the quantile gaps
$\Delta_{\theta_l}$, $l=1,\dots,L$.
\STATE Initialize $\ell_l \leftarrow L_{\theta_l}$ and $u_l \leftarrow U_{\theta_l}$ for $l=1,\dots,L$ by using \eqref{eq:bracket}.
\STATE Keep $u_l$ unchanged for directions with $\ell_l=u_l$.
\STATE Perform the following steps only for directions
with $\ell_l<u_l$ at initialization.
\FOR{$t=1$ to $T$}
  \STATE $v_l\leftarrow(\ell_l+u_l)/2$.
  \STATE Evaluate $\Psi_{\theta_l}(v_l)$ as in
  \eqref{eq:projected-quantile-gap}.
  \STATE $u_l\leftarrow v_l$ if $\Psi_{\theta_l}(v_l)\leq1$;
  \quad $\ell_l\leftarrow v_l$ otherwise.
\ENDFOR
\STATE \textbf{Return:} $\tfrac{1}{L}\sum_{l=1}^{L}u_l^{p}$.
\end{algorithmic}
\end{algorithm}
The following proposition bounds the approximation error of Algorithm~\ref{alg:SOW}. Its proof is given in Appendix~\ref{sec:proof-SOW-computation}.
\begin{proposition}\label{prop:SOW-computation}
The approximation returned by Algorithm~\ref{alg:SOW} satisfies
\begin{align*}
&\mathbb{E}_{\theta_1,\dots,\theta_L}\left|
\widehat{\operatorname{SOW}}_{\phi,p}^{p}(\widehat{\mu}_n,\widehat{\nu}_m,L,T)
-\operatorname{SOW}_{\phi,p}^{p}(\widehat{\mu}_n,\widehat{\nu}_m)\right| \\&\le
\frac{p}{2^{T}}\int_{\mathbb{S}^{d-1}}U_{\theta}^{p}\,\mathrm{d}\sigma(\theta)
+
\frac{1}{\sqrt{L}}\,
\operatorname{Var}_{\theta\sim\sigma}
\left[\operatorname{W}_{\phi}^{p}
\left(\Pi^{\theta}_{\#}\widehat{\mu}_n,\Pi^{\theta}_{\#}\widehat{\nu}_m\right)\right]^{1/2}.
\end{align*}
\end{proposition}

\textbf{Computational complexity.} For empirical measures with $n$ atoms, $\Delta_{\theta}$ is a step function with $n$ pieces, so each evaluation of $\Psi_{\theta}$ costs $\mathcal{O}(n)$
operations. Algorithm~\ref{alg:SOW} therefore runs in
$\mathcal{O}\left(L\left(dn+n\log n+Tn\right)\right)$ time and $\mathcal{O}(L(d+n))$ memory, the three terms accounting for projection, sorting, and bisection respectively.

\section{Experiments}
\label{sec:exps}

\subsection{Synthetic Studies}
\label{subsec:compare_OW}
We create data from $\mu = \mathcal{N}(0, I_d)$ and $\nu = 0.95\,\mathcal{N}(0, I_d) + 0.025\,\mathcal{N}(2.5\,e_1, I_d) + 0.025\,\mathcal{N}(-2.5\,e_1, I_d)$,
with $d=2$ and $e_1=(1,0)$. We report the computational speed of SOW ($L=100$, $p=2$) and OW using entropic
regularization (entropic coefficient $5$, up to 150 Sinkhorn iterations with $10^{-7}$ marginal error)~\cite{guha2023excess} in Figure~\ref{fig:speed}. Here, we consider
$\phi_1(x)=\frac{x^2+x^4}{2}$, and $\phi_2(x) = \frac{e^{x^2}-1}{e-1}$ as two Orlicz functions. $\phi_1$
additionally admits an exact closed-form per-projection solve for SOW (no bisection at all). SOW with
$\phi_2$ and OW (for both $\phi_1$ and $\phi_2$) are solved by bisection. The results are averaged over
5 runs. The left figure shows that SOW is much more scalable than OW. Here, we choose the
number of bisection iterations $T$ for OW and for SOW's $\phi_2$ such that we have a guaranteed error
of $10^{-9}$ or reaching $100$ iterations (the tolerance
is often reached after $\sim33$--$34$ iterations). In the figure, the computational speed of SOW with $n=10^4$ is even lower than OW with $n=100$. Moreover, the scaling rate is much lower for SOW. The
middle figure confirms that the scaling in the number of projections $L$ for SOW is linear. The right
figure presents the computational time of OW and SOW in terms of the number of bisection iterations
$T$, forced directly up to $T=100$ for SOW's $\phi_2$ and OW's bisection (rather than derived from
the tolerance above). We also observe that SOW's $\phi_2$ has a better scaling rate than OW in terms
of $T$, due to a faster computational time per iteration. SOW's $\phi_1$ curve is flat in $T$ by
construction, since its closed-form solve does not depend on $T$ at all.
\begin{figure}[t]
    \centering
    \begin{tabular}{ccc}       \includegraphics[width=0.285\linewidth]{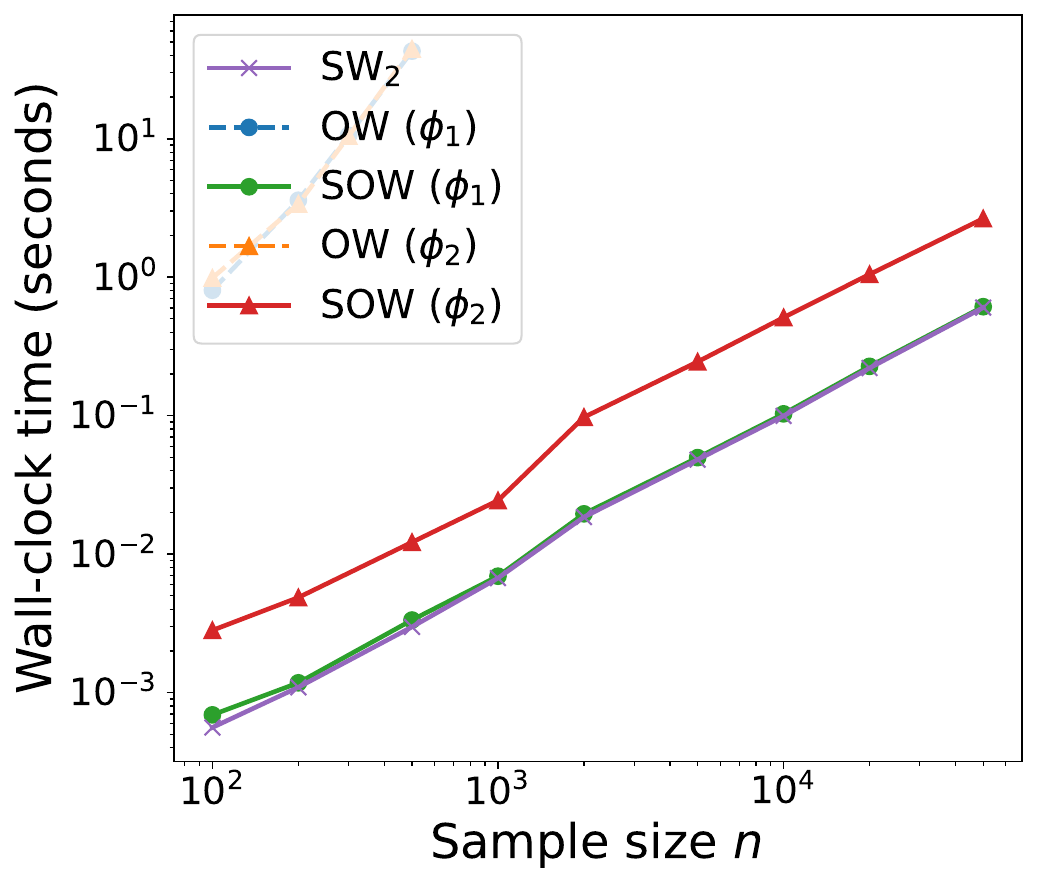}
         & 
         \includegraphics[width=0.285\linewidth]{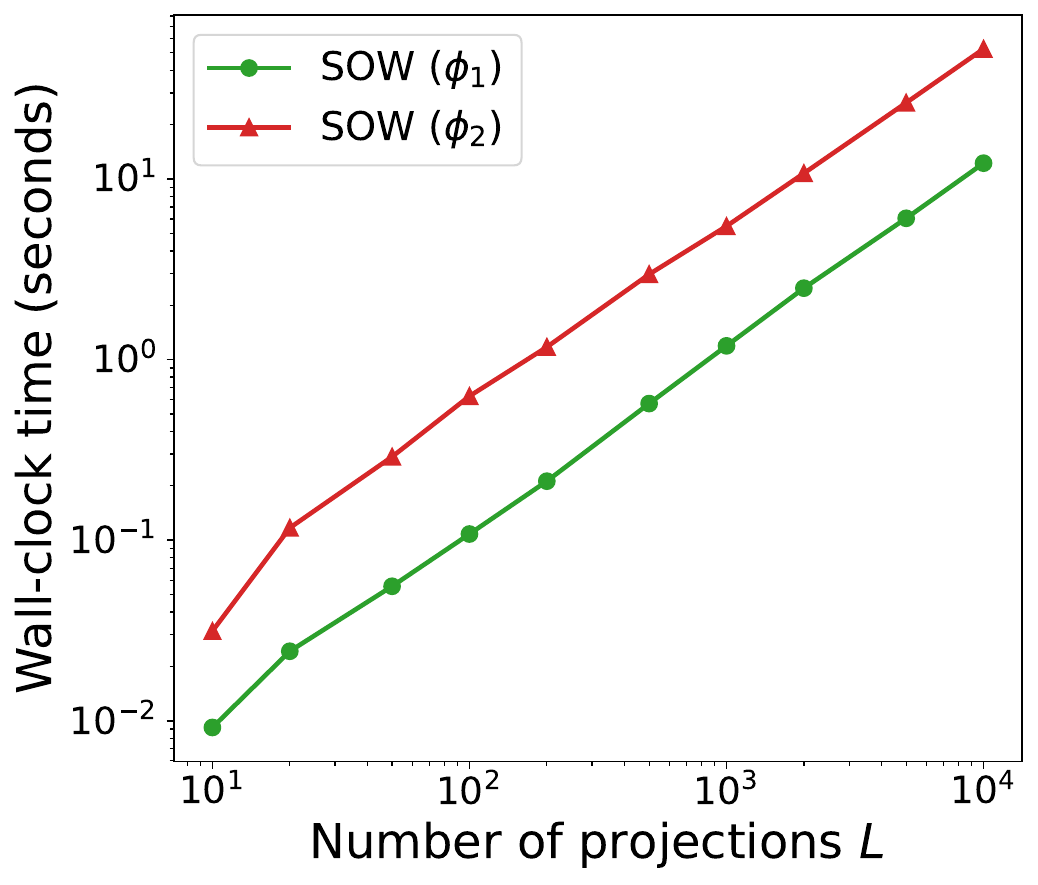}
         &\includegraphics[width=0.285\linewidth]{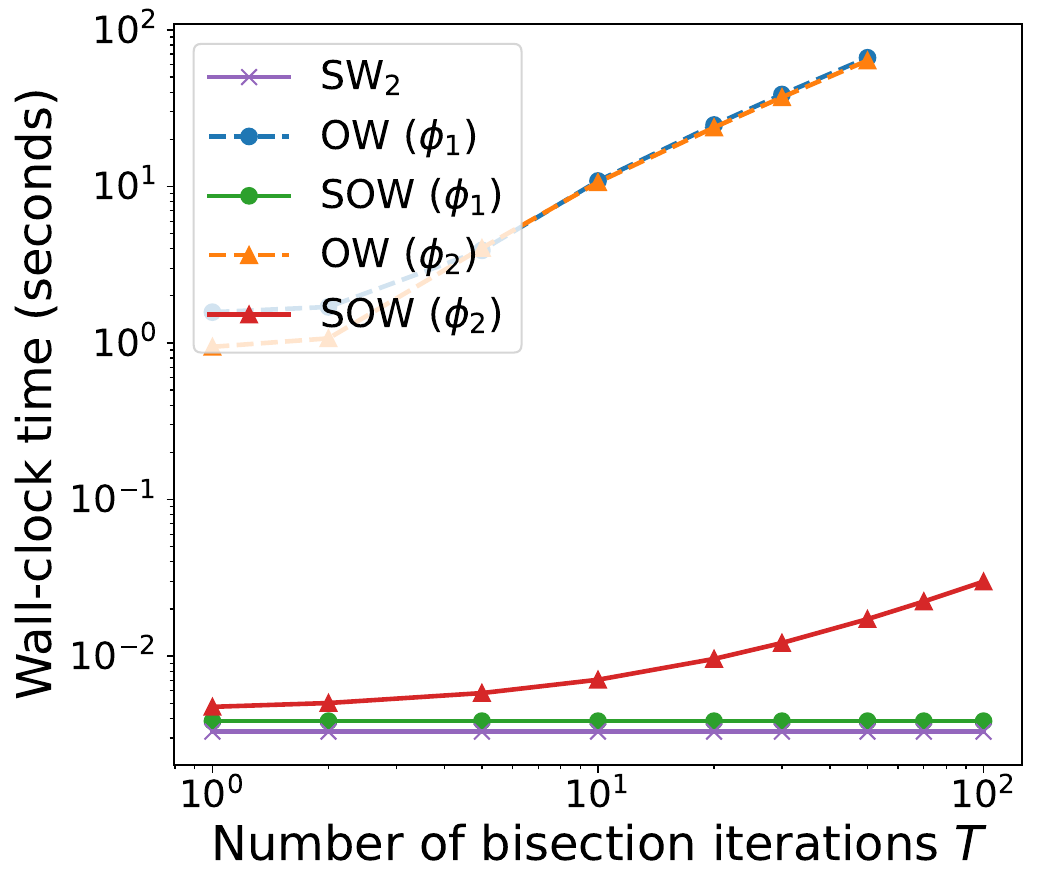}
    \end{tabular}
   \vspace{-1em}    \caption{Computational speed of SOW and OW.}
    \label{fig:speed}
\end{figure}

We further investigate the empirical sample-size, projection, and bisection errors in Figure~\ref{fig:discrimination}. For the left panel, we first generate and fix two empirical reference distributions, each consisting of $10^5$ observations. For each replication, we draw an i.i.d.\ sequence of $50,000$ observations with replacement from each fixed empirical reference distribution and use its prefixes to form nested empirical samples with sizes $n$ ranging from $100$ to $50,000$. We fix $L=100$ and use the same projection directions to evaluate each sample pair and the corresponding reference pair. The reported errors are averaged over $100$ replications. For the middle panel, we fix $n=1000$ and vary $L$ from $5$ to $5000$, using the value at $L=10^4$ on the same data and with the same projection sequence as the numerical reference. For the right panel, we fix $n=1000$ and $L=100$ and vary $T$ from $1$ to $100$ for $\operatorname{SOW}$ with $\phi_2$, with the reference computed using a tolerance of $10^{-15}$. The left panel shows decreasing empirical errors for all three distances, with $\operatorname{SOW}$ using $\phi_2$ exhibiting the slowest decay, consistent with the dependence on the Orlicz function through $\omega_{\phi}$ in Section~\ref{subsec:stats}. The middle panel shows decreasing Monte Carlo projection error as $L$ increases, while the right panel shows that the bisection error decreases rapidly and becomes very small after about $50$ iterations.


\begin{figure}[t]
    \centering
    \begin{tabular}{ccc}
\includegraphics[width=0.285\linewidth]{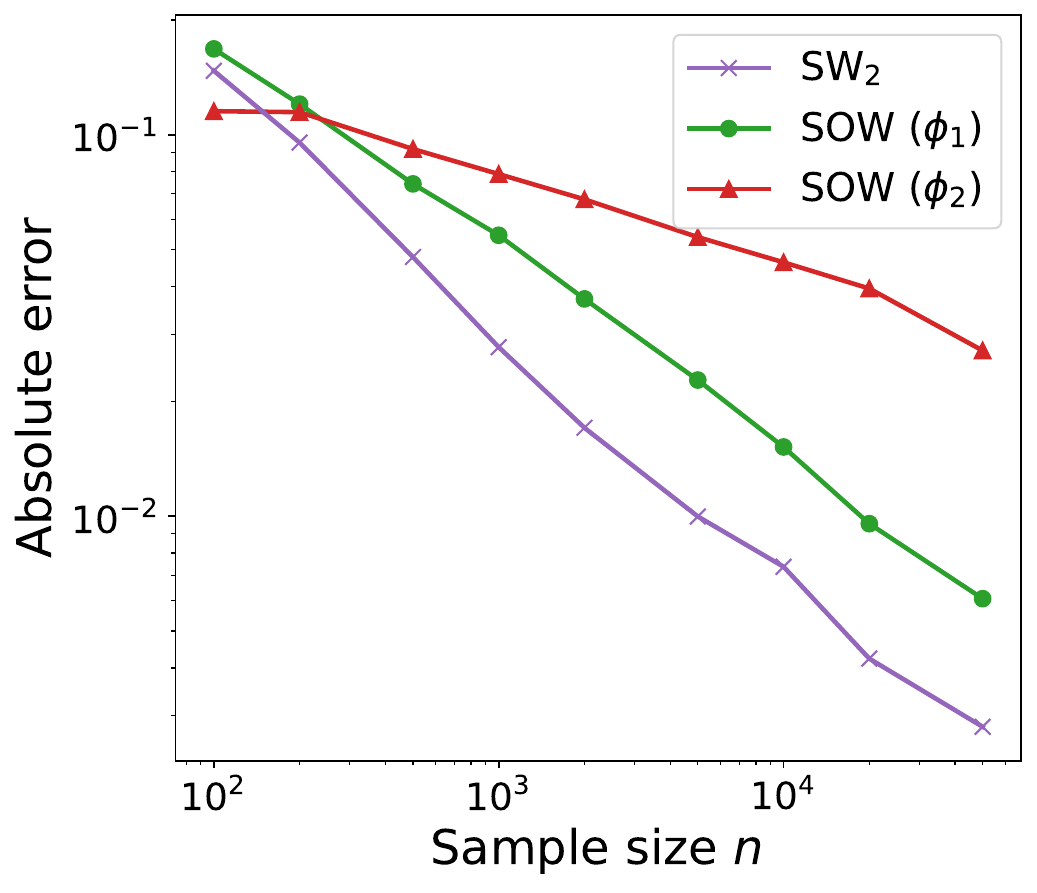}
&\includegraphics[width=0.285\linewidth]{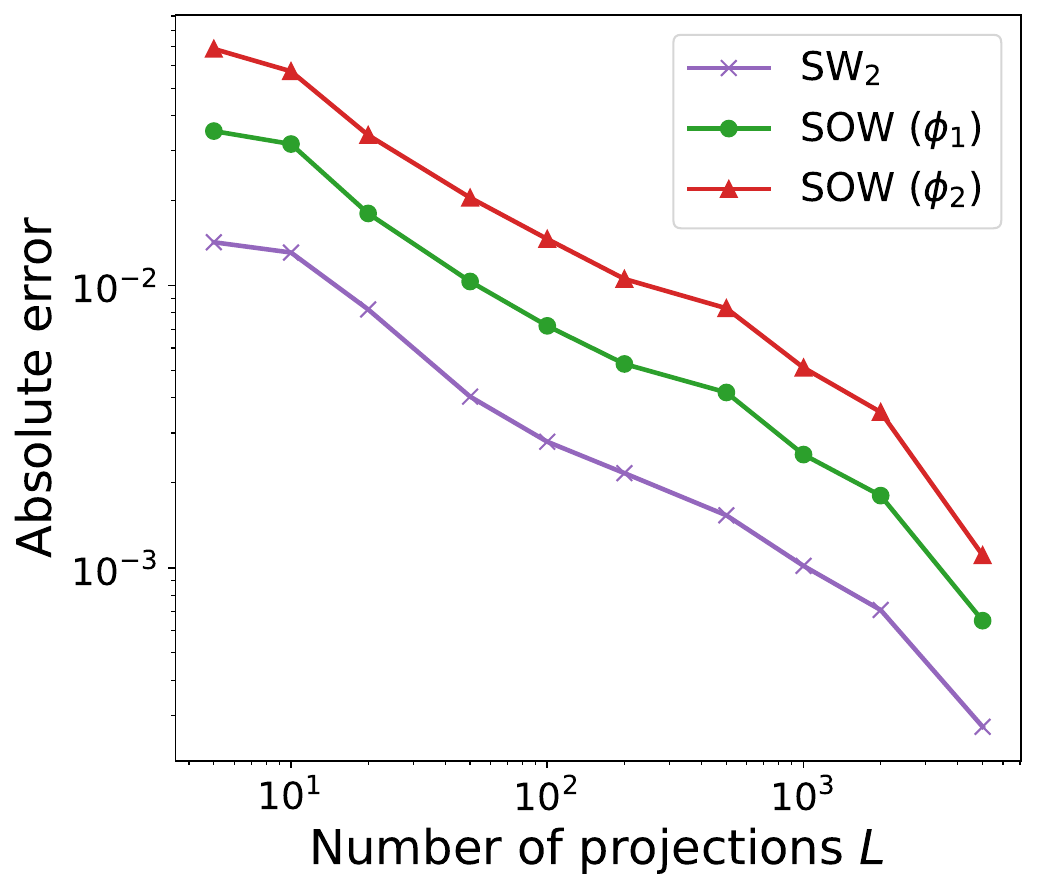}
         &    \includegraphics[width=0.285\linewidth]{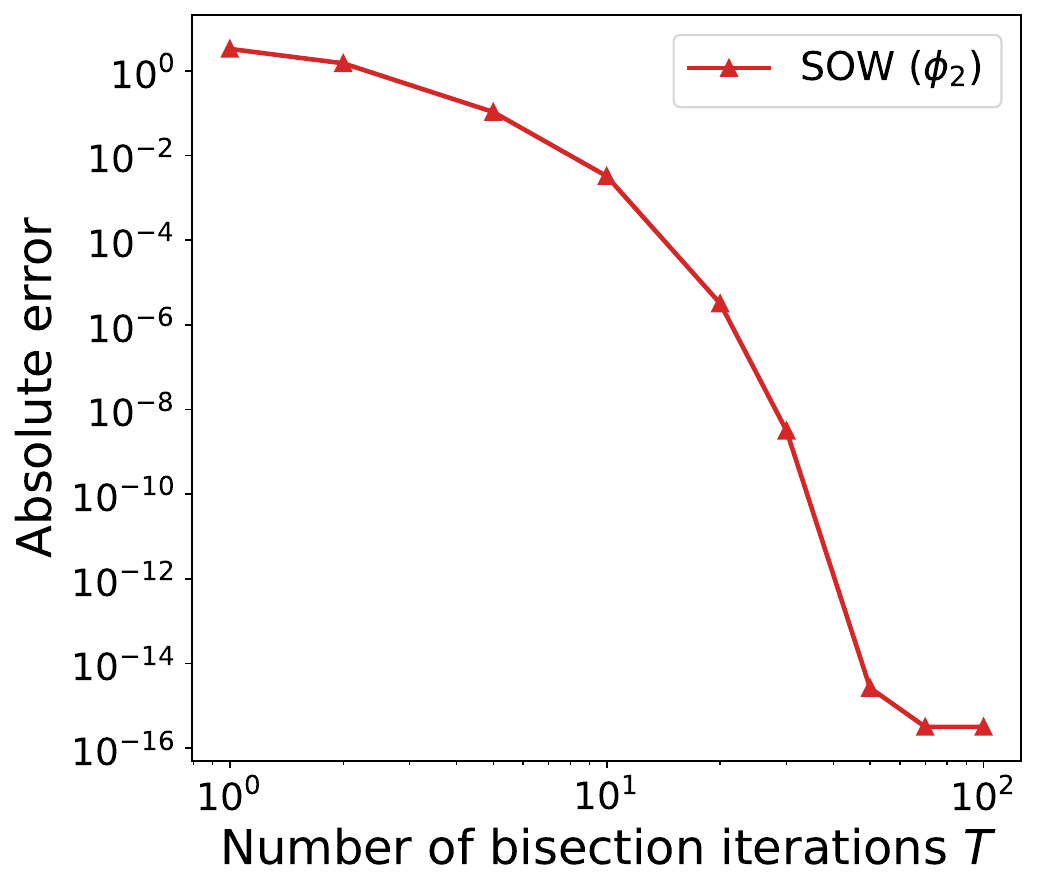}
    \end{tabular}
   \vspace{-1em}
    \caption{Sample and Monte Carlo projection errors of SW and SOW, and bisection error of SOW.}
    \label{fig:discrimination}
\end{figure}
\subsection{Two-sample Testing}
\label{subsec:testing}
We evaluate $\operatorname{SOW}$ for two-sample testing using $\operatorname{SOW}_{\phi,2}\left(\widehat{\mu}_n,\widehat{\nu}_n\right)$ as the test statistic, with $\phi_1$ and $\phi_2$ from Section~\ref{subsec:compare_OW}. We compare against $\operatorname{SW}_2$ \cite{pmlr-v300-tran26c} and MMD with RBF and Laplace kernels \cite{gretton2012kernel,schrab2025practical}, using the median heuristic for bandwidth selection. For $\operatorname{SW}$ and $\operatorname{SOW}$, we use $L=n/2$ shared projection directions within each trial. Empirical power is estimated over $200$ independent trials using permutation test with $199$ permutations at level $\alpha =0.05$, thereby controlling the finite-sample Type I error under the null exchangeability assumption \cite{hemerik2018exact}. In Figure~\ref{fig:test}, we consider three settings. For the synthetic experiments, $\mu=\mathcal{N}\left(0,I_{15}\right)$ and $\nu=0.95\mu+0.025\mathcal{N}\left(2.5e_1,I_{15}\right)+0.025\mathcal{N}\left(-2.5e_1,I_{15}\right)$. For MNIST \cite{lecun1998gradient}, following \cite{wang2021two}, we preprocess the dataset by applying a sigmoid transformation to each image so that all pixel
values lie within $[0,1]$. We then compare the processed digit-6 distribution $\mu_6$ with $0.9\mu_6+0.1\mu_9$. For CIFAR, following \cite{biggs2023mmd}, we compare the uniform distributions over CIFAR-10 \cite{krizhevsky2009learning} and CIFAR-10.1 \cite{recht2019imagenet}. Figure~\ref{fig:test} shows that suitable choices of the Orlicz function can yield higher empirical test power than $\operatorname{SW}_2$ at the same sample size.

\begin{figure}[t]
    \centering
    \begin{tabular}{ccc}
\includegraphics[width=0.285\linewidth]{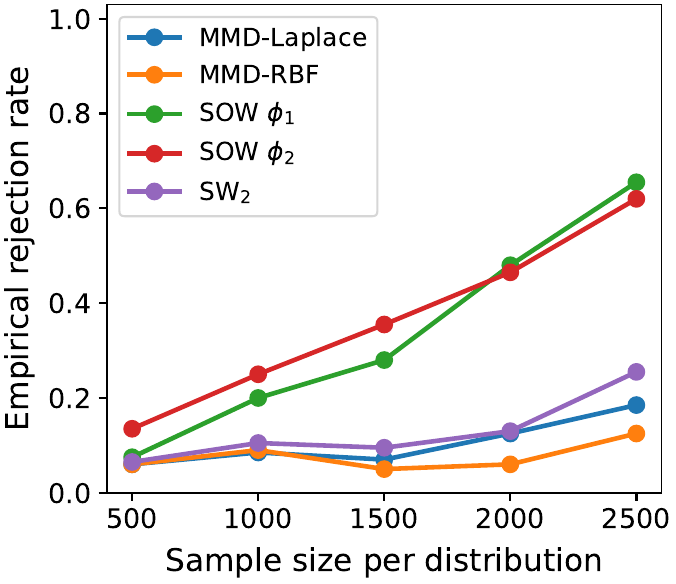}
         & 
\includegraphics[width=0.285\linewidth]{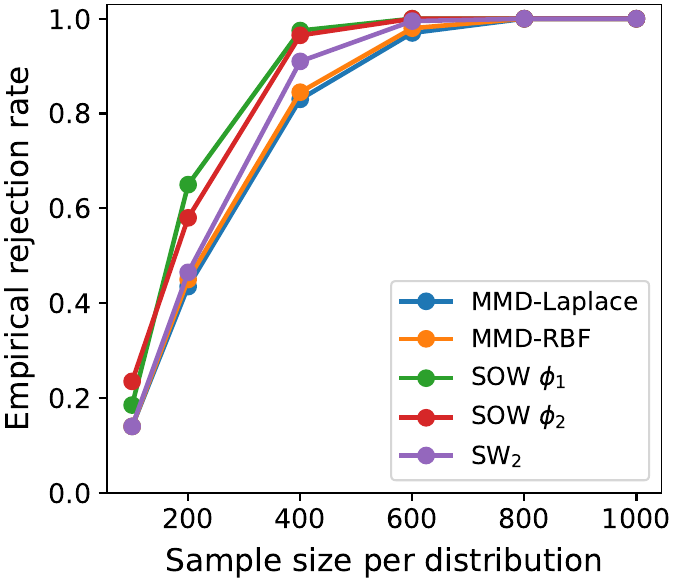}
     &\includegraphics[width=0.285\linewidth]{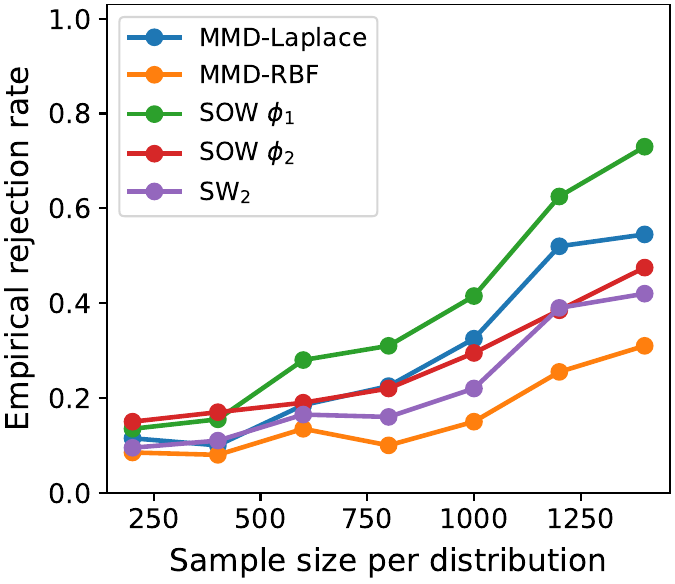}
    \end{tabular}
   \vspace{-1em}
    \caption{Power versus sample-size: synthetic data, MNIST and CIFAR-10 versus CIFAR-10.1.}
    \label{fig:test}
\end{figure}
\subsection{Evaluating Generative Models}
\label{subsec:evaluate_generative_model}


\begin{figure}[!t]
    \centering
\setlength{\belowcaptionskip}{-2pt}
\includegraphics[width=1\linewidth]{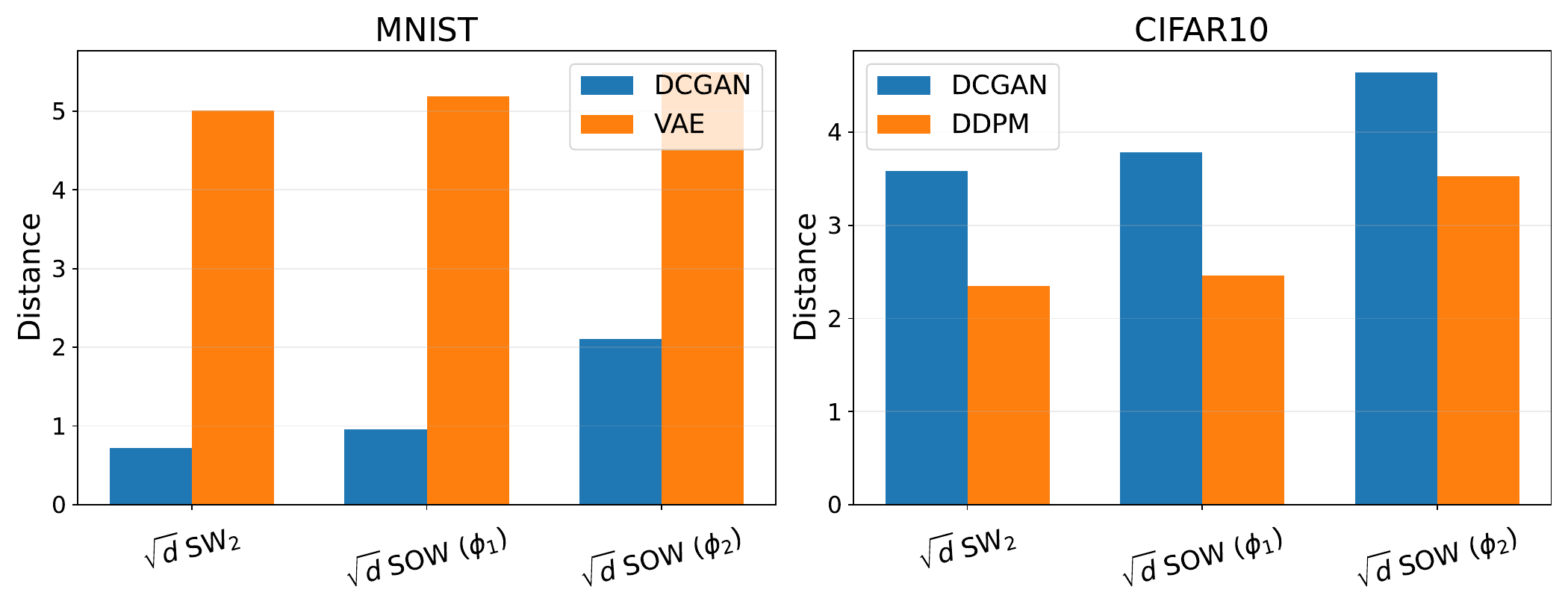}
\vspace{-1em}
    \caption{Evaluating generative models with SW and SOW.}
    \label{fig:gen}
\end{figure}
Evaluating a generative model is to measure a statistical distance between the model distribution and data distribution. The standard FID~\cite{heusel2017gans} uses Wasserstein-2 between fitted Gaussians on Inception features and
compares only their first two moments, which requires estimating a $d\times d$
covariance (hence $50$k samples for Inception features). Authors in~\cite{berthet2026mind} propose to use SW directly on the set of samples and demonstrate that SW can avoid the weaknesses of FID which is failing to separate distributions that match in mean and covariance but differ elsewhere. Since SOW is a family indexed by $\phi$, evaluating a model at several Orlicz functions of increasing growth yields a profile rather than
a single score. Fast-growing $\phi$ inflates the contribution of the quantiles at which the two projected measures are far apart, i.e., the tails and the rare or missing modes. How fast the score grows along the profile therefore says where the mismatch sits, which a single SW value cannot reveal.

In Figure~\ref{fig:gen} we compare a DCGAN~\cite{radford2015unsupervised} against a VAE~\cite{kingma2013auto} on MNIST and
against a DDPM~\cite{ho2020denoising} on CIFAR10, reporting $\sqrt{d}\,\mathrm{SW}_2$,
$\sqrt{d}\,\mathrm{SOW}(\phi_1)$ and $\sqrt{d}\,\mathrm{SOW}(\phi_2)$ with
$\phi_1(x)=\frac{x^2+x^4}{2}$ and $\phi_2(x)=\frac{e^{x^2}-1}{e-1}$. On MNIST the
DCGAN score rises considerably while VAE score increases slightly using $10,000$ samples and $L=10,000$ projections. This demonstrates that VAE's error is a large but
uniformly spread blur that every cost charges equally, whereas the DCGAN
concentrates its error on a small fraction of mass transported far, exactly what $\phi_1$ and $\phi_2$ penalize. We can then conclude that  DCGAN is better mainly where most
of the mass lies and much less so once far transport is charged heavily.  On CIFAR-10, the gaps between DCGAN and DDPM stay approximately constant. This says both models miss the data in the tails by a similar amount, and the DDPM is simply better by the same margin everywhere.

\section{Conclusion}
\label{sec:conclusion}
We introduced $\operatorname{SOW}$, a generalization of $\operatorname{SW}$ based on Orlicz geometry. We established its metric, topological, statistical, and computational properties, including minimax-optimal sample complexities under compact support and developed an approximation algorithm. Experiments show that $\operatorname{SOW}$ scales substantially better than $\operatorname{OW}$ and provide initial evidence that suitable choices of $\phi$ can increase sensitivity to discrepancies involving small amounts of mass transported over large distances. Our goal here is to demonstrate this flexibility rather than optimize $\phi$ for individual tasks. Principled and data-adaptive selection of $\phi$ is an important direction for future work. Another direction is to use $\operatorname{SOW}$ directly as a risk for model estimation, which requires differentiating through the scalar Luxemburg-norm solve and controlling its numerical error. Further extensions include non-uniform slicing distributions~\cite{nguyen2021distributional,nguyen2023energy}, non-linear projections~\cite{kolouri2019generalized,bonet2024sliced}, better Monte Carlo methods~\cite{nguyen2024quasimonte}. Extending the statistical guarantees beyond compact support is also an important direction (see Appendix~\ref{app:beyondboundedsupport} for our first step).
\section*{Acknowledgement}
Portions of this research were conducted with the advanced computing resources provided by Texas A\&M High Performance Research Computing.

\section*{AI use statement}

In this work, we used generative AI tools for editing the paper, for coding, and for brainstorming proof strategies. We take responsibility for the final content of this work,
including text, claims, or artifacts produced with the aid of generative AI.

\vspace{2em}
\appendix
\begin{center}
{\bf{\Large{Appendix to ``Sliced Orlicz-Wasserstein"}}}
\end{center}
\vspace{2em}

\section{Basic Properties of the Orlicz Quantities}
The purpose of this appendix is to collect the elementary properties that will be used later in the proofs. For further background on Orlicz
functions and spaces, we refer the reader to \cite{rao1991theory} and \cite[Section~13.1]{rubshtein2016foundations}.
\subsection{Properties of $\phi$ and $\phi^{-1}$}
\begin{lemma}[Basic properties of the Orlicz function]\label{lem:properties-Nfunctions}
Suppose that $\phi$ is the Orlicz function introduced in Section~\ref{sec:background}. Then $\phi$ is continuous on $[0,+\infty)$ and $\displaystyle \lim_{t \to +\infty}\phi(t)=+\infty$. Consequently, $\phi$ is a bijection from $[0,+\infty)$ onto $[0,+\infty)$. Its inverse $\phi^{-1}: [0,+\infty)\to [0,+\infty)$ is continuous and strictly increasing, with $\phi^{-1}(0)=0$ and $\displaystyle \lim_{s \to +\infty}\phi^{-1}(s)=+\infty$.
\end{lemma}
\begin{proof}
Since $\phi$ is finite-valued and convex, it is continuous on $(0,+\infty)$; see, for example, \cite[Theorem~10.1]{rockafellar1970convex}. It remains to verify right-continuity at zero. Fix $a>0$. For every $t\in [0,a]$, convexity and $\phi(0)=0$ give
\begin{equation*}
0 \le \phi(t) = \phi\left[\frac{t}{a}a+\left(1-\frac{t}{a}\right)0\right] \le \frac{t}{a}\phi(a).
\end{equation*}
Therefore, $\phi(t) \to 0=\phi(0)$ as $t \to 0^{+}$. Hence, $\phi$ is continuous on $[0,+\infty).$

Next, since $\phi$ is strictly increasing on $[0,+\infty)$ and $\phi(0)=0$, $\phi(1)>0$. For every $t \ge 1$, the convexity of $\phi$ gives
\begin{equation*}
\phi(1)=\phi\left[\frac{1}{t}t+\left(1-\frac{1}{t}\right)0\right] \le \frac{\phi(t)}{t}.
\end{equation*}
Then, $\phi(t) \ge t\phi(1)$. Therefore, $\phi(t) \ge t \phi(1) \to +\infty$ as $t \to +\infty$. Together with $\phi(0)=0$ and the continuity of $\phi$ on $[0,+\infty)$, this implies that $\phi\left([0,+\infty)\right)=[0,+\infty)$. Since $\phi$ is strictly increasing, it follows that $\phi$ is a bijection from $[0,+\infty)$ onto $[0,+\infty)$.


Since the inverse of a continuous strictly increasing function is continuous and strictly increasing on its image (see, for example, \cite[Theorem~5.6.5]{bartle2000introduction}), $\phi^{-1}$ is continuous and strictly increasing on $[0,+\infty)$. Moreover, $\phi^{-1}(0)=0$ as $\phi(0)=0$. Finally, suppose that $\phi^{-1}(s)$ did not tend to infinity as $s \to +\infty$. Then there would exist $M>0$ and a sequence $\left(s_n\right)_{n\ge 1}$ such that
$s_n\to+\infty$ and $\phi^{-1}(s_n)\le M$ for every $n\ge 1$. Since $\phi$ is increasing,
\begin{equation*}
s_n = \phi\left(\phi^{-1}\left(s_n\right)\right) \le \phi(M),
\end{equation*}
which contradicts $\displaystyle \lim_{n\to+\infty}s_n=+\infty$. Therefore, $\displaystyle \lim_{s \to +\infty}\phi^{-1}(s)=+\infty$.
\end{proof}
\subsection{Properties of $\omega_{\phi}$ and $\psi_{\phi,p}$}
Given an Orlicz function $\phi$ and $p \in [1,+\infty)$, we recall that the fundamental function of the Orlicz space $L^{\phi}$, equipped with the Luxemburg norm, is given by
\begin{equation*}
\omega_{\phi}(u) \coloneqq  \frac{1}{\phi^{-1}(1/u)}, \quad u\in (0,1]
\end{equation*}
and $\omega_{\phi}(0)\coloneqq 0$. Its $p$-th power is given by
\begin{equation*}
\psi_{\phi,p}(u)\coloneqq \omega_{\phi}(u)^p =\frac{1}{\phi^{-1}(1/u)^p}, \quad u\in(0,1], \quad \psi_{\phi,p}(0)=0.
\end{equation*}
Further properties of fundamental functions can be found in \cite[Section~5]{bennett1988interpolation} and \cite[Sections~10.1 and~13.3]{rubshtein2016foundations}. We state below the properties that will be used later. The following result is closely related to \cite[Corollary~5.3]{bennett1988interpolation}.
\begin{lemma}[Basic properties of $\omega_{\phi}$ and $\psi_{\phi,p}$]\label{lem:properties-omega-psi}
The functions $\omega_{\phi}$ and $\psi_{\phi,p}$ are continuous and strictly increasing on $[0,1]$. Moreover, for every $a\ge 1$ and $u\in (0,1]$ such that $au\le 1$,
\begin{equation*}
\omega_{\phi}\left(au\right) \le a \omega_{\phi}(u).
\end{equation*}
\end{lemma}
\begin{proof}
On $(0,1]$, the continuity of $\omega_{\phi}$ follows directly from Lemma~\ref{lem:properties-Nfunctions}. At the origin, Lemma~\ref{lem:properties-Nfunctions} gives
\begin{equation*}
\lim_{u \to 0^+}\phi^{-1}(1/u)=+\infty.
\end{equation*}
Therefore,
\begin{equation*}
\lim_{u \to 0^+} \omega_{\phi}(u) = \lim_{u \to 0^+} \frac{1}{\phi^{-1}(1/u)}=0=\omega_{\phi}(0).
\end{equation*}
Thus, $\omega_{\phi}$ is continuous on $[0,1]$.

Let $0<u_1<u_2 \le 1$. Then 
\begin{equation*}
\frac{1}{u_1} > \frac{1}{u_2}.
\end{equation*}
Since $\phi^{-1}$ is strictly increasing,
\begin{equation*}
\phi^{-1}\left(1/u_1\right) > \phi^{-1}\left(1/u_2\right),
\end{equation*}
and hence
\begin{equation*}
\omega_{\phi}(u_1) < \omega_{\phi}(u_2).
\end{equation*}
Also, $\omega_{\phi}(0)=0<\omega_{\phi}(u)$ for every $u>0$. Hence $\omega_{\phi}$ is strictly increasing on $[0,1]$.

Since $p\ge 1$, the map $r \mapsto r^p$ is continuous and strictly increasing on $[0,+\infty)$. Therefore, $\psi_{\phi,p}$ is continuous and strictly increasing on $[0,1]$.

Finally, by \cite[Proposition~13.3.1]{rubshtein2016foundations}, $\omega_{\phi}$ is the fundamental function of the corresponding Orlicz space. Moreover, by \cite[Theorem~10.4.2]{rubshtein2016foundations}, every fundamental function of a symmetric space is quasiconcave; see \cite[Definition~10.1.1]{rubshtein2016foundations} for the definition of a quasiconcave function. Therefore,
\begin{equation*}
v \mapsto \frac{\omega_{\phi}(v)}{v}
\end{equation*}is decreasing on $(0,1]$. Now, let $a\ge 1$ and $u \in (0,1]$ such that $au \le 1$. Since $u \le au$ and the map $v \mapsto \omega_{\phi}(v)/v$ is decreasing on $(0,1]$, \begin{equation*}
\frac{\omega_{\phi}(au)}{au} \le \frac{\omega_{\phi}(u)}{u}.
\end{equation*}
Hence,
\begin{equation*}
\omega_{\phi}(au)\le a \omega_{\phi}(u).
\end{equation*}
\end{proof}
\subsection{The existence and properties of $\psi^{\mathrm{cav}}_{\phi,p}$}
Recall that $\psi^{\mathrm{cav}}_{\phi,p}$ denotes the least concave majorant of $\psi_{\phi,p}$ on $[0,1]$, i.e., the smallest concave function on $[0,1]$ that dominates $\psi_{\phi,p}$. We summarize below the properties that will be used later. For further properties of $\psi^{\mathrm{cav}}_{\phi,p}$, we refer the reader to \cite{francu2017new}.
\begin{lemma}\label{lem:properties-psi-cav}
The least concave majorant $\psi^{\mathrm{cav}}_{\phi,p}$ exists and is continuous on $[0,1]$. Moreover,
\begin{equation*}
\psi^{\mathrm{cav}}_{\phi,p}(0)=0, \quad \psi^{\mathrm{cav}}_{\phi,p}(1)=\psi_{\phi,p}(1),
\end{equation*}
and $\psi^{\mathrm{cav}}_{\phi,p}$ is strictly increasing on $[0,1]$.
\end{lemma}
\begin{proof}
Let $\mathcal{G}$ denote the family of concave functions on $[0,1]$ dominating $\psi_{\phi,p}$. Since $\psi_{\phi,p}$ is strictly increasing on $[0,1]$, we have
\begin{equation*}
\psi_{\phi,p}(u)\leq \psi_{\phi,p}(1),
\quad u\in[0,1].
\end{equation*}
Hence, the constant function $G(u)\equiv \psi_{\phi,p}(1)$ is a concave majorant of $\psi_{\phi,p}$ on $[0,1]$. Therefore, $\mathcal G\neq\varnothing$. 

For $u\in[0,1]$ define
\begin{equation*}
H(u)\coloneqq \inf_{G\in\mathcal{G}}G(u).
\end{equation*}
Since every $G\in\mathcal G$ dominates $\psi_{\phi,p}$, so does $H$. Besides, since the pointwise infimum of a family of concave functions is concave \cite[Section~3.2.3]{boyd2004convex}, $H$ is concave on $[0,1]$. Thus, $H$ is a concave majorant of $\psi_{\phi,p}$. Furthermore, by construction, $H\leq G$ for every $G\in\mathcal G$. Hence, $H$ is the least concave majorant of $\psi_{\phi,p}$.

Applying \cite[Lemma~1]{francu2017new} (with $F=\psi_{\phi,p}$ and $I=[0,1]$) gives
\begin{equation*}
\psi^{\mathrm{cav}}_{\phi,p} \text{ is continuous on }[0,1], \quad \psi^{\mathrm{cav}}_{\phi,p}(0)=\psi_{\phi,p}(0)=0, \quad\psi^{\mathrm{cav}}_{\phi,p}(1)=\psi_{\phi,p}(1).
\end{equation*}
By Lemma~\ref{lem:properties-omega-psi}, $\psi_{\phi,p}$ is strictly increasing on $[0,1]$, so it attains its maximum uniquely at $1$. Therefore, applying \cite[Lemma~3]{francu2017new} yields that $\psi^{\mathrm{cav}}_{\phi,p}$ is strictly increasing on $(0,1)$. It remains only to include the two endpoints. For $u \in (0,1]$, 
\begin{equation*}
\psi^{\mathrm{cav}}_{\phi,p}(u) \ge \psi_{\phi,p}(u)>0=\psi^{\mathrm{cav}}_{\phi,p}(0).
\end{equation*}
Thus, $\psi^{\mathrm{cav}}_{\phi,p}$ is strictly increasing on $[0,1)$. Now, fix $u\in [0,1)$. Choose $w$ such that $u<w<1$. For every $v\in(w,1)$, the strict monotonicity on $[0,1)$ yields
\begin{equation*}
\psi^{\mathrm{cav}}_{\phi,p}(u)<\psi^{\mathrm{cav}}_{\phi,p}(w)<\psi^{\mathrm{cav}}_{\phi,p}(v).
\end{equation*}
Letting $v$ tend to $1^{-}$ and using continuity of $\psi^{\mathrm{cav}}_{\phi,p}$ at $1$, we obtain
\begin{equation*}
\psi^{\mathrm{cav}}_{\phi,p}(u) <\psi^{\mathrm{cav}}_{\phi,p}(w) \le \lim_{v \to 1^-}\psi^{\mathrm{cav}}_{\phi,p}(v)=\psi^{\mathrm{cav}}_{\phi,p}(1).
\end{equation*}
Thus, $\psi^{\mathrm{cav}}_{\phi,p}$ is strictly increasing on $[0,1]$.
\end{proof}
\subsection{Elementary properties of the Luxemburg norm}\label{sec:properties-Luxemburgnorm}
We record below a few short proofs, derived directly from the definition of the Luxemburg norm, of properties that will be used later. For further properties, see, for example, \cite[Chapter~13]{rubshtein2016foundations}. Recall that, given a measure space $(\Omega,m)$ and an Orlicz function $\phi$, the Luxemburg norm \cite{luxemburg1955banach} is defined by
$$
\|f\|_{\phi}\coloneqq\inf\left\{\lambda>0:
\int_\Omega\phi\left(\frac{|f|}{\lambda}\right)\mathrm{d}m\le 1\right\},
$$
where the underlying measure $m$ is omitted from the notation whenever it is clear from the context.
\begin{lemma}\label{lem:properties-luxemburgnorm}
Let $f$ and $g$ be measurable functions with finite Luxemburg norms.
\begin{enumerate}
    \item If $\left|f\right| \le \left|g\right|$ $m$-almost everywhere, then
    \begin{equation*}
    \|f\|_{\phi} \le \|g\|_{\phi}.
    \end{equation*}
    \item For every scalar $c\in \mathbb{R}$,
    \begin{equation*}
    \|cf\|_{\phi}=\left|c\right|\|f\|_{\phi}.
    \end{equation*}
\end{enumerate}
\end{lemma}
\begin{proof}
Let $\lambda>0$ be such that 
\begin{equation}\label{eq:forproof-lem:properties-luxemburg norm}
\int_{\Omega} \phi\left(\frac{\left|g\right|}{\lambda}\right)\mathrm{d}m\le 1.
\end{equation}
Since $\phi$ is increasing and $\left|f\right| \le \left|g\right|$ $m$-almost everywhere,
\begin{equation*}
\phi \left(\frac{\left|f\right|}{\lambda}\right) \le \phi\left(\frac{\left|g\right|}{\lambda}\right)
\end{equation*}
$m$-almost everywhere.
Hence, if $\lambda>0$ satisfies \eqref{eq:forproof-lem:properties-luxemburg norm}, then
\begin{equation*}
\int_{\Omega}\phi\left(\frac{\left|f\right|}{\lambda}\right)\mathrm{d}m \le 1.
\end{equation*}
Thus, every admissible $\lambda$ in the definition of $\|g\|_{\phi}$
is also admissible in the definition of $\|f\|_{\phi}$. Taking the infimum yields $\|f\|_{\phi}\le \|g\|_{\phi}$.

For the second assertion, the case $c=0$ is immediate. Assume $c \ne 0$. Then
\begin{align*}
\|cf\|_{\phi}&= \inf \left\{\lambda >0:\int_{\Omega} \phi\left(\frac{\left|c\right|\left|f\right|}{\lambda}\right)\mathrm{d}m\le 1\right\}\\
& = \left|c\right|\inf\left\{\eta>0:\int_{\Omega}\phi\left(\frac{\left|f\right|}{\eta}\right)\mathrm{d}m\le 1\right\}=\left|c\right|\|f\|_{\phi},
\end{align*}
where, in the second equality, we used the substitution $\lambda=|c|\eta$.
\end{proof}
\section{Proofs for Metric and Comparison Properties}
\label{app:proof-SOW-metric-comparison}
We first establish a simple regularity property of the projected
Orlicz--Wasserstein distance that will also be useful in the topological
analysis.

\begin{lemma}\label{lem:directional-lipschitz}
For $\rho\in\mathcal{P}_{\phi}(\mathbb{R}^{d})$, define
\begin{equation}\label{eq:momentorlicz}
m_{\phi}(\rho)\coloneqq \inf\left\{
\lambda>0:\int_{\mathbb{R}^{d}}
\phi\left(\frac{\|x\|}{\lambda}\right)\mathrm{d}\rho(x)\leq 1\right\}.
\end{equation}
Then, for any $\mu,\nu\in\mathcal{P}_{\phi}(\mathbb{R}^{d})$ and
$\theta,\eta\in\mathbb{S}^{d-1}$,
$$
\left|\operatorname{W}_{\phi}(\Pi^\theta_{\#}\mu,\Pi^\theta_{\#}\nu)-\operatorname{W}_{\phi}(\Pi^\eta_{\#}\mu,\Pi^\eta_{\#}\nu)\right|\leq\|\theta-\eta\|\bigl(m_{\phi}(\mu)+m_{\phi}(\nu)\bigr).
$$
\end{lemma}

\begin{proof}
By the triangle inequality for $\operatorname{W}_{\phi}$,
\begin{equation*}
\left|\operatorname{W}_{\phi}(\Pi^\theta_{\#}\mu,\Pi^\theta_{\#}\nu)-\operatorname{W}_{\phi}(\Pi^\eta_{\#}\mu,\Pi^\eta_{\#}\nu)\right|\leq
\operatorname{W}_{\phi}(\Pi^\theta_{\#}\mu,\Pi^\eta_{\#}\mu)+
\operatorname{W}_{\phi}(\Pi^\theta_{\#}\nu,\Pi^\eta_{\#}\nu).
\end{equation*}
Let $X\sim\mu$. Using the coupling induced by the pair
$(\langle\theta,X\rangle,\langle\eta,X\rangle)$ and Lemma~\ref{lem:properties-luxemburgnorm}, we obtain
\begin{equation*}
\operatorname{W}_{\phi}
\left(\Pi^\theta_{\#}\mu,\Pi^\eta_{\#}\mu\right)\le\left\|\langle\theta-\eta,X\rangle
\right\|_{\phi}\le \|\theta-\eta\|
\left\|\|X\|\right\|_{\phi}=\|\theta-\eta\|m_{\phi}(\mu).
\end{equation*}
The same argument for $\nu$ gives
\begin{equation*}
\operatorname{W}_{\phi}(\Pi^\theta_{\#}\nu,\Pi^\eta_{\#}\nu)\leq\|\theta-\eta\|m_{\phi}(\nu),
\end{equation*}
which proves the result.
\end{proof}

\begin{proof}[Proof of Theorem~\ref{thm:SOW-metric} and Proposition~\ref{prop:SOW-OW}]
Since $\operatorname{W}_{\phi}$ is a metric (see, e.g., \cite{guha2023excess}), the metricity assertion in Theorem~\ref{thm:SOW-metric} also follows directly from \cite{nadjahi2020statistical}, which shows that slicing preserves metricity. We nevertheless provide a direct proof here for completeness, together with the comparison inequality in Proposition~\ref{prop:SOW-OW}).

We first establish the comparison inequality in
Proposition~\ref{prop:SOW-OW}, which also proves that $\operatorname{SOW}_{\phi,p}$ is finite on
$\mathcal{P}_{\phi}(\mathbb{R}^{d})$.

For every
$\theta\in\mathbb{S}^{d-1}$,
$$|\langle\theta,x-y\rangle|\leq \|x-y\| \quad \text{for all } x,y \in \mathbb{R}^d.$$
Combining with the monotonicity of $\phi$, we obtain
$$\operatorname{W}_{\phi}(\Pi^\theta_{\#}\mu,\Pi^\theta_{\#}\nu)
\leq\operatorname{W}_{\phi}(\mu,\nu).
$$
Since $\sigma$ is a probability measure,
$$
\operatorname{SOW}_{\phi,p}(\mu,\nu)
\leq\operatorname{W}_{\phi}(\mu,\nu)<\infty.
$$
Non-negativity and symmetry follow directly from the corresponding
properties of $\operatorname{W}_{\phi}$.

For the triangle inequality, let
$\mu,\nu,\rho\in\mathcal{P}_{\phi}(\mathbb{R}^{d})$. For every
$\theta\in\mathbb{S}^{d-1}$,
$$\operatorname{W}_{\phi}(\Pi^\theta_{\#}\mu,\Pi^\theta_{\#}\rho)
\leq\operatorname{W}_{\phi}(\Pi^\theta_{\#}\mu,\Pi^\theta_{\#}\nu)
+\operatorname{W}_{\phi}(\Pi^\theta_{\#}\nu,\Pi^\theta_{\#}\rho).
$$
Applying Minkowski's inequality in
$L^p(\mathbb{S}^{d-1},\sigma)$ yields
$$\operatorname{SOW}_{\phi,p}(\mu,\rho)\leq \operatorname{SOW}_{\phi,p}(\mu,\nu)+\operatorname{SOW}_{\phi,p}(\nu,\rho).
$$

Clearly,
$\mu=\nu$ implies
$\operatorname{SOW}_{\phi,p}(\mu,\nu)=0$. It remains to show that $\operatorname{SOW}_{\phi,p}(\mu,\nu)=0$ implies $\mu=\nu$.

Suppose that $\operatorname{SOW}_{\phi,p}(\mu,\nu)=0$. Then $\operatorname{W}_{\phi}(\Pi^\theta_{\#}\mu,\Pi^\theta_{\#}\nu)=0$
for $\sigma$-almost every $\theta\in\mathbb{S}^{d-1}$.
By Lemma~\ref{lem:directional-lipschitz}, the map
$\theta\mapsto\operatorname{W}_{\phi}(\Pi^\theta_{\#}\mu,\Pi^\theta_{\#}\nu)$
is continuous. Since $\sigma$ has full support on
$\mathbb{S}^{d-1}$, it follows that
$$
\operatorname{W}_{\phi}(\Pi^\theta_{\#}\mu,\Pi^\theta_{\#}\nu)=0
\qquad
\text{for every }\theta\in\mathbb{S}^{d-1}.
$$
Since $\operatorname{W}_{\phi}$ is a metric (see, e.g., \cite[Lemma~3.3]{guha2023excess}),
$$
\Pi^\theta_{\#}\mu=\Pi^\theta_{\#}\nu
\qquad
\text{for every }\theta\in\mathbb{S}^{d-1}.
$$
The Cram\'er--Wold theorem then gives $\mu=\nu$.
\end{proof}

\begin{proof}[Proof of Proposition~\ref{prop:SOW-phi}]
Since $\phi_1(t)\leq\phi_2(t)$ for every $t\geq0$, \cite[Lemma~3.4]{guha2023excess} yields
$$\operatorname{W}_{\phi_1}(\alpha,\beta) \leq \operatorname{W}_{\phi_2}(\alpha,\beta)
$$
for every admissible pair $(\alpha,\beta)$.

Applying this inequality to each projected pair gives
$$\operatorname{W}_{\phi_1}
(\Pi^\theta_{\#}\mu,\Pi^\theta_{\#}\nu)\le\operatorname{W}_{\phi_2}
(\Pi^\theta_{\#}\mu,\Pi^\theta_{\#}\nu)
$$
for every $\theta\in\mathbb{S}^{d-1}$. Therefore,
$$
\operatorname{SOW}_{\phi_1,p}(\mu,\nu)\leq\operatorname{SOW}_{\phi_2,p}(\mu,\nu).
$$
\end{proof}
\section{Proofs and Counterexample for Topological Properties}
\subsection{Omitted Proofs for Topological Properties}\label{app:proof-SOW-topo}

We first establish a comparison between the radial Orlicz moment and the
SOW distance to the Dirac measure at the origin.

\begin{lemma}
\label{lem:moment-SOW}
For $\rho\in\mathcal{P}_{\phi}(\mathbb{R}^{d})$, let
$m_{\phi}(\rho)$ be defined in \eqref{eq:momentorlicz}. Then there exists
a constant $C_{d,p}>0$, depending only on $d$ and $p$, such that
$$
\operatorname{SOW}_{\phi,p}(\rho,\delta_0)\le
m_{\phi}(\rho)\le
C_{d,p}\operatorname{SOW}_{\phi,p}(\rho,\delta_0).
$$
\end{lemma}

\begin{proof}
Let $X\sim\rho$. For every $\theta\in\mathbb{S}^{d-1}$, the monotonicity of the Luxemburg norm in Lemma~\ref{lem:properties-luxemburgnorm} gives
$$\operatorname{W}_{\phi}(\Pi^\theta_{\#}\rho,\delta_0)=
\|\langle\theta,X\rangle\|_{\phi}\le
\|\|X\|\|_{\phi}=m_{\phi}(\rho).
$$
Integrating over $\theta$ gives
$$\operatorname{SOW}_{\phi,p}(\rho,\delta_0)\le m_{\phi}(\rho).
$$

For the reverse inequality, define
$$M \coloneqq \max_{\theta\in\mathbb{S}^{d-1}}
\|\langle\theta,X\rangle\|_{\phi}.
$$
The maximum exists by continuity of the map
$\theta\mapsto\|\langle\theta,X\rangle\|_{\phi}$. Indeed, by Lemma~\ref{lem:directional-lipschitz}, applied with
$\mu=\rho$ and $\nu=\delta_0$, the map
$$
\theta\mapsto \operatorname{W}_{\phi}(\Pi^\theta_{\#}\rho,\delta_0)= \|\langle\theta,X\rangle\|_{\phi}
$$
is continuous on $\mathbb{S}^{d-1}$. Since $\mathbb{S}^{d-1}$ is compact,
the maximum is attained.

Let $e_1,\ldots,e_d$ be the canonical basis of $\mathbb{R}^{d}$.
Since
$$ \|X\|\le \sum_{j=1}^{d}|X_j|,
$$
the triangle inequality for the Luxemburg norm gives
$$m_{\phi}(\rho)=\|\|X\|\|_{\phi}\le\sum_{j=1}^{d}\|X_j\|_{\phi}
\le  dM.
$$

Choose $\theta_0\in\mathbb{S}^{d-1}$ such that $$\|\langle\theta_0,X\rangle\|_{\phi}=M.$$
For every $\theta\in\mathbb{S}^{d-1}$ satisfying
$\|\theta-\theta_0\|\le 1/(2d)$,
\begin{align*}
\|\langle\theta,X\rangle\|_{\phi}&\ge
\|\langle\theta_0,X\rangle\|_{\phi}-
\|\langle\theta-\theta_0,X\rangle\|_{\phi}\\
&\ge
M-\|\theta-\theta_0\|\,m_{\phi}(\rho)\\
&\ge
\frac{M}{2}.
\end{align*}

Let
$$ c_d\coloneqq
\sigma\left(\left\{\theta\in\mathbb{S}^{d-1}:
\|\theta-\theta_0\|\le\frac{1}{2d}\right\}
\right)>0.
$$
By rotation invariance of the uniform measure $\sigma$, $c_d$ does not depend on $\theta_0$. Hence,
$$
\operatorname{SOW}_{\phi,p}^{p}(\rho,\delta_0) \ge
c_d\left(\frac{M}{2}\right)^p.
$$
Therefore,
$$M \le 2c_d^{-1/p} \operatorname{SOW}_{\phi,p}(\rho,\delta_0),
$$
and consequently,
$$m_{\phi}(\rho)\le 2dc_d^{-1/p}\operatorname{SOW}_{\phi,p}(\rho,\delta_0).
$$
This proves the result.
\end{proof}

\begin{proof}[Proof of Proposition~\ref{prop:SOW-maxSOW}]
Define $M_n\coloneqq 
\displaystyle \sup_{\theta\in\mathbb{S}^{d-1}}\operatorname{W}_{\phi}
\left(\Pi^\theta_{\#}\mu_n,
\Pi^\theta_{\#}\mu\right).$ Since $\sigma$ is a probability measure,

$$\operatorname{SOW}_{\phi,p}(\mu_n,\mu) \le M_n.$$
Hence,
$$ M_n\to0
\quad\Longrightarrow\quad \operatorname{SOW}_{\phi,p}(\mu_n,\mu)\to0.
$$

We now prove the converse. Suppose that $\operatorname{SOW}_{\phi,p}(\mu_n,\mu)\to0.$
\begin{align*}
m_{\phi}(\mu_n)\le
C_{d,p}\operatorname{SOW}_{\phi,p}(\mu_n,\delta_0)\le C_{d,p}\left[
\operatorname{SOW}_{\phi,p}(\mu_n,\mu)+\operatorname{SOW}_{\phi,p}(\mu,\delta_0)\right].
\end{align*}
Since
$\operatorname{SOW}_{\phi,p}(\mu_n,\mu)\to0$,
the sequence
$\left(\operatorname{SOW}_{\phi,p}(\mu_n,\mu)\right)_{n\ge1}$
is bounded. Therefore,
$$
\sup_{n\ge1}m_{\phi}(\mu_n)<\infty.
$$

Since $m_{\phi}(\mu)<\infty$, define
\begin{equation*}
L\coloneqq 1+ \sup_{n\ge1}m_{\phi}(\mu_n)+m_{\phi}(\mu) <\infty.
\end{equation*}
Then Lemma~\ref{lem:directional-lipschitz} gives, for every
$\theta,\eta\in\mathbb{S}^{d-1}$ and every $n\ge1$,
\begin{equation*}
\left|\operatorname{W}_{\phi}(\Pi^\theta_{\#}\mu_n,\Pi^\theta_{\#}\mu)-
\operatorname{W}_{\phi}(\Pi^\eta_{\#}\mu_n,\Pi^\eta_{\#}\mu)
\right|\le L\|\theta-\eta\|.
\end{equation*}
Thus, the projected distance functions are uniformly Lipschitz with
respect to the projection direction, with a constant independent of $n$.

Suppose, by contradiction, that $M_n$ does not converge to zero.
Then there exist $\varepsilon>0$ and a subsequence
$(n_k)_{k\ge1}$ such that
$$M_{n_k}\ge\varepsilon \qquad \text{for all }k\ge1.
$$
For each $k$, the map $\theta\mapsto
\operatorname{W}_{\phi}(\Pi^\theta_{\#}\mu_{n_k},\Pi^\theta_{\#}\mu)$ is continuous by Lemma~\ref{lem:directional-lipschitz}.
Since $\mathbb{S}^{d-1}$ is compact, there exists
$\theta_k\in\mathbb{S}^{d-1}$ such that
\begin{equation*}
\operatorname{W}_{\phi}\left( \Pi^{\theta_k}_{\#}\mu_{n_k},
\Pi^{\theta_k}_{\#}\mu \right) = M_{n_k}.
\end{equation*}
Set $$
r\coloneqq \min\left\{ 1,\frac{\varepsilon}{2L} \right\}.
$$
For every $\theta\in\mathbb{S}^{d-1}$ satisfying
$\|\theta-\theta_k\|\le r$, we have
\begin{align*}
\operatorname{W}_{\phi}\left(\Pi^\theta_{\#}\mu_{n_k},\Pi^\theta_{\#}\mu\right)
&\ge\operatorname{W}_{\phi}\left(
\Pi^{\theta_k}_{\#}\mu_{n_k},\Pi^{\theta_k}_{\#}\mu\right)-L\|\theta-\theta_k\|\\
&\ge\varepsilon-Lr\\
&\ge\frac{\varepsilon}{2}.
\end{align*}
By rotation invariance of $\sigma$, the spherical cap
$$\left\{\theta\in\mathbb{S}^{d-1}:
\|\theta-\theta_k\|\le r\right\}
$$
has measure $c_{d,r}>0$, where $c_{d,r}$ does not depend on $k$.
Therefore,
\begin{align*}
\operatorname{SOW}_{\phi,p}^{p}(\mu_{n_k},\mu)
&=
\int_{\mathbb{S}^{d-1}}\operatorname{W}_{\phi}^{p}
\left(\Pi^\theta_{\#}\mu_{n_k},\Pi^\theta_{\#}\mu\right)\mathrm{d}\sigma(\theta)
\\
&\ge c_{d,r}\left(\frac{\varepsilon}{2}\right)^p.
\end{align*}
The right-hand side is strictly positive and independent of $k$,
which contradicts
$$
\operatorname{SOW}_{\phi,p}(\mu_{n_k},\mu)\xrightarrow[]{k\to+\infty}0.$$
Hence,
$$
M_n\xrightarrow[]{n\to+\infty}0,
$$
which completes the proof.
\end{proof}
We next prove Corollary~\ref{cor:SOW-weak}. The main ingredient is the
following comparison between the Orlicz--Wasserstein distance and the
$1$-Wasserstein distance.
\begin{lemma}
\label{lem:W1-Wphi}
For any $\alpha,\beta\in\mathcal{P}_{\phi}(\mathbb{R}^{d})$,
$$\operatorname{W}_1(\alpha,\beta)\le\phi^{-1}(1)\operatorname{W}_{\phi}(\alpha,\beta).
$$
\end{lemma}
\begin{proof}
Let $\pi\in\Pi(\alpha,\beta)$ and let $\lambda>0$ satisfy
$$\int_{\mathbb{R}^{d}\times\mathbb{R}^{d}}\phi\left(\frac{\|x-y\|}{\lambda}\right)\mathrm{d}\pi(x,y)\le 1.
$$
Since $\pi$ is a probability measure, Jensen's inequality gives
$$
\phi\left(\frac{1}{\lambda}\int_{\mathbb{R}^{d}\times\mathbb{R}^{d}}
\|x-y\|\mathrm{d}\pi(x,y)\right) \le 1.
$$
Since $\phi$ is strictly increasing,
$$
\int_{\mathbb{R}^{d}\times\mathbb{R}^{d}}\|x-y\|
\mathrm{d}\pi(x,y)\le \phi^{-1}(1)\lambda.
$$
Taking the infimum first over admissible $\lambda$ and then over
$\pi\in\Pi(\alpha,\beta)$ yields
\begin{equation*}
\operatorname{W}_1(\alpha,\beta)\le\phi^{-1}(1)\operatorname{W}_{\phi}(\alpha,\beta).
\end{equation*}
\end{proof}
\begin{proof}[Proof of Corollary~\ref{cor:SOW-weak}]
Assume that $\operatorname{SOW}_{\phi,p}(\mu_n,\mu)\to0.$ By Proposition~\ref{prop:SOW-maxSOW},
\begin{equation*}
\sup_{\theta\in\mathbb{S}^{d-1}}\operatorname{W}_{\phi}\left(
\Pi^\theta_{\#}\mu_n,\Pi^\theta_{\#}\mu\right)\to0.
\end{equation*}
Hence, for every $\theta\in\mathbb{S}^{d-1}$,
\begin{equation*}
\operatorname{W}_{\phi}\left(\Pi^\theta_{\#}\mu_n,
\Pi^\theta_{\#}\mu\right)\to0.
\end{equation*}
By Lemma~\ref{lem:W1-Wphi},
\begin{equation*}
\operatorname{W}_1\left(\Pi^\theta_{\#}\mu_n,
\Pi^\theta_{\#}\mu\right)\to0.
\end{equation*}
Hence, by \cite[Theorem~6.9]{villani2008optimal},
\begin{equation*}
\Pi^\theta_{\#}\mu_n\rightharpoonup\Pi^\theta_{\#}\mu
\qquad\text{for every }\theta\in\mathbb{S}^{d-1}.
\end{equation*}
The Cram\'er--Wold theorem
\cite[Theorem~3.10.6]{durrett2019probability}
then implies
\begin{equation*}
\mu_n\rightharpoonup\mu.
\end{equation*}
\end{proof}
We now turn to Proposition~\ref{prop:SOW-compact-topology}. The key
ingredient is the following comparison between the Orlicz--Wasserstein
distance and the $1$-Wasserstein distance on bounded supports.
\begin{lemma}\label{lem:Wphi-W1-compact}
Let $D\in(0,\infty)$, and let $\alpha$ and $\beta$ be probability measures supported on a common interval of length at most $D$. Then
$$
\operatorname{W}_{\phi}(\alpha,\beta)\le D\,\omega_{\phi}\left(
\frac{\operatorname{W}_1(\alpha,\beta)}{D}\right),
$$
where $\omega_{\phi}$ is introduced in \eqref{eq:fundamentalfunction}.
\end{lemma}
\begin{proof}
If $\operatorname{W}_1(\alpha,\beta)=0$, then $\alpha=\beta$ and the result is immediate.
Hence, assume that $\operatorname{W}_1(\alpha,\beta)>0$.

Let $\pi^\star\in\Pi(\alpha,\beta)$ be an optimal coupling for
$\operatorname{W}_1(\alpha,\beta)$. Since $\alpha$ and $\beta$ are supported on a common
interval of length at most $D$,
\begin{equation*}
|x-y|\le D\qquad \pi^\star\text{-almost surely}.
\end{equation*}
Since $0\le z\le D$, we have $z/D\in[0,1]$.
By the convexity of $\phi$ and $\phi(0)=0$,
\begin{equation*}
\phi\left(\frac{z}{t}\right)=\phi\left(\frac{z}{D}\frac{D}{t}
+\left(1-\frac{z}{D}\right)0
\right)\le \frac{z}{D}\phi\left(\frac{D}{t}\right).
\end{equation*}
Therefore,
\begin{align*}
\int_{\mathbb{R}\times\mathbb{R}}
\phi\left(\frac{|x-y|}{t}\right)
\mathrm{d}\pi^\star(x,y)
&\le\frac{\phi(D/t)}{D}
\int_{\mathbb{R}\times\mathbb{R}}|x-y|\,\mathrm{d}\pi^\star(x,y)
\\
&=\frac{\phi(D/t)}{D}\operatorname{W}_1(\alpha,\beta).
\end{align*}
Choosing
$$t=
\frac{D}{\phi^{-1}\left(D/\operatorname{W}_1(\alpha,\beta)\right)
},$$
we obtain
$$
\int_{\mathbb{R}\times\mathbb{R}}\phi\left(\frac{|x-y|}{t}\right)
\mathrm{d}\pi^\star(x,y) \le1.
$$
Hence, by the definition of $\operatorname{W}_\phi$,
\begin{equation*}
\operatorname{W}_\phi(\alpha,\beta)\le
\frac{D}{\phi^{-1}\left(D/\operatorname{W}_1(\alpha,\beta)\right)
}= D\omega_\phi\left(\frac{\operatorname{W}_1(\alpha,\beta)}{D}
\right).
\end{equation*}
\end{proof}
\begin{proof}[Proof of Proposition~\ref{prop:SOW-compact-topology}]
The implication
$$\operatorname{SOW}_{\phi,p}(\mu_n,\mu)\to0
\quad\Longrightarrow\quad
\mu_n\rightharpoonup\mu$$
follows from Corollary~\ref{cor:SOW-weak}. We only need to prove the
reverse implication.

Assume that 
$$\mu_n\rightharpoonup\mu,$$
where $(\mu_n)_{n\ge1}\subset\mathcal{P}(K)$ and $\mu\in\mathcal{P}(K)$ for some compact set $K\subset\mathbb{R}^{d}$. Let $$D\coloneqq \operatorname{diam}(K).
$$
If $D=0$, then $K$ is a singleton, so $\mu_n=\mu$ for every $n$ and
the conclusion is immediate. Hence, we assume that $D>0$. Fix $\theta\in\mathbb{S}^{d-1}$. Since $\Pi^\theta$ is continuous,
the continuous mapping theorem (see, e.g., \cite[Lemma~4.3]{kallenberg1997foundations}) gives
\begin{equation*}
\Pi^\theta_{\#}\mu_n  \rightharpoonup \Pi^\theta_{\#}\mu.
\end{equation*}
Moreover, all projected measures are supported on a common compact interval
of length at most $D$. Therefore, by
\cite[Theorem~6.9]{villani2008optimal},
$$\operatorname{W}_1\left(\Pi^\theta_{\#}\mu_n, \Pi^\theta_{\#}\mu \right) \to0.$$
Applying Lemma~\ref{lem:Wphi-W1-compact}, we obtain
$$\operatorname{W}_{\phi}\left(\Pi^\theta_{\#}\mu_n,\Pi^\theta_{\#}\mu\right)
\le D\omega_{\phi}\left(
\frac{\operatorname{W}_1(\Pi^\theta_{\#}\mu_n,\Pi^\theta_{\#}\mu)}{D}\right).
$$
By Lemma~\ref{lem:properties-omega-psi}, $\omega_{\phi}(u)\to0$ as $u\to0$. Therefore, 
$$\operatorname{W}_{\phi}\left(\Pi^\theta_{\#}\mu_n,
\Pi^\theta_{\#}\mu\right) \to0\text{ for every }\theta\in\mathbb{S}^{d-1}.$$
Furthermore, since all projected measures are supported on an interval of
length at most $D$,
\begin{equation*}
\operatorname{W}_{\phi}\left(\Pi^\theta_{\#}\mu_n,
\Pi^\theta_{\#}\mu\right) \le \frac{D}{\phi^{-1}(1)},
\end{equation*}
for every $n\ge1$ and every $\theta\in\mathbb{S}^{d-1}$.
Hence,
\begin{equation*}
\operatorname{W}_{\phi}^{p}\left(\Pi^\theta_{\#}\mu_n,\Pi^\theta_{\#}\mu\right) \le
\left(\frac{D}{\phi^{-1}(1)}
\right)^p.
\end{equation*}
Since $\sigma$ is a probability measure, the right-hand side is integrable.
Therefore, by the dominated convergence theorem (see, e.g., \cite[Theorem~1.21]{kallenberg1997foundations}),
\begin{align*}
\operatorname{SOW}_{\phi,p}^{p}(\mu_n,\mu)=\int_{\mathbb{S}^{d-1}}\operatorname{W}_{\phi}^{p}\left(\Pi^\theta_{\#}\mu_n,
\Pi^\theta_{\#}\mu\right)\mathrm{d}\sigma(\theta)
\to0.
\end{align*}
Hence, $\operatorname{SOW}_{\phi,p}(\mu_n,\mu)\to0$.
This completes the proof.
\end{proof}
\subsection{Counterexample to the Converse of Corollary~\ref{cor:SOW-weak}}
\label{app:SOW-weak-counterexample}
We provide a counterexample showing that weak convergence does not, in
general, imply convergence in $\operatorname{SOW}_{\phi,p}$, even in
dimension one.

Let
$$\mu=\delta_0,\qquad\mu_n=
\left(1-\frac{1}{n}\right)\delta_0+
\frac{1}{n}\delta_{\phi^{-1}(n)},
\qquad n\ge1.
$$
We first show that $\mu_n\rightharpoonup\mu.$
Indeed, for every bounded continuous function $f:\mathbb{R}\to\mathbb{R}$,
$$ \int_{\mathbb{R}} f\,\mathrm{d}\mu_n =
\left(1-\frac{1}{n}\right)f(0)+\frac{1}{n}f\bigl(\phi^{-1}(n)\bigr),
$$
while
$$\int_{\mathbb{R}} f\,\mathrm{d}\mu=f(0).
$$
Therefore,
\begin{align*}
\left|\int_{\mathbb{R}} f\,\mathrm{d}\mu_n-\int_{\mathbb{R}} f\,\mathrm{d}\mu\right|
&=\frac{1}{n}\left|f\bigl(\phi^{-1}(n)\bigr)-f(0)\right|
\\
&\le\frac{2\|f\|_{\infty}}{n}\longrightarrow 0.
\end{align*}
Hence, $\mu_n\rightharpoonup\mu$.

On the other hand,
\begin{align*}
\operatorname{W}_\phi(\mu_n,\mu)
&=\inf\left\{\lambda>0: \frac{1}{n}\phi\left(
\frac{\phi^{-1}(n)}{\lambda}\right)\le1\right\}
\\
&=\inf\left\{\lambda>0:
\phi\left(\frac{\phi^{-1}(n)}{\lambda}
\right)\le n\right\}.
\end{align*}
Since $\phi$ is strictly increasing,
$$\phi\left(\frac{\phi^{-1}(n)}{\lambda}\right)\le n
=\phi\bigl(\phi^{-1}(n)\bigr)
$$
if and only if
$$\frac{\phi^{-1}(n)}{\lambda} \le \phi^{-1}(n),
$$
or equivalently, $\lambda\ge1$.
Thus,
$$\operatorname{W}_\phi(\mu_n,\mu)= 1 \text{ for every }n\ge1.
$$

Finally, when $d=1$, $\mathbb{S}^{0}=\{-1,+1\}$. For either projection direction $\theta\in\{-1,+1\}$,
$$\operatorname{W}_\phi \left(\Pi^\theta_{\#}\mu_n,
\Pi^\theta_{\#}\mu\right)=1.
$$
Therefore,
$$\operatorname{SOW}_{\phi,p}^{p}(\mu_n,\mu)=\frac{1}{2}\left(1^p+1^p\right)=1,$$
and hence
$$\operatorname{SOW}_{\phi,p}(\mu_n,\mu)=1 \text{ for every }n\ge1.$$

Thus, although $\mu_n\rightharpoonup\mu$, we have
$\operatorname{SOW}_{\phi,p}(\mu_n,\mu)\not\to0.$
This shows that the converse of
Corollary~\ref{cor:SOW-weak} does not hold in general.
\section{Proofs for Statistical Properties} \label{app:proof-SOW-statistics}

We first establish a link between the one-dimensional
Orlicz--Wasserstein distance and the uniform distance between the
corresponding cumulative distribution functions.

\begin{lemma}\label{lem:Wphi-CDF}
Let $D>0$, and let $\alpha$ and $\beta$ be probability measures on
$\mathbb{R}$ supported on a common interval of length at most $D$.
Let $F_\alpha$ and $F_\beta$ denote their respective cumulative
distribution functions. Then
$$\operatorname{W}_\phi(\alpha,\beta)\le
D\omega_\phi\left(\|F_\alpha-F_\beta\|_\infty\right),
$$
where
$$\|F_\alpha-F_\beta\|_\infty\coloneqq \sup_{x\in\mathbb{R}}|F_\alpha(x)-F_\beta(x)|.
$$
\end{lemma}
\begin{proof}
By Lemma~\ref{lem:Wphi-W1-compact},
\begin{equation*}
\operatorname{W}_\phi(\alpha,\beta)\le D\omega_\phi\left(\frac{\operatorname{W}_1(\alpha,\beta)}{D}
\right).
\end{equation*}
On the other hand, the one-dimensional CDF representation of $\operatorname{W}_1$
\cite[Theorem~2.9]{bobkov2019one} gives
\begin{equation*}
\operatorname{W}_1(\alpha,\beta)=\int_{\mathbb{R}}|F_\alpha(x)-F_\beta(x)|\mathrm{d}x.
\end{equation*}
Since $\alpha$ and $\beta$ are supported on a common interval of length
at most $D$,
\begin{equation*}
\operatorname{W}_1(\alpha,\beta)\le D \|F_\alpha-F_\beta\|_\infty.
\end{equation*}
Since $\omega_\phi$ is increasing by Lemma~\ref{lem:properties-omega-psi}, we conclude that
\begin{equation*}
\operatorname{W}_\phi(\alpha,\beta)\le D\omega_\phi \left( \|F_\alpha-F_\beta\|_\infty
\right).
\end{equation*}
\end{proof}
We next derive a moment bound for the fundamental function $\omega_\phi$ under a sub-Gaussian tail condition.
\begin{lemma}\label{lem:omega-DKW}
Let $n \in \mathbb{N}$ with $n\ge 1$ and let $0\le\Delta_n\le1$ be a random variable satisfying
$$ \mathbb{P}(\Delta_n>u) \le
2\exp(-2nu^2), \qquad u\ge0.
$$
Then, for every $p\in[1,\infty)$, there exists a constant $C_p>0$,
depending only on $p$, such that
$$\left[\mathbb{E}\omega_\phi(\Delta_n)^p
\right]^{1/p}\le C_p\,\omega_\phi(n^{-1/2}).
$$
\end{lemma}
\begin{proof}
We use a dyadic decomposition. Define
$$
E_\ast\coloneqq  \left\{ \Delta_n\le n^{-1/2}
\right\},
$$
and, for $k\ge0$,
$$ E_k\coloneqq  \left\{
2^k n^{-1/2} < \Delta_n \le \min\left\{
2^{k+1}n^{-1/2},1 \right\}
\right\}.
$$
On $E_\ast$, since $\omega_{\phi}$ is increasing by Lemma~\ref{lem:properties-omega-psi},
$$
\omega_\phi(\Delta_n) \le \omega_\phi(n^{-1/2}),
$$
and therefore
\begin{equation*}
\mathbb{E} \left[ \omega_\phi(\Delta_n)^p
\mathbbm{1}_{E_\ast} \right] \le
\omega_\phi(n^{-1/2})^p.
\end{equation*}
Fix $k\ge0$. On $E_k$, we claim that
\begin{equation*}
\omega_\phi(\Delta_n) \le 2^{k+1}\omega_\phi(n^{-1/2}).
\end{equation*}
Indeed, if $2^{k+1}n^{-1/2}\le 1$, then this follows directly from Lemma~\ref{lem:properties-omega-psi}. We now assume that $2^{k+1}n^{-1/2}>1$. Applying Lemma~\ref{lem:properties-omega-psi} (with $a=n^{1/2}$ and $u=n^{-1/2}$), we obtain
\begin{equation*}
\omega_\phi(\Delta_n) \le \omega_\phi(1)= \omega_\phi \left( n^{1/2}n^{-1/2}\right) \le n^{1/2}\omega_\phi(n^{-1/2})< 2^{k+1}\omega_\phi(n^{-1/2}).
\end{equation*}
Hence,
\begin{align*}
\mathbb{E} \left[ \omega_\phi(\Delta_n)^p
\mathbbm{1}_{E_k} \right]&\le 2^{p(k+1)}
\omega_\phi(n^{-1/2})^p \mathbb{P}
\left( \Delta_n>2^k n^{-1/2} \right)\\
&\le 2^{p(k+1)+1} \omega_\phi(n^{-1/2})^p \exp(-2\cdot4^k).
\end{align*}
Since $E_\ast,E_0,E_1,\ldots$ form a partition of the sample space,
\begin{align*}
\mathbb{E}\omega_\phi(\Delta_n)^p&\le \omega_\phi(n^{-1/2})^p\left[ 1+ 2\sum_{k\ge0}
2^{p(k+1)} \exp(-2\cdot4^k)\right].
\end{align*}
The series is finite and depends only on $p$. Taking the $p$-th root
completes the proof.
\end{proof}
We are now ready to prove the one-sample empirical convergence bound.
\begin{proof}[Proof of Theorem~\ref{thm:SOW-one-sample}]
Fix $\theta\in\mathbb{S}^{d-1}$. Since $\mu$ is supported on a set of diameter at most $D$, the projected measure $\Pi^\theta_{\#}\mu$ is supported on an interval of length at most $D$.

Let $F_\theta$ denote the CDF of $\Pi^\theta_{\#}\mu$, and let
$\widehat{F}_{n,\theta}$ denote the empirical CDF associated with
$\Pi^\theta_{\#}\widehat{\mu}_n$. By the
Dvoretzky--Kiefer--Wolfowitz inequality 
\cite{massart1990tight},
\begin{equation*}
\mathbb{P} \left(\|\widehat{F}_{n,\theta}-F_\theta\|_\infty>u\right) \le 2\exp(-2nu^2),\quad u\ge0.
\end{equation*}
Since $\|\widehat{F}_{n,\theta}-F_\theta\|_\infty\in[0,1]$, Lemma~\ref{lem:omega-DKW} yields
\begin{equation*}
\left[ \mathbb{E} \omega_\phi
\left( \|\widehat{F}_{n,\theta}-F_\theta\|_\infty\right)^p \right]^{1/p}
\le C_p\omega_\phi(n^{-1/2}).
\end{equation*}
Applying Lemma~\ref{lem:Wphi-CDF}, we obtain
$$\mathbb{E}\operatorname{W}_\phi^p
\left(\Pi^\theta_{\#}\widehat{\mu}_n,
\Pi^\theta_{\#}\mu \right)\le C_p^pD^p \omega_\phi(n^{-1/2})^p.
$$
Integrating over $\theta$ and using Tonelli's theorem,
\begin{align*}
\mathbb{E} \operatorname{SOW}_{\phi,p}^p
\left(\widehat{\mu}_n,\mu
\right)
&=\int_{\mathbb{S}^{d-1}}\mathbb{E} \operatorname{W}_\phi^p \left( \Pi^\theta_{\#}\widehat{\mu}_n,\Pi^\theta_{\#}\mu
\right)\mathrm{d}\sigma(\theta)
\\&\le
C_p^pD^p\omega_\phi(n^{-1/2})^p.
\end{align*}
Taking the $p$-th root gives
$$\left[\mathbb{E}\operatorname{SOW}_{\phi,p}^p
\left(\widehat{\mu}_n,\mu\right)
\right]^{1/p}\le C_pD\,\omega_\phi(n^{-1/2}),
$$
which proves the theorem.
\end{proof}
We finally prove the two-sample bound for the SOW distance.
\begin{proof}[Proof of Corollary~\ref{cor:SOW-two-sample}]
Once we have established Theorem~\ref{thm:SOW-one-sample}, this corollary can be deduced directly by applying \cite[Theorem~5]{nadjahi2020statistical}. However, we provide the proof for completeness.

By Jensen's inequality and
Theorem~\ref{thm:SOW-one-sample},
$$
\mathbb{E} \operatorname{SOW}_{\phi,p}
\left(\widehat{\mu}_n,\mu
\right)\le\left[
\mathbb{E}\operatorname{SOW}_{\phi,p}^p
\left(\widehat{\mu}_n,\mu
\right)\right]^{1/p}
\le C_pD \omega_\phi(n^{-1/2}).
$$
Similarly,
$$
\mathbb{E}\operatorname{SOW}_{\phi,p} \left(
\widehat{\nu}_m,\nu \right) \le C_pD\omega_\phi(m^{-1/2}).
$$
Since $\operatorname{SOW}_{\phi,p}$ is a metric,
$$ \left| \operatorname{SOW}_{\phi,p}
\left( \widehat{\mu}_n,\widehat{\nu}_m
\right)-\operatorname{SOW}_{\phi,p}(\mu,\nu)
\right|\le\operatorname{SOW}_{\phi,p}
\left(\widehat{\mu}_n,\mu
\right)+\operatorname{SOW}_{\phi,p}
\left(\widehat{\nu}_m,\nu
\right).
$$
Taking expectations and combining the preceding bounds gives
$$ \mathbb{E} \left|
\operatorname{SOW}_{\phi,p}\left(
\widehat{\mu}_n,\widehat{\nu}_m\right)-
\operatorname{SOW}_{\phi,p}(\mu,\nu)\right|
\le C_pD \left[ \omega_\phi(n^{-1/2})
+\omega_\phi(m^{-1/2})
\right].
$$
This completes the proof.
\end{proof}

\section{Proof of Theorem~\ref{thm:powerSOW-two-sample}}\label{app:proof-SOW-power-twosample}
We first establish several auxiliary results for the powered Luxemburg
functional. These results will be used to control its variation under
$L^1$ perturbations.

We begin by defining
\begin{equation}\label{eq:psisharp}
\psi_{\phi,p}^{\sharp}(u)\coloneqq  u\sup_{u\le v\le1}\frac{\psi_{\phi,p}(v)}{v}, \qquad u\in(0,1],
\end{equation}
with $\psi_{\phi,p}^{\sharp}(0)\coloneqq 0$, where
$\psi_{\phi,p}$ is defined in \eqref{eq:psiphip}.

We record some elementary properties of
$\psi_{\phi,p}^{\sharp}$.
\begin{lemma}\label{lem:psi-sharp-properties}
The function $\psi_{\phi,p}^{\sharp}$ is increasing on $[0,1]$,
the map
$$u\longmapsto
\frac{\psi_{\phi,p}^{\sharp}(u)}{u}$$
is decreasing on $(0,1]$, and
$$ \psi_{\phi,p}^{\sharp}(u) \ge \psi_{\phi,p}(u),\quad u\in[0,1].$$
\end{lemma}
\begin{proof}
For $u=0$, both sides are equal to zero. For $u\in(0,1]$,
$$\frac{\psi_{\phi,p}^{\sharp}(u)}{u} = \sup_{u\le v\le1} \frac{\psi_{\phi,p}(v)}{v},
$$
which is decreasing on $(0,1]$. Moreover, choosing $v=u$ in
the supremum gives
$$
\psi_{\phi,p}^{\sharp}(u) \ge \psi_{\phi,p}(u).
$$
It remains to prove that $\psi_{\phi,p}^{\sharp}$ is increasing.
Let $0<u_1<u_2\le1$. For $v\in[u_1,u_2]$, since $\psi_{\phi,p}$ is increasing by Lemma~\ref{lem:properties-omega-psi},
\begin{equation*}
\frac{u_1}{v}\psi_{\phi,p}(v) \le \psi_{\phi,p}(v) \le \psi_{\phi,p}(u_2) \le \psi_{\phi,p}^{\sharp}(u_2).
\end{equation*}
For $v\in[u_2,1]$,
$$ \frac{u_1}{v}\psi_{\phi,p}(v) \le
\frac{u_2}{v}\psi_{\phi,p}(v) \le \psi_{\phi,p}^{\sharp}(u_2).
$$
Taking the supremum over $v\in[u_1,1]$ yields
$$\psi_{\phi,p}^{\sharp}(u_1) \le \psi_{\phi,p}^{\sharp}(u_2).
$$
The case $u_1=0$ follows from
$\psi_{\phi,p}^{\sharp}(0)=0$.
\end{proof}
To connect $\psi_{\phi,p}^{\sharp}$ with the modulus appearing in Theorem~\ref{thm:powerSOW-two-sample}, we first show that $\psi_{\phi,p}$ and $\psi_{\phi,p}^{\sharp}$ have the same least concave majorant on $[0,1]$.
\begin{lemma}\label{lem:same-leastconcavemajorant}
The functions $\psi_{\phi,p}$ defined in \eqref{eq:psiphip} and
$\psi_{\phi,p}^{\sharp}$ defined in \eqref{eq:psisharp} have the same
least concave majorant.
\end{lemma}
\begin{proof}
By Lemma~\ref{lem:psi-sharp-properties}, $\psi_{\phi,p} \le \psi^{\sharp}_{\phi,p}$ on $[0,1]$. Hence, every concave function dominating $\psi_{\phi,p}^{\sharp}$
also dominates $\psi_{\phi,p}$.

Let $g$ be a concave function on $[0,1]$ such that
\begin{equation*}
g(u)\ge \psi_{\phi,p}(u), \quad u\in[0,1].
\end{equation*}
Since $\psi_{\phi,p}(0)=0$, we have $g(0)\ge 0$. For any
$0<u\le v\le 1$, we write 
\begin{equation*}
u=\frac{u}{v}v+\left(1-\frac{u}{v}\right)0.
\end{equation*}
By the concavity of $g$,
\begin{equation*}
g(u)\ge\frac{u}{v}g(v)
+\left(1-\frac{u}{v}\right)g(0)\ge
\frac{u}{v}\psi_{\phi,p}(v).
\end{equation*}
Taking the supremum over $v\in[u,1]$ gives
\begin{equation*}
g(u)\ge u\sup_{u\le v\le 1}
\frac{\psi_{\phi,p}(v)}{v}=
\psi_{\phi,p}^{\sharp}(u).
\end{equation*}
At $u=0$, we also have
\begin{equation*}
g(0)\ge 0=\psi_{\phi,p}^{\sharp}(0).
\end{equation*}
Consequently, $g(u) \ge \psi^{\sharp}_{\phi,p}\left(u\right) \forall u \in [0,1].$ Therefore, every concave function dominating $\psi_{\phi,p}$ also dominates $\psi_{\phi,p}^{\sharp}$.

As a result, $\psi_{\phi,p}$ and $\psi_{\phi,p}^{\sharp}$ are dominated by
exactly the same concave functions, and consequently have the same
least concave majorant.
\end{proof}
We now relate $\psi_{\phi,p}^{\sharp}$ to
$\psi_{\phi,p}^{\mathrm{cav}}$ used in
Theorem~\ref{thm:powerSOW-two-sample}. By Lemma~\ref{lem:same-leastconcavemajorant},
$\psi_{\phi,p}^{\mathrm{cav}}$ is also the least concave majorant of
$\psi_{\phi,p}^{\sharp}$. Moreover, by Lemma~\ref{lem:psi-sharp-properties},
$\psi_{\phi,p}^{\sharp}$ is quasiconcave 
; see \cite[Definition~5.6]{bennett1988interpolation}.
Therefore, \cite[Proposition~5.10]{bennett1988interpolation} yields
\begin{equation}\label{eq:psisharp-concaveenvelope}
\psi_{\phi,p}^{\sharp}(u)\le
\psi_{\phi,p}^{\mathrm{cav}}(u)\le
2\psi_{\phi,p}^{\sharp}(u),\qquad u\in[0,1].
\end{equation}

We next record a useful property of the Luxemburg norm; see, e.g.,
\cite[Proposition~6, page 77]{rao1991theory}. We state the version needed in our setting and provide its proof. To be more precise, we note that, for the measurable profiles in this appendix, ``nonzero'' means nonzero on a set of positive Lebesgue measure. The notation $f \equiv0$ means that $f=0$ Lebesgue-almost everywhere.
\begin{lemma}\label{lem:luxemburg-property-for-proof}
Let $0\le h\le1$ be a nonzero measurable function and let $b=\|h\|_\phi$. Then
$$\int_0^1 \phi\left(\frac{h(s)}{b}\right)\mathrm{d}s =1.
$$
\end{lemma}
\begin{proof}
Define
\begin{equation*}
G(\lambda)\coloneqq  \int_0^1 \phi\left(\frac{h(s)}{\lambda}\right)\mathrm{d}s, \qquad \lambda>0.
\end{equation*}
Since $0\le h\le1$ and $\phi$ is continuous on $[0,\infty)$, the map $G$ is continuous on $(0,\infty)$. Besides, since $\phi$ is
increasing, $G$ is decreasing. Moreover, by the definition of the Luxemburg norm,
\begin{equation*}
b = \inf\{\lambda>0:G(\lambda)\le1\}.
\end{equation*}
Hence, $G(\lambda)>1$ for every $\lambda<b$. If $\lambda>b$, by the definition of the infimum, there exists
$\lambda'\in[b,\lambda)$ such that $G(\lambda')\le1$. Since $G$ is
decreasing, $G(\lambda)\le G(\lambda')\le 1$. Moreover, since $G$ is continuous on $(0,\infty)$, we deduce $G(b)=1$.
\end{proof}
We now construct a linear functional associated with the Luxemburg norm
that will be useful for deriving the upper bound below.
\begin{lemma}\label{lem:luxemburg-linear-functional}
Let $0\le h\le1$ be a nonzero measurable function. For $s \in [0,1]$, we define
$$z(s)\coloneqq \frac{h(s)}{\|h\|_{\phi}},\qquad\xi(s)\coloneqq \phi'_+\bigl(z(s)\bigr),
$$
and
$$D_0\coloneqq \int_0^1 z(s)\xi(s)\mathrm{d}s, \qquad y(s)\coloneqq \frac{\xi(s)}{D_0},
$$
where $\phi'_+$ denotes the right derivative of $\phi$. Then
$1 \le D_0<+\infty$, and for every $0\le r\le1$,
$$\int_0^1 r(s)y(s)\mathrm{d}s \le\|r\|_\phi.
$$
Moreover,
$$\int_0^1 h(s)y(s)\mathrm{d}s=\|h\|_\phi.
$$
\end{lemma}
\begin{proof}
Fix $s\in[0,1]$. Since $\phi$ is convex,
$\phi'_+(z(s))$ is a subgradient of $\phi$ at $z(s)$. Hence, for every $t\ge0$,
\begin{equation}\label{eq:subgradient}
\phi(t) \ge \phi(z(s)) + \phi'_+(z(s)) \left(t-z(s)\right).
\end{equation}
Taking $t=0$ and using $\phi(0)=0$, we obtain
\begin{equation*}
z(s)\phi'_+(z(s)) \ge \phi(z(s)).
\end{equation*}
Since $\xi(s)=\phi'_+(z(s))$, it follows that
\begin{equation*}
z(s)\xi(s)\ge\phi(z(s)).
\end{equation*}
Integrating over $s\in[0,1]$ and using
Lemma~\ref{lem:luxemburg-property-for-proof}, we have
\begin{equation*}
D_0 = \int_0^1 z(s)\xi(s)\mathrm{d}s \ge \int_0^1\phi(z(s))\mathrm{d}s = 1.
\end{equation*}
Since $0 \le z(s) \le \|h\|_{\phi}^{-1}$ for every $s \in [0,1]$ and since $\phi'_{+}$ is finite-valued and increasing on $[0,+\infty)$, we have
\begin{equation*}
0 \le D_0 \le \|h\|_{\phi}^{-1}\phi'_{+}\left(\|h\|_{\phi}^{-1}\right) <+\infty.
\end{equation*}
Let $0\le r \le 1$. If $r\equiv0$, the inequality is immediate. Hence, assume that $r$ is nonzero
and define
\begin{equation*}
w(s)\coloneqq \frac{r(s)}{\|r\|_\phi}.
\end{equation*}
Again fix $s\in[0,1]$. Applying the subgradient inequality at $z(s)$ with $t=w(s)$ gives
\begin{equation*}
\phi(w(s)) \ge \phi(z(s)) + \xi(s)\bigl(w(s)-z(s)\bigr).
\end{equation*}
Integrating over $[0,1]$ and using
Lemma~\ref{lem:luxemburg-property-for-proof},
\begin{equation*}
\int_0^1\phi(w(s))\mathrm{d}s=\int_0^1\phi(z(s))\mathrm{d}s=1.
\end{equation*}
Therefore,
$$\int_0^1 \xi(s)w(s)\mathrm{d}s\le \int_0^1 \xi(s)z(s)\mathrm{d}s= D_0.
$$
Since $r(s)=\|r\|_\phi w(s)$ and
$y(s)=\xi(s)/D_0$,
\begin{align*}
\int_0^1r(s)y(s)\mathrm{d}s=\frac{\|r\|_\phi}{D_0}\int_0^1 w(s)\xi(s)\mathrm{d}s\le \|r\|_\phi.
\end{align*}
Finally, since $h(s)=\|h\|_\phi z(s)$,
we have
\begin{align*}
\int_0^1 h(s)y(s)\mathrm{d}s=\frac{\|h\|_\phi}{D_0} \int_0^1 z(s)\xi(s)\mathrm{d}s=\|h\|_\phi.
\end{align*}
\end{proof}
The following identity expresses the regularization $\psi_{\phi,p}^{\sharp}$ directly in terms of the growth of $\phi$ and
will be useful for the upper bound below.
\begin{lemma}\label{lem:psi-sharp-growth}
For every $u\in(0,1]$,
$$\frac{\psi_{\phi,p}^{\sharp}(u)}{u}= \sup_{\phi^{-1}(1)\le t\le\phi^{-1}(1/u)} \frac{\phi(t)}{t^p}.
$$
\end{lemma}
\begin{proof}
By definition,
\begin{equation*}
\frac{\psi_{\phi,p}^{\sharp}(u)}{u} = \sup_{u\le v\le1} \frac{\psi_{\phi,p}(v)}{v}.
\end{equation*}
For $v\in[u,1]$, set
\begin{equation*}
t\coloneqq \phi^{-1}(1/v).
\end{equation*}
Since $\phi$ is strictly increasing,
\begin{equation*}
v=\frac{1}{\phi(t)},
\end{equation*}
and, as $v$ varies over $[u,1]$, $t$ varies over $\left[\phi^{-1}(1),\,\phi^{-1}(1/u)\right].$
Moreover,
\begin{equation*}
\frac{\psi_{\phi,p}(v)}{v}=\frac{1}{\phi^{-1}(1/v)^p}\frac{1}{v} = \frac{\phi(t)}{t^p}.
\end{equation*}
Taking the supremum over $v\in[u,1]$ gives the result.
\end{proof}
We now introduce a quantity that measures the largest possible variation between $\|f\|^p_{\phi}$ and $\|g\|^{p}_{\phi}$ among functions $f$ and $g$ that are close in $L^1$.
\begin{equation}\label{eq:Omega-supL1}
\Omega_{\phi,p}(u) \coloneqq\sup_{\substack{0\le f,g\le1\\\|f-g\|_{L^1(0,1)}\le u}} \left|\|f\|_\phi^p-\|g\|_\phi^p\right|, \quad u \in [0,1].
\end{equation}
For the upper-bound argument in Theorem~\ref{thm:powerSOW-two-sample}, it will be useful to work with pointwise ordered pairs $f$ and $g$. By symmetry, one may interchange $f$ and $g$ so that $\|f\|_{\phi} \ge \|g\|_{\phi}$, but this does not imply $f \ge g$ pointwise. However, replacing $f$ and $g$ by their pointwise maximum and minimum preserves their $L^1$-distance and can only increase the difference of their powered Luxemburg norms. We formalize this observation in the following lemma.
\begin{lemma}\label{lem:order-Omega}
For every $u\in [0,1]$,
\begin{equation*}
\Omega_{\phi,p}(u) = \sup_{\substack{0\le g\le f\le1\\\|f-g\|_{L^1(0,1)}\le u}}\left( \|f\|_\phi^p-\|g\|_\phi^p\right).
\end{equation*}
\end{lemma}
\begin{proof}
Let $(f,g)$ be
any pair satisfying
\begin{equation*}
0\le f,g\le1, \quad \|f-g\|_{L^1(0,1)}\le u.
\end{equation*}
Define
$F\coloneqq \max \left\{f,g\right\}, \quad G\coloneqq \min \left\{f,g\right\}.$
Then, $0\le G\le F \le 1$ and
$$F-G=|f-g|,\quad \|F-G\|_{L^1(0,1)}
=\|f-g\|_{L^1(0,1)}\le u.
$$
Moreover, by the monotonicity of the Luxemburg norm,
\begin{equation*}
\|F\|_\phi\ge\max\{\|f\|_\phi,\|g\|_\phi\},
\quad\|G\|_\phi \le \min\{\|f\|_\phi,\|g\|_\phi\}.
\end{equation*}
Hence,
\begin{equation*}
\|F\|^{p}_{\phi}-\|G\|^p_{\phi} \ge \max\{\|f\|^p_\phi,\|g\|^p_\phi\} - \min\{\|f\|^p_\phi,\|g\|^p_\phi\}=\left|\|f\|^p_{\phi}-\|g\|^p_{\phi}\right|.
\end{equation*}
It follows that
$$\Omega_{\phi,p}(u)\le\sup_{\substack{0\le G\le F\le1\\\|F-G\|_{L^1(0,1)}\le u}}\left(
\|F\|_\phi^p-\|G\|_\phi^p\right).
$$
Conversely, every ordered pair $0\le G\le F\le 1 $ satisfying $\|F-G\|_{L^1(0,1)}\le u$ is admissible in the definition of $\Omega_{\phi,p}(u)$. Moreover, since $G \le F$, the monotonicity of the Luxemburg norm gives $\|G\|_{\phi} \le \|F\|_{\phi}$. Hence,
\begin{equation*}
\|F\|^p_{\phi}-\|G\|^p_{\phi} =\left|\|F\|^p_{\phi}-\|G\|^p_{\phi}\right|\le \Omega_{\phi,p}(u).
\end{equation*}
Taking the supremum over such $(F,G)$ gives the reverse inequality. Relabeling $F$ and $G$ as $f,g$ completes the proof.
\end{proof}
With this reduction in hand, we are now ready to establish the main technical estimate for $\Omega_{\phi,p}$.
\begin{proposition}\label{prop:luxemburg-L1-modulus}
For every $u\in[0,1]$,
$$\psi_{\phi,p}^{\sharp}(u)\le \Omega_{\phi,p}(u)
\le p2^p\psi_{\phi,p}^{\sharp}(u).
$$
Consequently,
$$\frac{1}{2}\psi_{\phi,p}^{\mathrm{cav}}(u) \le\Omega_{\phi,p}(u)\le p2^p\psi_{\phi,p}^{\mathrm{cav}}(u).
$$
\end{proposition}
\begin{proof}
The case $u=0$ is immediate. Hence, let $u\in(0,1]$.
\paragraph{Lower bound.}
Fix $v\in[u,1]$ and choose a measurable set $A\subset(0,1)$ with $|A|=v$. Define
\begin{equation*}
f\coloneqq\mathbbm{1}_A, \quad g\coloneqq \left(1-\frac{u}{v}\right)\mathbbm{1}_A.
\end{equation*}
Then $0\le g\le f\le1$ and
\begin{equation*}
\|f-g\|_{L^1(0,1)} = \frac{u}{v}|A| = u.
\end{equation*}
    Since $\|\mathbbm{1}_A\|_\phi= \omega_\phi(v),$ the homogeneity of the Luxemburg norm gives
    \begin{equation*}
    \|f\|_\phi = \omega_\phi(v),
    \quad \|g\|_\phi = \left(1-\frac{u}{v}\right)\omega_\phi(v).
    \end{equation*}
    Therefore,
    \begin{align*}
    \|f\|_\phi^p-\|g\|_\phi^p=\omega_\phi(v)^p\left[
    1-\left(1-\frac{u}{v}\right)^p
    \right]=\psi_{\phi,p}(v)\left[1-
    \left(1-\frac{u}{v}\right)^p\right].
    \end{align*}
    Since $0 \le 1-u/v \le 1 $ and $p \ge 1$,
    \begin{equation*}
    1- \left(1-\frac{u}{v}\right)^p\ge \frac{u}{v}.
    \end{equation*}
    Hence
    \begin{equation*}
    \Omega_{\phi,p}(u)\ge \frac{u}{v}\psi_{\phi,p}(v).
    \end{equation*}
    Taking the supremum over $v\in[u,1]$ yields
    $$\Omega_{\phi,p}(u)\ge u\sup_{u\le v\le1}
    \frac{\psi_{\phi,p}(v)}{v}= \psi_{\phi,p}^{\sharp}(u).
    $$
    \paragraph{Upper bound.} By Lemma~\ref{lem:order-Omega}, it suffices to consider ordered pairs. Fix $0 \le g \le f \le 1$ satisfying
    \begin{equation*}
    \|f-g\|_{L^1(0,1)} \le u.
    \end{equation*}
    If $f \equiv 0$, then $g \equiv 0$ and the claim is trivial. Hence, assume $f \not\equiv 0.$

    We distinguish two cases.
    
    \emph{Case 1.} We assume
    \begin{equation*}
    \|f\|_{\phi} \le \frac{2}{\phi^{-1}\left(1/u\right)}. 
    \end{equation*}
    Then, for $p \in [1,+\infty)$, using Lemma~\ref{lem:psi-sharp-properties}, we have
    \begin{align*}
    \|f\|^p_{\phi} - \|g\|^p_{\phi} & \le \|f\|^p_{\phi}\\
    &\le \frac{2^p}{\left[\phi^{-1}\left(1/u\right)\right]^{p}}
    = 2 ^p\psi_{\phi,p}\left(u\right)\\
    & \le 2^p \psi^{\sharp}_{\phi,p}(u) \le p 2^p\psi^{\sharp}_{\phi,p}(u).
    \end{align*}
    \emph{Case 2.} We now assume
    \begin{equation*}
    \|f\|_{\phi} > \frac{2}{\phi^{-1}(1/u)}.
    \end{equation*}
    Applying Lemma~\ref{lem:luxemburg-linear-functional} to $f$, there exists a function $y$ such that
    \begin{equation*}
    \int_{0}^{1} f(s)y(s)\mathrm{d}s = \|f\|_{\phi}
    \end{equation*}
    and
    \begin{equation*}
    \int_{0}^{1}g(s)y(s)\mathrm{d}s \le \|g\|_{\phi}.
    \end{equation*}
    Therefore,
    \begin{align*}
    \|f\|_{\phi} - \|g\|_{\phi} &\le \int_{0}^{1} f(s)y(s)\mathrm{d}s - \int_{0}^{1}g(s)y(s)\mathrm{d}s\\
    &= \int_{0}^{1} \left(f(s)-g(s)\right)y(s)\mathrm{d}s.
    \end{align*}
    Since $0 \le \|g\|_{\phi}\le \|f\|_{\phi}$, the inequality $a^p - b^p \le p a^{p-1}(a-b),$ for all $0 \le b \le a$ gives
    \begin{equation}\label{eq:proof-firstbound}
    \|f\|^p_{\phi} - \|g\|^{p}_{\phi} \le p \|f\|^{p-1}_{\phi} \int_{0}^{1} \left(f(s)-g(s)\right)y(s)\mathrm{d}s.
    \end{equation}
    \paragraph{Pointwise estimate.} We now claim that
    \begin{equation*}
    \|f\|^{p-1}_{\phi}y(s) \le 2^p \frac{\psi^\sharp_{\phi,p}(u)}{u}, \quad  \forall s \in [0,1] \text{ such that } f(s)>0. 
    \end{equation*}
    By Lemma~\ref{lem:luxemburg-linear-functional}, 
    \begin{equation*}
    y(s) = \frac{\phi'_{+}\left(f(s)/\|f\|_{\phi}\right)}{D_0}, \quad D_0 \ge 1.
    \end{equation*}
    Since $\phi$ is increasing, $\phi'_{+}\ge 0$. Hence, $D_0 \ge 1$ implies 
    \begin{equation*}
    y(s) \le \phi'_{+}\left(\frac{f(s)}{\|f\|_{\phi}}\right).
    \end{equation*}
    Therefore, it is enough to show that 
    \begin{equation*}
    \|f\|^{p-1}_{\phi}\phi'_{+}\left(\frac{f(s)}{\|f\|_{\phi}}\right) \le 2^p \frac{\psi^{\sharp}_{\phi,p}(u)}{u}.
    \end{equation*}
    Fix $s \in [0,1]$ such that $f(s)>0$. From \eqref{eq:subgradient}, we deduce
    \begin{equation}\label{eq:subgradient-consequence}
    \phi'_{+}\left(\frac{f(s)}{\|f\|_{\phi}}\right) \le \frac{\phi\left(2f(s)/\|f\|_{\phi}\right)-\phi\left(f(s)/\|f\|_{\phi}\right)}{f(s)/\|f\|_{\phi}} \le \frac{\phi\left(2f(s)/\|f\|_{\phi}\right)}{f(s)/\|f\|_{\phi}}.
    \end{equation}
    First, suppose that
    \begin{equation*}
    0 < \frac{f(s)}{\|f\|_{\phi}} \le \frac{\phi^{-1}(1)}{2}.
    \end{equation*}
    Since 
    \begin{equation*}
    0 \le \frac{2f(s)}{\|f\|_{\phi}\phi^{-1}(1)}\le 1,
    \end{equation*}
    the convexity of $\phi$, together with $\phi(0)=0$ and $\phi\left(\phi^{-1}(1)\right)=1$ gives
    \begin{equation*}
    \phi \left(\frac{2f(s)}{\|f\|_{\phi}}\right) \le \frac{2f(s)}{\|f\|_{\phi}\phi^{-1}(1)}.
    \end{equation*}
    Therefore,
    \begin{equation*}
    \phi'_{+}\left(\frac{f(s)}{\|f\|_{\phi}}\right) \le \frac{2}{\phi^{-1}(1)}.
    \end{equation*}
    Moreover, since $0 \le f \le 1$,
    \begin{equation*}
    \|f\|_{\phi}\le \|1\|_{\phi} = 1/\phi^{-1}(1).
    \end{equation*}
    Hence
    \begin{equation*}
    \|f\|^{p-1}_{\phi} \phi'_{+}\left(\frac{f(s)}{\|f\|_{\phi}}\right) \le \frac{2}{\left[\phi^{-1}(1)\right]^p}.
    \end{equation*}
    By Lemma~\ref{lem:psi-sharp-growth}, we have
    \begin{equation*}
    \frac{\psi_{\phi,p}^{\sharp}(u)}{u}= \sup_{\phi^{-1}(1)\le t\le\phi^{-1}(1/u)} \frac{\phi(t)}{t^p} \ge \frac{1}{[\phi^{-1}(1)]^p}.
    \end{equation*}
    Thus, in this case
    \begin{equation*}
    \|f\|^{p-1}_{\phi} \phi'_{+}\left(\frac{f(s)}{\|f\|_{\phi}}\right) \le 2 \frac{\psi^{\sharp}_{\phi,p}(u)}{u} \le 2^p \frac{\psi^{\sharp}_{\phi,p}(u)}{u}.
    \end{equation*}
    It remains to consider the case where
    \begin{equation*}
    \frac{f(s)}{\|f\|_{\phi}} >\frac{\phi^{-1}(1)}{2}.
    \end{equation*}
    Since we are in the case where $$\|f\|_{\phi} > \frac{2}{\phi^{-1}(1/u)},$$
    and $f (s) \le 1$, we have
    \begin{equation*}
    \phi^{-1}(1) < \frac{2f(s)}{\|f\|_{\phi}} < \phi^{-1}(1/u).
    \end{equation*}
    Hence, using Lemma~\ref{lem:psi-sharp-growth}, 
    \begin{equation*}
    \phi\left(\frac{2f(s)}{\|f\|_{\phi}}\right) \le \frac{\psi^{\sharp}_{\phi,p}(u)}{u} \left(\frac{2f(s)}{\|f\|_{\phi}}\right)^p.
    \end{equation*}
    Therefore, from \eqref{eq:subgradient-consequence}, we have
    \begin{equation*}
    \phi'_{+}\left(\frac{f(s)}{\|f\|_{\phi}}\right) \le \frac{\phi\left(2f(s)/\|f\|_{\phi}\right)}{f(s)/\|f\|_{\phi}} \le 2^p\frac{\psi^{\sharp}_{\phi,p}(u)}{u}\left(\frac{f(s)}{\|f\|_{\phi}}\right)^{p-1}.
    \end{equation*}
    Consequently, 
    \begin{equation*}
    \|f\|^{p-1}_{\phi}\phi'_{+}\left(\frac{f(s)}{\|f\|_{\phi}}\right) \le 2^p\frac{\psi^{\sharp}_{\phi,p}(u)}{u} f(s)^{p-1} \le 2^p \frac{\psi^{\sharp}_{\phi,p}(u)}{u},
    \end{equation*}
    where the last inequality comes from $0 \le f(s) \le 1$ and $p \ge 1$.
    
    Thus, for every $s\in [0,1]$ such that $f(s)>0$, we have
    \begin{equation*}
    \|f\|^{p-1}_{\phi}y(s) \le 2^p \frac{\psi^{\sharp}_{\phi,p}(u)}{u}.
    \end{equation*}
    Since $0 \le g \le f$, we have $f-g=0$ on the set $\{f=0\}$. Substituting the pointwise bound into \eqref{eq:proof-firstbound} yields
    \begin{align*}
    \|f\|^p_{\phi}-\|g\|^{p}_{\phi} &\le p \|f\|_{\phi}^{p-1} \int_{0}^{1}\left(f(s)-g(s)\right)y(s)\mathrm{d}s\\
    & = p \|f\|_{\phi}^{p-1}\int_{\{s\in[0,1]:\,f(s)>0\}}\left(f(s)-g(s)\right)y(s)\mathrm{d}s\\
    & \le p2^p\frac{\psi^{\sharp}_{\phi,p}(u)}{u} \int_{\{s\in[0,1]:\,f(s)>0\}} \left(f(s)-g(s)\right) \mathrm{d}s\\
    & = p2^p\frac{\psi^{\sharp}_{\phi,p}(u)}{u} \int_{0}^1 \left(f(s)-g(s)\right) \mathrm{d}s   \\
    &=p2^p\frac{\psi^{\sharp}_{\phi,p}(u)}{u}
    \|f-g\|_{L^1(0,1)}\\
    & \le p2^p \psi^{\sharp}_{\phi,p}(u).
    \end{align*}
    Taking the supremum over all admissible pairs gives
    \begin{equation*}
    \Omega_{\phi,p}(u) \le p2^p \psi^{\sharp}_{\phi,p}(u).
    \end{equation*}
    Together with the lower bound proved above, this establishes the first assertion. The second assertion follows by combining this result with \eqref{eq:psisharp-concaveenvelope}.
    \end{proof}
With the estimate for $\Omega_{\phi,p}$ established, we are now ready to prove Theorem~\ref{thm:powerSOW-two-sample}.
\begin{proof}[Proof of Theorem~\ref{thm:powerSOW-two-sample}]
We have
\begin{align*}
\mathbb{E}\left|\operatorname{SOW}^p_{\phi,p}\left(\widehat{\mu}_n,\widehat{\nu}_m\right)-\operatorname{SOW}^p_{\phi,p}\left(\mu,\nu\right)\right| & = \mathbb{E}\left|\int_{\mathbb{S}^{d-1}}\operatorname{W}^p_{\phi}\left(\Pi^{\theta}_{\#}\widehat{\mu}_n,\Pi^{\theta}_{\#}\widehat{\nu}_m\right)-\operatorname{W}^p_{\phi}\left(\Pi^{\theta}_{\#}\mu,\Pi^{\theta}_{\#}\nu\right)\mathrm{d}\sigma\left(\theta\right)\right|\\
& \le \int_{\mathbb{S}^{d-1}}\mathbb{E}\left|\operatorname{W}^p_{\phi}\left(\Pi^{\theta}_{\#}\widehat{\mu}_n,\Pi^{\theta}_{\#}\widehat{\nu}_m\right)-\operatorname{W}^p_{\phi}\left(\Pi^{\theta}_{\#}\mu,\Pi^{\theta}_{\#}\nu\right)\right|\mathrm{d}\sigma(\theta),
\end{align*}
where the inequality follows from the triangle inequality and Tonelli's theorem.

Fix $\theta \in \mathbb{S}^{d-1}$. We now bound
\begin{equation*}
\left|\operatorname{W}^p_{\phi}\left(\Pi^{\theta}_{\#}\widehat{\mu}_n,\Pi^{\theta}_{\#}\widehat{\nu}_m\right)-\operatorname{W}^p_{\phi}\left(\Pi^{\theta}_{\#}\mu,\Pi^{\theta}_{\#}\nu\right)\right|.
\end{equation*}
For a probability measure $\alpha$ on $\mathbb{R}$, let $F_{\alpha}$ denote its cumulative distribution function and let $F^{-1}_{\alpha}$ denote its quantile function. With this convention, define
\begin{equation*}
\widehat{h}_{\theta}(s) \coloneqq\frac{ \left|
F_{\Pi^{\theta}_{\#}\widehat{\mu}_n}^{-1}(s)-F_{\Pi^{\theta}_{\#}\widehat{\nu}_m}^{-1}(s) \right|
}{D},\quad s\in(0,1),
\end{equation*}
and
\begin{equation*}
h_{\theta}(s)\coloneqq\frac{
\left|F_{\Pi^{\theta}_{\#}\mu}^{-1}(s)-
F_{\Pi^{\theta}_{\#}\nu}^{-1}(s)\right|
}{D},
\quad s\in(0,1).
\end{equation*}
By the one-dimensional quantile representation and the homogeneity of the Luxemburg norm, we have
\begin{align*}
\left|\operatorname{W}_{\phi}^p\left(\Pi^{\theta}_{\#}\widehat{\mu}_n,\Pi^{\theta}_{\#}\widehat{\nu}_m\right)-\operatorname{W}_{\phi}^p\left(\Pi^{\theta}_{\#}\mu,
\Pi^{\theta}_{\#}\nu\right)
\right|=D^p
\left|\|\widehat{h}_{\theta}\|_{\phi}^{p}-
\|h_{\theta}\|_{\phi}^{p}\right|.
\end{align*}
Since $\mu$ and $\nu$ are supported on $K$, the empirical measures
$\widehat{\mu}_n$ and $\widehat{\nu}_m$ are also supported on $K$. Since $\theta \in \mathbb{S}^{d-1}$, the projected supports have diameter at most $D$. Hence,
\begin{equation*}
0\le \widehat{h}_{\theta}(s),h_{\theta}(s) \le 1, \quad s \in (0,1).
\end{equation*}
As a result,
\begin{equation*}
0 \le \|\widehat{h}_{\theta}-h_{\theta}\|_{L^1(0,1)}\le 1.
\end{equation*}
By the definition of $\Omega_{\phi,p}$ in \eqref{eq:Omega-supL1}, we have
\begin{equation*}
\left|\|\widehat{h}_{\theta}\|_{\phi}^{p}-\|h_{\theta}\|_{\phi}^{p}\right| \le \Omega_{\phi,p}\left(\|\widehat{h}_{\theta}-h_{\theta}\|_{L^1(0,1)}\right).
\end{equation*}
Applying Proposition~\ref{prop:luxemburg-L1-modulus}, we have
\begin{equation*}
\Omega_{\phi,p}\left(\|\widehat{h}_{\theta}-h_{\theta}\|_{L^1(0,1)}\right)\le p2^p\psi_{\phi,p}^{\mathrm{cav}}\left(\|\widehat{h}_{\theta}-h_{\theta}\|_{L^1(0,1)}\right).
\end{equation*}
As a consequence,
\begin{equation*}
\mathbb{E}\left|\operatorname{W}_{\phi}^p\left(\Pi^{\theta}_{\#}\widehat{\mu}_n,\Pi^{\theta}_{\#}\widehat{\nu}_m\right)-\operatorname{W}_{\phi}^p\left(\Pi^{\theta}_{\#}\mu,\Pi^{\theta}_{\#}\nu\right)
\right|\le p2^pD^p\mathbb{E}\left[\psi^{\mathrm{cav}}_{\phi,p}\left(\|\widehat{h}_{\theta}-h_{\theta}\|_{L^1(0,1)}\right)\right].
\end{equation*}
Since $\psi_{\phi,p}^{\mathrm{cav}}$ is concave, Jensen's inequality yields
\begin{equation*}
\mathbb{E}\left[\psi_{\phi,p}^{\mathrm{cav}}\left(\|\widehat h_\theta-h_\theta\|_{L^1(0,1)}\right)
\right]\le\psi_{\phi,p}^{\mathrm{cav}}\left(\mathbb{E}\|\widehat h_\theta-h_\theta\|_{L^1(0,1)}\right).
\end{equation*}
We now bound $\mathbb{E}\|\widehat h_\theta-h_\theta\|_{L^1(0,1)}$. By the triangle inequality,
\begin{align*}
\|\widehat{h}_{\theta}-h_{\theta}\|_{L^1(0,1)}&= \frac{1}{D}\int_{0}^1\left|\left|
F_{\Pi^{\theta}_{\#}\widehat{\mu}_n}^{-1}(s)-F_{\Pi^{\theta}_{\#}\widehat{\nu}_m}^{-1}(s) \right|-\left|F_{\Pi^{\theta}_{\#}\mu}^{-1}(s)-
F_{\Pi^{\theta}_{\#}\nu}^{-1}(s)\right|\right|\mathrm{d}s\\
& \le \frac{1}{D}\int_{0}^{1}\left|
F_{\Pi^{\theta}_{\#}\widehat{\mu}_n}^{-1}(s)-F_{\Pi^{\theta}_{\#}\mu}^{-1}(s) \right|\mathrm{d}s+\frac{1}{D}\int_{0}^{1}\left|F_{\Pi^{\theta}_{\#}\widehat{\nu}_m}^{-1}(s)-
F_{\Pi^{\theta}_{\#}\nu}^{-1}(s)\right|\mathrm{d}s\\
&= \frac{1}{D}\left[\operatorname{W}_1\left(\Pi^{\theta}_{\#}\widehat{\mu}_n,\Pi^{\theta}_{\#}\mu\right)+\operatorname{W}_1\left(\Pi^{\theta}_{\#}\widehat{\nu}_m,\Pi^{\theta}_{\#}\nu\right)\right],
\end{align*}
where the last equality follows from the one-dimensional quantile
representation of the Wasserstein distance
(see, e.g., \cite[Theorem~2.10]{bobkov2019one}).

Taking expectations gives
\begin{equation*}
\mathbb{E}\left[\|\widehat{h}_{\theta}-h_{\theta}\|_{L^1(0,1)}\right] \le \frac{1}{D}\left[\mathbb{E}\operatorname{W}_1\left(\Pi^{\theta}_{\#}\widehat{\mu}_n,\Pi^{\theta}_{\#}\mu\right)+\mathbb{E}\operatorname{W}_1\left(\Pi^{\theta}_{\#}\widehat{\nu}_m,\Pi^{\theta}_{\#}\nu\right)\right].
\end{equation*}
Applying \cite[Theorem~3.2]{bobkov2019one}, we have
\begin{equation*}
\mathbb{E}\operatorname{W}_1\left(\Pi^{\theta}_{\#}\widehat{\mu}_n,\Pi^{\theta}_{\#}\mu\right) \le \frac{1}{\sqrt{n}}\int_{\mathbb{R}}\sqrt{F_{\Pi^{\theta}_{\#}\mu}(x)\left(1-F_{\Pi^{\theta}_{\#}\mu}(x)\right)}\mathrm{d}x \le \frac{D}{2\sqrt{n}},
\end{equation*}
and similarly,
\begin{equation*}
\mathbb{E}\operatorname{W}_1\left(\Pi^{\theta}_{\#}\widehat{\nu}_m,\Pi^{\theta}_{\#}\nu\right) \le \frac{D}{2\sqrt{m}}.
\end{equation*}
Consequently,
\begin{equation*}
\mathbb{E}\left[\|\widehat{h}_{\theta}-h_{\theta}\|_{L^1(0,1)}\right] \le \frac{1}{2}\left(n^{-1/2}+m^{-1/2}\right).
\end{equation*}
Since $\psi_{\phi,p}^{\mathrm{cav}}$ is increasing by Lemma~\ref{lem:properties-psi-cav}, we obtain
\begin{equation*}
\mathbb{E}\left|\operatorname{W}_{\phi}^p\left(\Pi^{\theta}_{\#}\widehat{\mu}_n,\Pi^{\theta}_{\#}\widehat{\nu}_m\right)-\operatorname{W}_{\phi}^p\left(\Pi^{\theta}_{\#}\mu,
\Pi^{\theta}_{\#}\nu\right)
\right|\le p2^pD^p\psi_{\phi,p}^{\mathrm{cav}}\left(\frac{1}{2}n^{-1/2}+\frac{1}{2}m^{-1/2}\right).
\end{equation*}
Integrating the above bound over $\theta\in\mathbb{S}^{d-1}$ and using that
$\sigma$ is a probability measure yields the desired result.
\end{proof}
\section{Proofs for Minimax Optimality}\label{app:proof-minimax}
\subsection{Le Cam Two-Point Bounds}\label{app:lecam-two-point}
The minimax lower bounds are based on Le Cam's two-point method
\cite{lecam1973convergence,le2012asymptotic}. We first state the
one-sample version, followed by a two-sample version for real-valued
functionals.

Throughout this appendix, for $\mu,\nu \in \mathcal{P}(K)$, we write $\mathbb{P}_{\mu}$ and $\mathbb{E}_{\mu}$ for probability and expectation with respect to $\mu^{\otimes n}$. Likewise, $\mathbb{P}_{\mu,\nu}$ and $\mathbb{E}_{\mu,\nu}$ denote probability and expectation with respect to $\mu^{\otimes n}\otimes \nu^{\otimes m}$.
\begin{lemma}[One-sample two-point bound] \label{lem:lecam-one-sample}
Let $\mu_0,\mu_1\in\mathcal{P}(K)$. Then
\begin{equation*}
\inf_{\widetilde{\mu}_n}\sup_{\mu\in\mathcal{P}(K)}\left[\mathbb{E}_{\mu}
\operatorname{SOW}_{\phi,p}^{p}\left(
\widetilde{\mu}_n,\mu\right)
\right]^{1/p}\ge \frac{
\operatorname{SOW}_{\phi,p}(\mu_0,\mu_1)
}{2^{1+1/p}
}\left[1-\sqrt{\frac{
n\operatorname{KL}(\mu_0\|\mu_1)}{2}}
\right].
\end{equation*}
\end{lemma}
\begin{proof}
Let $\widetilde{\mu}_n$ be any estimator based on
$X_1,\ldots,X_n$. Since $\mu_0,\mu_1\in\mathcal{P}(K)$, we have
\begin{equation*}
\sup_{\mu\in\mathcal{P}(K)}\left[\mathbb{E}_{\mu}
\operatorname{SOW}_{\phi,p}^{p}\left(
\widetilde{\mu}_n,\mu\right)
\right]^{1/p}\ge \max_{i \in \{0,1\}}\left[\mathbb{E}_{\mu_i}\operatorname{SOW}^p_{\phi,p}\left(\widetilde{\mu}_n,\mu_i\right)\right]^{1/p}.
\end{equation*}
Define the test
$$
\Psi(\widetilde{\mu}_n)\coloneqq 
\begin{cases}
0,
&
\text{if }\operatorname{SOW}_{\phi,p}
\left(\widetilde{\mu}_n,\mu_0
\right)\le\operatorname{SOW}_{\phi,p}
\left(\widetilde{\mu}_n,\mu_1\right),
\\
1,
&\text{if }
\operatorname{SOW}_{\phi,p}\left(
\widetilde{\mu}_n,\mu_0\right)
>\operatorname{SOW}_{\phi,p}
\left(\widetilde{\mu}_n,\mu_1
\right).
\end{cases}
$$
Consider first the event
$\{\Psi(\widetilde{\mu}_n)=1\}$. By the definition of
$\Psi(\widetilde{\mu}_n)$,
$$\operatorname{SOW}_{\phi,p}
\left(\widetilde{\mu}_n,\mu_0\right)>
\operatorname{SOW}_{\phi,p}\left(
\widetilde{\mu}_n,\mu_1\right).
$$
Hence, by the triangle inequality for
$\operatorname{SOW}_{\phi,p}$,
\begin{align*}
\operatorname{SOW}_{\phi,p}(\mu_0,\mu_1)
&\le\operatorname{SOW}_{\phi,p}
\left(\mu_0,\widetilde{\mu}_n
\right)+\operatorname{SOW}_{\phi,p}
\left(\widetilde{\mu}_n,\mu_1
\right)
\\
&<2\operatorname{SOW}_{\phi,p}\left(
\widetilde{\mu}_n,\mu_0\right).
\end{align*}
Therefore, on the event
$\{\Psi(\widetilde{\mu}_n)=1\}$,
$$\operatorname{SOW}_{\phi,p}
\left(\widetilde{\mu}_n,\mu_0
\right)\ge\frac{\operatorname{SOW}_{\phi,p}(\mu_0,\mu_1)
}{2}.
$$
Raising both sides to the power $p$ gives
$$\operatorname{SOW}_{\phi,p}^{p}
\left(\widetilde{\mu}_n,\mu_0
\right)\ge\left(
\frac{\operatorname{SOW}_{\phi,p}(\mu_0,\mu_1)
}{2}\right)^p
\qquad
\text{on }
\{\Psi(\widetilde{\mu}_n)=1\}.
$$
It follows that
$$\operatorname{SOW}_{\phi,p}^{p}
\left(\widetilde{\mu}_n,\mu_0
\right)\ge\left(
\frac{\operatorname{SOW}_{\phi,p}(\mu_0,\mu_1)
}{2}\right)^p
\mathbbm{1}_{\{\Psi(\widetilde{\mu}_n)=1\}}.
$$
Taking expectation under $\mu_0^{\otimes n}$ yields
$$\mathbb{E}_{\mu_0}
\operatorname{SOW}_{\phi,p}^{p}\left(
\widetilde{\mu}_n,\mu_0\right)
\ge\left(\frac{
\operatorname{SOW}_{\phi,p}(\mu_0,\mu_1)
}{2}\right)^p\mathbb{P}_{\mu_0}
\left(\Psi(\widetilde{\mu}_n)=1
\right).
$$
Taking the $p$-th root, we obtain
$$
\left[\mathbb{E}_{\mu_0}
\operatorname{SOW}_{\phi,p}^{p}\left(
\widetilde{\mu}_n,\mu_0\right)\right]^{1/p}
\ge\frac{
\operatorname{SOW}_{\phi,p}(\mu_0,\mu_1)
}{2}\mathbb{P}_{\mu_0}
\left(
\Psi(\widetilde{\mu}_n)=1
\right)^{1/p}.
$$
Similarly, we can show that
$$
\left[\mathbb{E}_{\mu_1}
\operatorname{SOW}_{\phi,p}^{p}\left(\widetilde{\mu}_n,\mu_1\right)
\right]^{1/p}\ge\frac{
\operatorname{SOW}_{\phi,p}(\mu_0,\mu_1)
}{2}\mathbb{P}_{\mu_1}\left(\Psi(\widetilde{\mu}_n)=0\right)^{1/p}.
$$
Combining these estimates with the initial bound and using $\max\left\{a^{1/p},b^{1/p}\right\}\ge \left(\frac{a+b}{2}\right)^{1/p}$ for $a,b \ge 0$, we obtain 
\begin{equation*}
\sup_{\mu\in\mathcal{P}(K)}\left[
\mathbb{E}_{\mu}\operatorname{SOW}_{\phi,p}^{p}
\left(\widetilde{\mu}_n,\mu
\right)\right]^{1/p}
\ge\frac{
\operatorname{SOW}_{\phi,p}(\mu_0,\mu_1)
}{2^{1+1/p}}\Bigg[
\mathbb{P}_{\mu_0}\left(
\Psi(\widetilde{\mu}_n)=1\right)+
\mathbb{P}_{\mu_1}\left(\Psi(\widetilde{\mu}_n)=0
\right)\Bigg]^{1/p}.
\end{equation*}
By the definition of total variation
\cite[Definition~2.4]{tsybakov2009introduction}, we have
\begin{align*}
\sup_{\mu\in\mathcal{P}(K)}\left[
\mathbb{E}_{\mu}\operatorname{SOW}_{\phi,p}^{p}
\left(\widetilde{\mu}_n,\mu
\right)\right]^{1/p}\ge 
\frac{\operatorname{SOW}_{\phi,p}(\mu_0,\mu_1)
}{2^{1+1/p}
}\left[1-\operatorname{TV}\left(\mu_0^{\otimes n},\mu_1^{\otimes n}\right)\right]^{1/p}.
\end{align*}
Since $0 \le1-\operatorname{TV}\left(\mu_0^{\otimes n},\mu_1^{\otimes n}\right) \le 1$ and $p\ge1$, we have
$$\sup_{\mu\in\mathcal{P}(K)}\left[ \mathbb{E}_{\mu}\operatorname{SOW}_{\phi,p}^{p}
\left(\widetilde{\mu}_n,\mu
\right)\right]^{1/p}\ge 
\frac{\operatorname{SOW}_{\phi,p}(\mu_0,\mu_1)
}{2^{1+1/p}}
\left[1-\operatorname{TV}\left(\mu_0^{\otimes n},
\mu_1^{\otimes n}\right)
\right].
$$
By Pinsker's inequality
\cite[Lemma~2.5]{tsybakov2009introduction},
\begin{equation*}
\operatorname{TV} \left(
\mu_0^{\otimes n},\mu_1^{\otimes n}
\right)\le\sqrt{\frac{1}{2}
\operatorname{KL}\left(
\mu_0^{\otimes n}\|\mu_1^{\otimes n}\right)}.
\end{equation*}
Using the additivity of the Kullback--Leibler divergence over product measures, we have
\begin{equation*}
\sup_{\mu\in\mathcal{P}(K)}\left[
\mathbb{E}_{\mu}\operatorname{SOW}_{\phi,p}^{p}
\left(\widetilde{\mu}_n,\mu
\right)\right]^{1/p}\ge
\frac{\operatorname{SOW}_{\phi,p}(\mu_0,\mu_1)
}{2^{1+1/p}}\left[1-
\sqrt{\frac{
n\operatorname{KL}(\mu_0\|\mu_1)}{2}
}\right].
\end{equation*}
Taking the infimum over all estimators
$\widetilde{\mu}_n$ completes the proof.
\end{proof}
We next state the corresponding two-sample bound for real-valued functionals.
\begin{lemma}[Two-sample two-point bound] \label{lem:lecam-two-sample}
Let $(\mu_0,\nu_0),(\mu_1,\nu_1)
\in\mathcal{P}(K)\times\mathcal{P}(K)$,
and let
$$T:\mathcal{P}(K)\times\mathcal{P}(K)
\longrightarrow \mathbb{R}$$
be a real-valued functional. Define
$$T_0\coloneqq T(\mu_0,\nu_0),
\quad T_1\coloneqq T(\mu_1,\nu_1).
$$
Then
\begin{equation*}
\inf_{\widehat T}\sup_{\mu,\nu\in\mathcal{P}(K)}
\mathbb{E}_{\mu,\nu}\left|
\widehat T-T(\mu,\nu)\right| \ge
\frac{|T_1-T_0|}{4}\left[1-\sqrt{\frac{
n\operatorname{KL}(\mu_0\|\mu_1)+m\operatorname{KL}(\nu_0\|\nu_1)
}{2}}\right].
\end{equation*}
\end{lemma}
\begin{proof}
Let $\widehat T=
\widehat T(X_1,\ldots,X_n,Y_1,\ldots,Y_m)$
be any estimator based on $X_1,\ldots,X_n,Y_1,\ldots,Y_m$. We have
\begin{align*}
\sup_{\mu,\nu\in\mathcal{P}(K)}
\mathbb{E}_{\mu,\nu}\left|
\widehat T-T(\mu,\nu)\right|
&\ge\max\left\{\mathbb{E}_{\mu_0,\nu_0}
|\widehat T-T_0|,\mathbb{E}_{\mu_1,\nu_1}
|\widehat T-T_1|\right\}
\\
&\ge\frac{1}{2}\mathbb{E}_{\mu_0,\nu_0}
|\widehat T-T_0|+\frac{1}{2}\mathbb{E}_{\mu_1,\nu_1}|\widehat T-T_1|.
\end{align*}
Define
$$\Psi(\widehat T)
\coloneqq 
\begin{cases}
0,&\text{if }|\widehat T-T_0|
\le|\widehat T-T_1|,
\\1,&\text{if }|\widehat T-T_0|>
|\widehat T-T_1|.
\end{cases}
$$
On the event $\{\Psi(\widehat T)=1\}$, we have
$$|\widehat T-T_0|> |\widehat T-T_1|.
$$
Hence, by the triangle inequality,
$$|T_1-T_0|\le|T_1-\widehat T|
+|\widehat T-T_0|\le 2|\widehat T-T_0|.
$$
Therefore,
$$\mathbb{E}_{\mu_0,\nu_0}
|\widehat T-T_0|\ge\frac{|T_1-T_0|}{2}
\mathbb{P}_{\mu_0,\nu_0}\left(\Psi(\widehat T)=1\right).
$$
Similarly,
$$\mathbb{E}_{\mu_1,\nu_1}|\widehat T-T_1|
\ge\frac{|T_1-T_0|}{2}\mathbb{P}_{\mu_1,\nu_1}\left(\Psi(\widehat T)=0\right).
$$
Consequently,
\begin{align*}
\sup_{\mu,\nu\in\mathcal{P}(K)}\mathbb{E}_{\mu,\nu}\left|\widehat T-T(\mu,\nu)\right|
&\ge\frac{|T_1-T_0|}{4}\left[
\mathbb{P}_{\mu_0,\nu_0}\left(
\Psi(\widehat T)=1\right)
+\mathbb{P}_{\mu_1,\nu_1}\left(\Psi(\widehat T)=0
\right)\right]
\\
&=\frac{|T_1-T_0|}{4}
\left[1-\mathbb{P}_{\mu_0,\nu_0}
\left(\Psi(\widehat T)=0
\right)+\mathbb{P}_{\mu_1,\nu_1}
\left(\Psi(\widehat T)=0\right)
\right]
\\
&\ge\frac{|T_1-T_0|}{4}
\left[1-\operatorname{TV}\left(\mu_0^{\otimes n}\otimes\nu_0^{\otimes m},\mu_1^{\otimes n}\otimes\nu_1^{\otimes m}
\right)
\right]
\\
&\ge\frac{|T_1-T_0|}{4}
\left[1-\sqrt{\frac{1}{2}\operatorname{KL}
\left(\mu_0^{\otimes n}\otimes\nu_0^{\otimes m}
\|\mu_1^{\otimes n}\otimes\nu_1^{\otimes m}
\right)}
\right]
\\
&=\frac{|T_1-T_0|}{4}
\left[1-\sqrt{\frac{n\operatorname{KL}(\mu_0\|\mu_1)+m\operatorname{KL}(\nu_0\|\nu_1)
}{2}}
\right].
\end{align*}
The second inequality follows from the definition of total variation
\cite[Definition~2.4]{tsybakov2009introduction}, and the third
inequality follows from Pinsker's inequality
\cite[Lemma~2.5]{tsybakov2009introduction}. The last equality follows
from the additivity of the Kullback--Leibler divergence over product
measures \cite[Page 85]{tsybakov2009introduction}.
Taking the infimum over all estimators $\widehat T$ completes the proof.
\end{proof}
With this general framework in place, we now construct suitable pairs of
distributions and apply the preceding two-point bounds to the one-sample
and two-sample SOW estimation problems. Before doing so, we recall the
useful spherical identity
\begin{equation}\label{eq:sphereuseful}
\int_{\mathbb{S}^{d-1}}
\left|\langle \theta,z\rangle\right|^p
\mathrm{d}\sigma(\theta)= \kappa_{p,d}\|z\|^p,
\end{equation}
where $\sigma$ denotes the uniform probability measure on $\mathbb{S}^{d-1}$ and $\kappa_{p,d}>0$ is a constant depending only on
$p$ and $d$ (see, e.g., \cite[Section~11.2]{peyre2025optimal}).

Besides, since $K$ is compact, the continuous map
$(x,y)\mapsto \|x-y\|$ attains its maximum on $K\times K$.
Hence, there exist $x_0,x_1\in K$ such that
$$\|x_0-x_1\|=\operatorname{diam}(K)=D.$$
Throughout the remainder of this section, we fix such a pair
$(x_0,x_1)$ and use it in all subsequent lower-bound constructions.

\subsection{One-sample estimation} \label{app:minimax-one-sample}
We first prove the one-sample lower bound. We choose
\begin{equation*}
\mu_0 = \frac{1}{2}\delta_{x_0}+\frac{1}{2}\delta_{x_1}, \quad \mu_1 = \left(\frac{1}{2}-\frac{1}{4\sqrt{n}}\right)\delta_{x_0}+\left(\frac{1}{2}+\frac{1}{4\sqrt{n}}\right)\delta_{x_1}.
\end{equation*}

For this choice, we have
\begin{align*}
\operatorname{SOW}_{\phi,p}^p(\mu_0,\mu_1)&=
\int_{\mathbb{S}^{d-1}}\operatorname{W}_\phi^p
\left(\Pi^\theta_{\#}\mu_0,
\Pi^\theta_{\#}\mu_1\right)\mathrm{d}\sigma(\theta)\\
&=\int_{\mathbb{S}^{d-1}}\left[D\left|
\left\langle\theta,\frac{x_0-x_1}{\|x_0-x_1\|}
\right\rangle\right|\omega_\phi
\left(\frac{1}{4\sqrt{n}}
\right)\right]^p\mathrm{d}\sigma(\theta)
\\&=D^p\omega_\phi\left(\frac{1}{4\sqrt{n}}\right)^p\kappa_{p,d}.
\end{align*}
Therefore,
\begin{equation*}
\operatorname{SOW}_{\phi,p}\left(\mu_0,\mu_1\right)=D\omega_{\phi}\left(\frac{1}{4\sqrt{n}}\right)\kappa_{p,d}^{1/p}.
\end{equation*}
Moreover, applying Lemma~\ref{lem:properties-omega-psi} with
$a=4$ and $u=1/(4\sqrt{n})$ gives
\begin{equation*}
\omega_{\phi}\left(n^{-1/2}\right) \le 4\omega_{\phi}\left(\frac{1}{4\sqrt{n}}\right).
\end{equation*}
Hence,
\begin{equation*}
\operatorname{SOW}_{\phi,p}\left(\mu_0,\mu_1\right)\ge \frac{\kappa_{p,d}^{1/p}D}{4}\omega_{\phi}\left(n^{-1/2}\right).
\end{equation*}
On the other hand, we also have
\begin{align*}
\operatorname{KL}(\mu_0\|\mu_1)
&=\frac{1}{2}\log\frac{1/2}{1/2+1/(4\sqrt{n})}
+\frac{1}{2}\log\frac{1/2}{1/2-1/(4\sqrt{n})}\\
&\le\frac{1}{2}\left(\frac{1/2}{1/2+1/(4\sqrt{n})}-1\right)+\frac{1}{2}\left(\frac{1/2}{1/2-1/(4\sqrt{n})}-1
\right)\\
&=\frac{1/(16n)}{1/4-1/(16n)},
\end{align*}
where the inequality comes from the fact that $\log t \le t -1 \, \forall t >0$.
Since $n \ge 1$, 
\begin{equation*}
\frac{1}{4}-\frac{1}{16n}\ge \frac{3}{16}.
\end{equation*}
As a result,
\begin{equation*}
\operatorname{KL}\left(\mu_0\|\mu_1\right) \le \frac{1}{3n}.
\end{equation*}
Combining the preceding estimates with
Lemma~\ref{lem:lecam-one-sample}, we obtain
\begin{equation*}
\inf_{\widetilde{\mu}_n}\sup_{\mu\in\mathcal{P}(K)}\left[\mathbb{E}_{\mu}
\operatorname{SOW}_{\phi,p}^{p}\left(
\widetilde{\mu}_n,\mu\right)
\right]^{1/p} \ge \frac{\kappa^{1/p}_{p,d}}{2^{3+1/p}}\left(1-\frac{1}{\sqrt{6}}\right)D\omega_{\phi}\left(n^{-1/2}\right).
\end{equation*}

The matching upper bound follows directly from
Theorem~\ref{thm:SOW-one-sample} by taking the empirical measure
$\widehat{\mu}_n$ as the estimator. Combining the upper and lower bounds
yields
$$\inf_{\widetilde{\mu}_n}\sup_{\mu\in\mathcal{P}(K)}\left[\mathbb{E}_{\mu}
\operatorname{SOW}_{\phi,p}^{p}\left(
\widetilde{\mu}_n,\mu\right)\right]^{1/p}
\asymp_{d,p}D\omega_\phi\left(n^{-1/2}\right).$$
\subsection{Two-sample estimation}\label{app:minimax-two-sample}
We next prove the two-sample statement in
Theorem~\ref{thm:SOW-minimax}. Without loss of generality, assume that
$n\le m$.
\subsubsection{Two-sample estimation of the SOW distance}
We choose 
\begin{equation*}
\mu_0 = \frac{1}{2}\delta_{x_0}+\frac{1}{2}\delta_{x_1}, \quad \nu_0=\frac{1}{2}\delta_{x_0}+\frac{1}{2}\delta_{x_1},
\end{equation*}
and
\begin{equation*}
\mu_1 = \left(\frac{1}{2}-\frac{1}{4\sqrt{n}}\right)\delta_{x_0}+\left(\frac{1}{2}+\frac{1}{4\sqrt{n}}\right)\delta_{x_1}, \quad \nu_1=\nu_0.
\end{equation*}

As in the one-sample case, our construction gives
\begin{equation*}
\operatorname{SOW}_{\phi,p}\left(\mu_0,\nu_0\right)=0, \quad \operatorname{SOW}_{\phi,p}(\mu_1,\nu_1)=D\kappa_{p,d}^{1/p}
\omega_\phi\left(\frac{1}{4\sqrt{n}}\right)\ge
\frac{D\kappa_{p,d}^{1/p}}{4}
\omega_\phi\left(n^{-1/2}\right).
\end{equation*}
For the KL term, since $\nu_0=\nu_1$,
\begin{equation*}
\operatorname{KL}\left(\nu_0\|\nu_1\right)=0.
\end{equation*}
The one-sample calculation also gives
\begin{equation*}
\operatorname{KL}\left(\mu_0\|\mu_1\right) \le \frac{1}{3n}.
\end{equation*}
Therefore,
\begin{equation*}
n\operatorname{KL}\left(\mu_0\|\mu_1\right)+m\operatorname{KL}\left(\nu_0\|\nu_1\right)\le n\frac{1}{3n}+0=\frac{1}{3}.
\end{equation*}
Applying Lemma~\ref{lem:lecam-two-sample} to the functional $T(\mu,\nu) = \operatorname{SOW}_{\phi,p}\left(\mu,\nu\right)$, we obtain
\begin{equation*}
\inf_{\widehat T}\sup_{\mu,\nu\in\mathcal{P}(K)}
\mathbb{E}_{\mu,\nu}\left|
\widehat T-\operatorname{SOW}_{\phi,p}(\mu,\nu)\right|\ge \frac{\kappa^{1/p}_{p,d}}{16}\left(1-\frac{1}{\sqrt{6}}\right)D\omega_{\phi}\left(n^{-1/2}\right).
\end{equation*}
Since $n\le m$ and $\omega_{\phi}$ is increasing by Lemma~\ref{lem:properties-omega-psi},
\begin{equation*}
\omega_{\phi}\left(n^{-1/2}\right) \ge \omega_{\phi}\left(m^{-1/2}\right).
\end{equation*}
Therefore,
\begin{equation*}
\omega_{\phi}\left(n^{-1/2}\right) \ge \frac{1}{2}\left[\omega_{\phi} \left(n^{-1/2}\right)+\omega_{\phi}\left(m^{-1/2}\right)\right].
\end{equation*}
Consequently,
\begin{equation*}
\inf_{\widehat T}\sup_{\mu,\nu\in\mathcal{P}(K)}
\mathbb{E}_{\mu,\nu}\left|
\widehat T-\operatorname{SOW}_{\phi,p}(\mu,\nu)\right|\ge \frac{\kappa^{1/p}_{p,d}}{32}\left(1-\frac{1}{\sqrt{6}}\right)D\left[\omega_{\phi}\left(n^{-1/2}\right)+\omega_{\phi}\left(m^{-1/2}\right)\right].
\end{equation*}
The matching upper bound follows from
Corollary~\ref{cor:SOW-two-sample}. Combining the upper and lower bounds
yields
\begin{equation*}
\inf_{\widehat T} \sup_{\mu,\nu\in\mathcal{P}(K)}
\mathbb{E}_{\mu,\nu}\left|
\widehat T-\operatorname{SOW}_{\phi,p}(\mu,\nu) \right|\asymp_{d,p}D\left[\omega_\phi\left(n^{-1/2}\right)+\omega_\phi\left(m^{-1/2}\right)
\right].
\end{equation*}
\subsubsection{Two-sample estimation of the powered SOW functional}
The construction is more delicate for the powered functional. In contrast to the SOW distance considered above, a direct perturbation at the statistical scale $n^{-1/2}$ only yields $\psi_{\phi,p}\left(n^{-1/2}\right)$, whereas the sharp rate is controlled by $\psi^{\mathrm{cav}}_{\phi,p}\left(n^{-1/2}\right)$. Bridging this gap is the main technical difficulty in the lower-bound argument. We resolve it through the following useful localization result.
\begin{lemma}\label{lem:parameter-powerSOW}
For every $h\in (0,1/4]$, there exist $0\le q_0<q_1\le \frac{1}{2}$ such that $q_1-q_0 \le h$ and
\begin{equation*}
\psi_{\phi,p}\left(q_1\right)-\psi_{\phi,p}\left(q_0\right)\ge 2^{-(p+2)}\psi_{\phi,p}^{\mathrm{cav}}(h).
\end{equation*}
\end{lemma}
\begin{proof}
The main idea is to find $v^\ast\in[h,1/2]$ and a constant $C_p>0$ depending only on $p$ such that
\begin{equation*}
h\frac{\psi_{\phi,p}(v^\ast)}{v^\ast}\ge C_p\psi_{\phi,p}^{\mathrm{cav}}(h).
\end{equation*}
Since $\psi_{\phi,p}(0)=0$, the ratio
$\psi_{\phi,p}(v^\ast)/v^\ast$ represents the average increase of
$\psi_{\phi,p}$ over the interval $[0,v^\ast]$. We then partition
$[0,v^\ast]$ into subintervals of length at most $h$. Since the sum of
the increments of $\psi_{\phi,p}$ over these subintervals is exactly
$\psi_{\phi,p}(v^\ast)$, at least one of them must have an increment
comparable to this average. This yields the desired pair $q_0,q_1$.

To carry out the first step, we fix $h \in (0,1/4]$. We then define
\begin{equation*}
S_h \coloneqq \sup_{h \le v \le 1} \frac{\psi_{\phi,p}(v)}{v}, \quad S_{h,1/2}\coloneqq \sup_{h \le v \le 1/2}\frac{\psi_{\phi,p}(v)}{v}.
\end{equation*}
We first compare these two suprema. Since $h \in (0,1/4]$,
\begin{equation*}
S_h = \max \left\{S_{h,1/2},\sup_{1/2\le v \le 1}\frac{\psi_{\phi,p}(v)}{v}\right\}.
\end{equation*}
It remains to control the contribution from $v \in [1/2,1]$.
By the definition of $S_{h,1/2}$, we have
\begin{equation*}
S_{h,1/2} \ge 2\psi_{\phi,p}\left(1/2\right)=\frac{2}{\left[\phi^{-1}(2)\right]^p}.
\end{equation*}
Applying Lemma~\ref{lem:properties-omega-psi} (with $a=2$ and $u=1/2$) gives
\begin{equation*}
\phi^{-1}(2) \le 2\phi^{-1}(1).
\end{equation*}
Consequently,
\begin{equation*}
S_{h,1/2} \ge \frac{2^{1-p}}{\left[\phi^{-1}(1)\right]^p}.
\end{equation*}
On the other hand, for $v \in [1/2,1]$, since $\psi_{\phi,p}$ is increasing by Lemma~\ref{lem:properties-omega-psi},
\begin{equation*}
\frac{\psi_{\phi,p}(v)}{v} \le 2\psi_{\phi,p}(1)=\frac{2}{\left[\phi^{-1}(1)\right]^p}.
\end{equation*}
Hence,
\begin{equation*}
\sup_{v \in [1/2,1]}\frac{\psi_{\phi,p}(v)}{v} \le \frac{2}{\left[\phi^{-1}(1)\right]^p}\le 2^pS_{h,1/2}.
\end{equation*}
As a result, $S_h \le 2^p S_{h,1/2}$.

By Lemma~\ref{lem:properties-omega-psi}, $\psi_{\phi,p}$ is continuous on $[h,1/2]$. Hence, the map $v \mapsto \psi_{\phi,p}(v)/v$ is continuous on $[h,1/2]$. Therefore, there exists $v^\ast\in[h,1/2]$ such that
\begin{equation*}
\frac{\psi_{\phi,p}(v^\ast)}{v^\ast} = S_{h,1/2}.
\end{equation*}
Since $hS_h=\psi^{\sharp}_{\phi,p}(h)$, combining $S_h \le 2^pS_{h,1/2}$ with
\eqref{eq:psisharp-concaveenvelope} gives
\begin{equation*}
h\frac{\psi_{\phi,p}(v^\ast)}{v^\ast}= hS_{h,1/2} \ge
2^{-p}hS_h\ge
2^{-(p+1)}\psi_{\phi,p}^{\mathrm{cav}}(h).
\end{equation*}
We now partition $[0,v^\ast]$ into
\begin{equation*}
k\coloneqq \left\lceil \frac{v^\ast}{h}\right\rceil
\end{equation*}
equal subintervals with endpoints
\begin{equation*}
t_i\coloneqq \frac{i v^\ast}{k},\qquad i=0,\ldots,k.
\end{equation*}
Then $t_i-t_{i-1}=v^\ast/k\le h$. Moreover,
\begin{equation*}
\sum_{i=1}^k\left[\psi_{\phi,p}(t_i)-\psi_{\phi,p}(t_{i-1})\right]=\psi_{\phi,p}(v^\ast),
\end{equation*}
since $\psi_{\phi,p}(0)=0$. Hence, for some $i\in\{1,\ldots,k\}$,
\begin{equation*}
\psi_{\phi,p}(t_i)-\psi_{\phi,p}(t_{i-1}) \ge \frac{\psi_{\phi,p}(v^\ast)}{k}.
\end{equation*}
Since $v^\ast\ge h$ and $k \le v^\ast/h +1$, we deduce
\begin{equation*}
\psi_{\phi,p}(t_i)-\psi_{\phi,p}(t_{i-1}) \ge\frac{h}{2v^\ast}\psi_{\phi,p}(v^\ast)\ge 2^{-(p+2)}\psi_{\phi,p}^{\mathrm{cav}}(h).
\end{equation*}
Setting $q_0=t_{i-1}$ and $q_1=t_i$ gives
\begin{equation*}
0\le q_0<q_1\le \frac{1}{2},\quad q_1-q_0\le h.
\end{equation*}
\end{proof}
With this lemma in hand, we are now ready to prove the lower bound for the powered SOW functional in Theorem~\ref{thm:SOW-minimax}.
\begin{proof}
Assume without loss of generality that $n\le m$. Applying
Lemma~\ref{lem:parameter-powerSOW} with $h=\frac{1}{4\sqrt{n}},$ there exist $0\le q_0<q_1\le 1/2$ such that
\begin{equation*}
q_1-q_0\le \frac{1}{4\sqrt{n}}
\end{equation*}
and
\begin{equation*}
\psi_{\phi,p}(q_1)-\psi_{\phi,p}(q_0)
\ge 2^{-(p+2)}\psi_{\phi,p}^{\mathrm{cav}} \left(\frac{1}{4\sqrt{n}}\right).
\end{equation*}
We then choose
\begin{equation*}
a_i=\frac{1}{4}+q_i,\qquad i=0,1,
\end{equation*}
and define
\begin{equation*}
\mu_i=(1-a_i)\delta_{x_0}+a_i\delta_{x_1}, \quad 
\nu_0=\nu_1=\frac34\delta_{x_0}+\frac14\delta_{x_1}.
\end{equation*}
This choice gives $\operatorname{KL}(\nu_0\|\nu_1)=0.$
Besides, since $a_0,a_1\in[1/4,3/4]$, using $\log t\le t-1$ for all $t>0$ gives
\begin{align*}
\operatorname{KL}(\mu_0\|\mu_1)
&=a_0\log\frac{a_0}{a_1}
+(1-a_0)\log\frac{1-a_0}{1-a_1}\\
&\le\frac{(a_1-a_0)^2}{a_1(1-a_1)}\\&\le
\frac{16}{3}(a_1-a_0)^2\\&\le \frac{1}{3n}.
\end{align*}
Hence,
\begin{equation*}
n\operatorname{KL}(\mu_0\|\mu_1)+ m\operatorname{KL}(\nu_0\|\nu_1) \le \frac{1}{3}.
\end{equation*}
Similarly, we also have
\begin{equation*}
\operatorname{SOW}_{\phi,p}^p(\mu_i,\nu_i) = \kappa_{p,d}D^p\psi_{\phi,p}(q_i),
\quad i=0,1.
\end{equation*}
Combining with Lemma~\ref{lem:parameter-powerSOW}, we obtain
\begin{align*}
\left|\operatorname{SOW}_{\phi,p}^p(\mu_1,\nu_1)-\operatorname{SOW}_{\phi,p}^p(\mu_0,\nu_0)\right|&=\kappa_{p,d}D^p\left[
\psi_{\phi,p}(q_1)-\psi_{\phi,p}(q_0)
\right] \\
&\ge 2^{-(p+2)}\kappa_{p,d}D^p
\psi_{\phi,p}^{\mathrm{cav}}\left(\frac{1}{4\sqrt n}\right).
\end{align*}
On the other hand, since $n \le m$
\begin{equation*}
0<\frac{1}{4}\left(\frac{1}{2}n^{-1/2}+\frac{1}{2}m^{-1/2}\right) \le \frac{1}{4\sqrt{n}}\le \frac{1}{4}.
\end{equation*}
By Lemma~\ref{lem:properties-psi-cav}, $\psi_{\phi,p}^{\mathrm{cav}}$ is increasing on $[0,1]$ and satisfies $\psi_{\phi,p}^{\mathrm{cav}}(0)=0$. Therefore, by its concavity,
\begin{equation*}
\psi^{\mathrm{cav}}_{\phi,p}\left(\frac{1}{4\sqrt{n}}\right) \ge \frac{1}{4} \psi ^{\mathrm{cav}}_{\phi,p}\left(\frac{1}{2}n^{-1/2}+\frac{1}{2}m^{-1/2}\right).
\end{equation*}
Consequently,
\begin{equation*}
\left|\operatorname{SOW}_{\phi,p}^p(\mu_1,\nu_1)-\operatorname{SOW}_{\phi,p}^p(\mu_0,\nu_0)\right|\ge 2^{-(p+4)}\kappa_{p,d}D^p\psi^{\mathrm{cav}}_{\phi,p}\left(\frac{1}{2}n^{-1/2}+\frac{1}{2}m^{-1/2}\right).
\end{equation*}
Plugging this estimate, together with the bound on the KL term, into
Lemma~\ref{lem:lecam-two-sample} for the functional
$T(\mu,\nu)=\operatorname{SOW}_{\phi,p}^p(\mu,\nu)$ yields
\begin{equation*}
\inf_{\widehat T}\sup_{\mu,\nu\in\mathcal P(K)}
\mathbb E_{\mu,\nu}
\left|\widehat T-\operatorname{SOW}_{\phi,p}^p(\mu,\nu)\right|\ge2^{-(p+6)}
\left(1-\frac{1}{\sqrt{6}}\right)\kappa_{p,d}D^p\psi_{\phi,p}^{\mathrm{cav}}
\left(\frac{1}{2}n^{-1/2}+\frac{1}{2}m^{-1/2}
\right).
\end{equation*}
This lower bound matches the upper bound in
Theorem~\ref{thm:powerSOW-two-sample} up to a multiplicative constant
depending only on $d$ and $p$, thereby completing the proof.
\end{proof}
\section{Computational details}
\subsection{Closed-form computation for $\phi_1$}\label{app:closeform}
Algorithm~\ref{alg:SOW} approximates $\operatorname{SOW}^p_{\phi,p}\left(\widehat{\mu}_n,\widehat{\nu}_m\right)$ by combining Monte Carlo integration with bisection for the projected $\operatorname{OW}$ distances. When these distances admit closed-form expressions, the bisection steps can be omitted, leaving only the Monte Carlo approximation.

Recall that, for $\theta_1,\dots,\theta_L\overset{\mathrm{i.i.d.}}{\sim}\sigma$, the Monte Carlo approximation is
\begin{equation*}
\widehat{\operatorname{SOW}}^p_{\phi,p}\left(\widehat{\mu}_n,\widehat{\nu}_m,L\right)=\frac{1}{L}\sum_{l=1}^L\left[\inf\left\{\lambda>0:\int_{0}^{1}\phi\left(\frac{\left|\Delta_{\theta_l}(u)\right|}{\lambda}\right)\mathrm{d}u \le 1\right\}\right]^p,
\end{equation*}
where $\Delta_{\theta}$ is the projected quantile gap defined in \eqref{eq:projected-quantile-gap}. Thus, for each direction, it remains to evaluate the scalar infimum inside the sum. Bisection provides a numerical approximation when no closed-form expression is available. 

We consider the Orlicz function $\phi_1(t) = \frac{t^2+t^4}{2}$ with $t \ge 0$, used in our experiments. Fix a direction $\theta$ and define
\begin{equation*}
A_{\theta}\coloneqq \int_{0}^1\left|\Delta_{\theta}(u)\right|^2\mathrm{d}u,\quad B_{\theta}\coloneqq \int_{0}^{1}\left|\Delta_{\theta}(u)\right|^4\mathrm{d}u.
\end{equation*}
For every $\lambda>0$, the constraint in the infimum becomes
\begin{align*}
\int_{0}^{1}\phi_1\left(\frac{\left|\Delta_{\theta}(u)\right|}{\lambda}\right)\mathrm{d} u \le 1 &\Leftrightarrow \frac{A_{\theta}}{2\lambda^2}+\frac{B_{\theta}}{2\lambda^4}\le 1\\
& \Leftrightarrow \lambda ^2 \ge \frac{A_{\theta}+\sqrt{A^2_{\theta}+8B_{\theta}}}{4}.
\end{align*}
Taking the infimum over $\lambda>0$ yields
\begin{equation*}
\operatorname{W}_{\phi_1}\left(\Pi^{\theta}_{\#}\widehat{\mu}_n,\Pi^{\theta}_{\#}\widehat{\nu}_m\right)=\sqrt{\frac{A_{\theta}+\sqrt{A^2_{\theta}+8B_{\theta}}}{4}}.
\end{equation*}
Consequently, each projected distance can be computed directly from $A_{\theta}$ and $B_{\theta}$. Then, the Monte Carlo approximation can be evaluated without bisection.
\subsection{Proof of Proposition~\ref{prop:SOW-computation}}
\label{sec:proof-SOW-computation}
Throughout this proof, the observed samples are held fixed, and all expectations are taken over the projection directions.

Using the triangle inequality, we obtain
\begin{align*}
&\mathbb{E}\left|\widehat{\operatorname{SOW}}_{\phi,p}^p(\widehat{\mu}_n,\widehat{\nu}_m,L,T)-\operatorname{SOW}_{\phi,p}^p(\widehat{\mu}_n,\widehat{\nu}_m)\right|\\
&\le\mathbb{E}\left|\widehat{\operatorname{SOW}}_{\phi,p}^p(\widehat{\mu}_n,\widehat{\nu}_m,L,T)-\widehat{\operatorname{SOW}}_{\phi,p}^p(\widehat{\mu}_n,\widehat{\nu}_m,L)\right|+\mathbb{E}\left|\widehat{\operatorname{SOW}}_{\phi,p}^p(\widehat{\mu}_n,\widehat{\nu}_m,L)-\operatorname{SOW}_{\phi,p}^p(\widehat{\mu}_n,\widehat{\nu}_m)\right|.
\end{align*}

Now, we control the first term $\mathbb{E}\left|\widehat{\operatorname{SOW}}_{\phi,p}^p(\widehat{\mu}_n,\widehat{\nu}_m,L,T)-\widehat{\operatorname{SOW}}_{\phi,p}^p(\widehat{\mu}_n,\widehat{\nu}_m,L)\right|$. We have:
\begin{align*}
&\mathbb{E}\left|\widehat{\operatorname{SOW}}_{\phi,p}^p(\widehat{\mu}_n,\widehat{\nu}_m,L,T)-\widehat{\operatorname{SOW}}_{\phi,p}^p(\widehat{\mu}_n,\widehat{\nu}_m,L)\right|\\
&=\mathbb{E}\left|\frac{1}{L}\sum_{l=1}^L\widehat{\operatorname{W}}^p_\phi(\Pi^{\theta_l}_{\#}\widehat{\mu}_n,\Pi^{\theta_l}_{\#}\widehat{\nu}_m,T)-\frac{1}{L}\sum_{l=1}^L\operatorname{W}^p_\phi(\Pi^{\theta_l}_{\#}\widehat{\mu}_n,\Pi^{\theta_l}_{\#}\widehat{\nu}_m)\right|\\
&=\mathbb{E}\left|\frac{1}{L}\sum_{l=1}^L\left(\widehat{\operatorname{W}}^p_\phi(\Pi^{\theta_l}_{\#}\widehat{\mu}_n,\Pi^{\theta_l}_{\#}\widehat{\nu}_m,T)-\operatorname{W}^p_\phi(\Pi^{\theta_l}_{\#}\widehat{\mu}_n,\Pi^{\theta_l}_{\#}\widehat{\nu}_m)\right)\right|\\
&\le\frac{1}{L}\sum_{l=1}^L\mathbb{E}\left|\widehat{\operatorname{W}}^p_\phi(\Pi^{\theta_l}_{\#}\widehat{\mu}_n,\Pi^{\theta_l}_{\#}\widehat{\nu}_m,T)-\operatorname{W}^p_\phi(\Pi^{\theta_l}_{\#}\widehat{\mu}_n,\Pi^{\theta_l}_{\#}\widehat{\nu}_m)\right|.
\end{align*}

From \eqref{eq:bracket}, we have
$$0\le L_\theta\le\operatorname{W}_\phi(\Pi^{\theta}_{\#}\widehat{\mu}_n,\Pi^{\theta}_{\#}\widehat{\nu}_m)\le U_\theta,\qquad\theta\in\mathbb{S}^{d-1},$$
so that $[L_\theta,U_\theta]$ is a valid initial bracket for Algorithm~\ref{alg:SOW}. Since $\phi$ is increasing, $\lambda\mapsto\phi(|\Delta_\theta(u)|/\lambda)$ is decreasing for each $u$. Therefore each bisection step retains $\operatorname{W}_\phi(\Pi^{\theta_l}_{\#}\widehat{\mu}_n,\Pi^{\theta_l}_{\#}\widehat{\nu}_m)$ in the current bracket while halving its width. In particular $\widehat{\operatorname{W}}_\phi(\Pi^{\theta_l}_{\#}\widehat{\mu}_n,\Pi^{\theta_l}_{\#}\widehat{\nu}_m,T)\in[L_{\theta_l},U_{\theta_l}]$, and the error of the bisection method~\cite{burden2011numerical} is
$$\left|\widehat{\operatorname{W}}_\phi(\Pi^{\theta_l}_{\#}\widehat{\mu}_n,\Pi^{\theta_l}_{\#}\widehat{\nu}_m,T)-\operatorname{W}_\phi(\Pi^{\theta_l}_{\#}\widehat{\mu}_n,\Pi^{\theta_l}_{\#}\widehat{\nu}_m)\right|\le\frac{|U_{\theta_l}-L_{\theta_l}|}{2^T}.$$
If $L_{\theta_l}=U_{\theta_l}$, the algorithm returns their common value without bisection, so the error is zero and the same bound still holds.

Moreover, since $p\ge1$, the map $s\mapsto s^p$ has increasing derivative $ps^{p-1}$ on $(0,\infty)$, so the mean value theorem gives $|a^p-b^p|\le pM^{p-1}|a-b|$ for all $a,b\in[0,M]$. Applying this with $M=U_{\theta_l}$, we obtain:
\begin{align*}
&\mathbb{E}\left[\left|\widehat{\operatorname{SOW}}_{\phi,p}^p(\widehat{\mu}_n,\widehat{\nu}_m,L,T)-\widehat{\operatorname{SOW}}_{\phi,p}^p(\widehat{\mu}_n,\widehat{\nu}_m,L)\right|\right]\\
&\le\frac{p}{L}\sum_{l=1}^L\mathbb{E}\left[\left|\widehat{\operatorname{W}}_\phi(\Pi^{\theta_l}_{\#}\widehat{\mu}_n,\Pi^{\theta_l}_{\#}\widehat{\nu}_m,T)-\operatorname{W}_\phi(\Pi^{\theta_l}_{\#}\widehat{\mu}_n,\Pi^{\theta_l}_{\#}\widehat{\nu}_m)\right|U_{\theta_l}^{p-1}\right]\\
&\le\frac{p}{L}\sum_{l=1}^L\mathbb{E}\left[\frac{|U_{\theta_l}-L_{\theta_l}|}{2^T}U_{\theta_l}^{p-1}\right]\\
&\le\frac{p}{L}\sum_{l=1}^L\mathbb{E}\left[\frac{U_{\theta_l}}{2^T}U_{\theta_l}^{p-1}\right]\\
&=\frac{p}{L2^{T}}\sum_{l=1}^L\mathbb{E}\left[U_{\theta_l}^p\right]=\frac{p}{2^{T}}\int_{\mathbb{S}^{d-1}}U_{\theta}^p\,\mathrm{d}\sigma(\theta),
\end{align*}
where the third inequality uses $0\le L_{\theta_l}\le U_{\theta_l}$, and the last equality holds since $\theta_1,\dots,\theta_L\overset{\mathrm{i.i.d.}}{\sim}\sigma$.

Now, we control the second term $\mathbb{E}\left|\widehat{\operatorname{SOW}}_{\phi,p}^p(\widehat{\mu}_n,\widehat{\nu}_m,L)-\operatorname{SOW}_{\phi,p}^p(\widehat{\mu}_n,\widehat{\nu}_m)\right|$. We follow the argument of \cite[Theorem~6]{nadjahi2020statistical}, with $\operatorname{W}_{\phi}$ as the base distance.

By the Cauchy--Schwarz inequality, 
\begin{align*}
&\mathbb{E}\left|\widehat{\operatorname{SOW}}_{\phi,p}^p(\widehat{\mu}_n,\widehat{\nu}_m,L)-\operatorname{SOW}_{\phi,p}^p(\widehat{\mu}_n,\widehat{\nu}_m)\right|\\
&=\mathbb{E}\left|\frac{1}{L}\sum_{l=1}^L\operatorname{W}_\phi^p(\Pi^{\theta_l}_{\#}\widehat{\mu}_n,\Pi^{\theta_l}_{\#}\widehat{\nu}_m)-\mathbb{E}\left[\operatorname{W}_\phi^p(\Pi^{\theta}_{\#}\widehat{\mu}_n,\Pi^{\theta}_{\#}\widehat{\nu}_m)\right]\right|\\
&\le\sqrt{\mathbb{E}\left|\frac{1}{L}\sum_{l=1}^L\operatorname{W}_\phi^p(\Pi^{\theta_l}_{\#}\widehat{\mu}_n,\Pi^{\theta_l}_{\#}\widehat{\nu}_m)-\mathbb{E}\left[\operatorname{W}_\phi^p(\Pi^{\theta}_{\#}\widehat{\mu}_n,\Pi^{\theta}_{\#}\widehat{\nu}_m)\right]\right|^2}\\
&=\sqrt{\operatorname{Var}\left[\frac{1}{L}\sum_{l=1}^L\operatorname{W}_\phi^p(\Pi^{\theta_l}_{\#}\widehat{\mu}_n,\Pi^{\theta_l}_{\#}\widehat{\nu}_m)\right]}=\frac{1}{\sqrt{L}}\sqrt{\operatorname{Var}_{\theta\sim\sigma}\operatorname{W}_\phi^p(\Pi^{\theta}_{\#}\widehat{\mu}_n,\Pi^{\theta}_{\#}\widehat{\nu}_m)}.
\end{align*}
Combining results, we have:
\begin{align*}
&\mathbb{E}\left|\widehat{\operatorname{SOW}}_{\phi,p}^p(\widehat{\mu}_n,\widehat{\nu}_m,L,T)-\operatorname{SOW}_{\phi,p}^p(\widehat{\mu}_n,\widehat{\nu}_m)\right|\\
&\le\frac{p}{2^{T}}\int_{\mathbb{S}^{d-1}}U_{\theta}^p\,\mathrm{d}\sigma(\theta)+\frac{1}{\sqrt{L}}\operatorname{Var}_{\theta\sim\sigma}\left[\operatorname{W}_\phi^p(\Pi^{\theta}_{\#}\widehat{\mu}_n,\Pi^{\theta}_{\#}\widehat{\nu}_m)\right]^{1/2},
\end{align*}
which completes the proof. 
\section{Beyond bounded support}\label{app:beyondboundedsupport}
The statistical results in the main text are developed under a common bounded-support assumption. A natural question is whether this assumption can be relaxed. For the Wasserstein and sliced Wasserstein distances, a standard route is to replace bounded support by finite higher-order moments (see, for example, \cite{fournier2015rate,nadjahi2020statistical}). Although such conditions allow genuinely unbounded distributions, the same approach can fail for a general Orlicz generator $\phi$. The purpose of this appendix
is not to characterize the largest class of unbounded distributions for which empirical $\operatorname{SOW}_{\phi,p}$ convergence holds. Instead, in this discussion, we ask a more modest question: If $\mathbb{E}\operatorname{SOW}^p_{\phi,p}(\widehat{\mu}_n,\mu)\to 0$ as $n\to\infty$, where must the target law $\mu$ lie? We first give a counterexample to demonstrate that polynomial moment conditions can still be insufficient, then isolate the obstruction suggested by that example, and finally derive a necessary tail-integrability condition.
\subsection{Why polynomial moment conditions can still be insufficient}
\begin{example}[All polynomial moments may still be insufficient]\label{ex:appendix-beyond}
Let $d=1$ and let $\phi(t)=e^t-t-1$ for $t\geq 0$. Consider the exponential distribution with rate $1$, denoted by $\mu=\mathrm{Exp}(1)$, and let $X\sim\mu$. First, $\mu$ has finite moments of every polynomial order. Indeed, for every $q>0$,
\begin{equation*}
\mathbb{E}\left[X^q\right]=\Gamma(q+1)<+\infty, 
\end{equation*}
where $\Gamma$ denotes the gamma function.

We next show that $\mu\in\mathcal{P}_{\phi}(\mathbb{R})$. For every $\lambda>1$, we have
\begin{equation*}
\mathbb{E} e^{X/\lambda} = \int_{0}^{+\infty} e^{x/\lambda}e^{-x}\mathrm{d}x=\frac{\lambda}{\lambda -1}.
\end{equation*}
Since $\mathbb{E}[X]=1$, it follows that, for every $\lambda>1$,
\begin{equation*}
\mathbb{E}\phi(\left|X\right|/\lambda)=\mathbb{E}\phi(X/\lambda)=\mathbb{E}e^{X/\lambda}-\mathbb{E}\left(X/\lambda\right)-1=\frac{\lambda}{\lambda -1}-\frac{1}{\lambda }-1<+\infty.
\end{equation*}
In particular, taking $\lambda=2$, we obtain
\begin{equation*}
\mathbb{E}\phi\left(\frac{|X|}{2}\right)=\frac{1}{2}\leq 1.
\end{equation*}
Therefore, $\mu\in\mathcal{P}_{\phi}(\mathbb{R})$.

Nevertheless, let $\nu$ be any probability measure with bounded
support. More precisely, assume that, for some $M>0$,
\begin{equation*}
\operatorname{supp}\left(\nu\right) \subset[-M,M].
\end{equation*}
Let $(X,Y)$ be an arbitrary coupling of $\mu$ and $\nu$. Since
$Y\leq M$ almost surely, on the event $\{X>M\}$ we have
\begin{equation*}
\left|X-Y\right| \ge X-M.
\end{equation*}
Fix $0<\lambda \le 1$. By the monotonicity of $\phi$, we obtain
\begin{align*}
\mathbb{E}\phi\left(\frac{\left|X-Y\right|}{\lambda}\right)&\ge \int_{M}^{+\infty}\phi\left(\frac{x-M}{\lambda}\right)e^{-x}\mathrm{d}x\\
&=\int_{M}^{+\infty}\left[e^{(x-M)/\lambda}-\frac{x-M}{\lambda}-1\right]e^{-x}\mathrm{d}x.
\end{align*}
The contributions of the last two terms are finite, whereas the
exponential term is not integrable. Indeed,
\begin{equation*}
\int_{M}^{+\infty}e^{(x-M)/\lambda}e^{-x}\mathrm{d}x=e^{-M/\lambda}\int_{M}^{+\infty}e^{(1/\lambda -1)x}\mathrm{d}x=+\infty,
\end{equation*}
because $1/\lambda-1 \ge 0$. Therefore, no $\lambda\leq1$ is admissible in the definition
of the Luxemburg norm, and hence
\begin{equation*}
\|X-Y\|_{\phi}\ge 1.
\end{equation*}
Since the coupling $(X,Y)$ was arbitrary, we conclude that
\begin{equation*}
\operatorname{W}_{\phi}\left(\mu,\nu\right)\ge 1.
\end{equation*}
Now, for every sample realization, the empirical measure $\widehat{\mu}_n=\frac{1}{n}\sum_{i=1}^n\delta_{X_i}$ has bounded support. Applying the preceding argument with $\nu=\widehat{\mu}_n$ gives
\begin{equation*}
\operatorname{W}_{\phi}\left(\mu,\widehat{\mu}_n\right)\ge 1, \quad \text{for every } n\geq1.
\end{equation*}
In dimension one, $\operatorname{SOW}_{\phi,p}=\operatorname{W}_{\phi}$, and hence 
\begin{equation*}
\mathbb{E}\left[\operatorname{SOW}^p_{\phi,p}\left(\widehat{\mu}_n,\mu\right)\right]\ge 1, \quad \text{for every } n\geq1.
\end{equation*}
\end{example}
In this example, although $\mu$ already belongs to $\mathcal{P}_{\phi}\left(\mathbb{R}\right)$ and it has finite moments of every polynomial order, the empirical measures do not converge to $\mu$ in $\operatorname{SOW}_{\phi,p}$ for any sample realization. Thus one-scale Orlicz integrability, even together with finite moments of every polynomial order, can still be insufficient. However, the proof also suggests why the approximation fails: empirical measures are boundedly supported,
while the far tail of the target law remains expensive in the Luxemburg geometry. We now isolate this mechanism.
\subsection{Insights from the Counterexample}\label{appendix:beyond-insights}
We begin with the one-dimensional setting. Let $\mu$ be a probability
measure on $\mathbb{R}$, and let $X\sim\mu$. Assume that $\mu$ can be
approximated in $\operatorname{W}_{\phi}$ by compactly supported probability measures. More precisely, suppose that there exists a
sequence $(\nu_k)_{k\in\mathbb{N}}$ of compactly supported probability
measures satisfying
\begin{equation}\label{eq:appendix-beyond-eq1}
\operatorname{W}_{\phi}\left(\mu,\nu_k\right) \to 0 \quad \text{ as } k\to +\infty.
\end{equation}
For each $k$, choose $0<M_k<+\infty$ such that 
\begin{equation*}
\operatorname{supp}\left(\nu_k\right) \subset [-M_k,M_k].
\end{equation*}
Let $(X,Y_k)$ be any coupling of $\mu$ and $\nu_k$. Since $\left|Y_k\right|\le M_k$ almost surely,
\begin{equation*}
\left|X-Y_k\right| \ge \left|X\right|-\left|Y_k\right|\ge \left|X\right|-M_k.
\end{equation*}
Because the left-hand side is nonnegative,
\begin{equation*}
\left|X-Y_k\right| \ge \left(\left|X\right|-M_k\right)_+,
\end{equation*}
where we denote $(a)_+=\max\{a,0\}$.
By monotonicity of the Luxemburg norm,
\begin{equation*}
\left\|X-Y_k\right\|_{\phi} \ge \left\|\left(\left|X\right|-M_k\right)_{+}\right\|_{\phi}.
\end{equation*}
Consequently,
\begin{equation*}
\operatorname{W}_{\phi}\left(\mu,\nu_k\right) \ge \|\left(\left|X\right|-M_k\right)_{+}\|_{\phi}.
\end{equation*}
Combining with \eqref{eq:appendix-beyond-eq1} yields
\begin{equation}\label{eq:appendix-beyond-eq2}
\|\left(\left|X\right|-M_k\right)_{+}\|_{\phi} \to 0 \quad \text{ as } k \to +\infty.
\end{equation}
Now define
\begin{equation*}
\tau(R)\coloneqq\|\left(\left|X\right|-R\right)_{+} \|_{\phi}, \quad R \ge 0.
\end{equation*}
It is easy to check that the function $R \mapsto \tau(R)$ is decreasing. For every $\varepsilon>0$, \eqref{eq:appendix-beyond-eq2} provides an index $k$ such that $\tau(M_k)<\varepsilon$. Hence, for every $R\ge M_k$, we have $\tau(R)\le\tau(M_k)<\varepsilon$. As a consequence,
\begin{equation*}
\|\left(\left|X\right|-R\right)_{+}\|_{\phi} \to 0 \quad \text{ as } R \to +\infty.
\end{equation*}

This is the first structural conclusion suggested by Example~\ref{ex:appendix-beyond}: if boundedly supported measures can approach the target law in Orlicz–Wasserstein distance, then the far tail of the target must disappear in Luxemburg norm.

\paragraph{Informal interpretation.} Recall that the Luxemburg norm is
\begin{equation*}
\|X\|_{\phi} = \inf \left\{\lambda >0:\mathbb{E}\phi \left(\frac{\left|X\right|}{\lambda}\right)\le 1\right\}.
\end{equation*}
If the tail $\left(\left|X\right|-R\right)_{+}$ is to become arbitrarily small in this norm, then one must eventually control it at arbitrarily small Luxemburg scales. A small scale $\epsilon$ places a large factor $1/\epsilon$ inside $\phi$. This suggests that one-scale integrability may not be enough. Therefore, a natural candidate is integrability at every positive multiplicative scale, namely
\begin{equation*}
\mathbb{E}\phi\left(a\left|X\right|\right)<+\infty \quad \text{for every }a>0. 
\end{equation*}
The next lemma verifies that intuition.
\begin{lemma}\label{appendix:beyond-usefullemma}
Let $Z$ be a real-valued random variable. If $\displaystyle\lim_{R \to +\infty}\|\left(\left|Z\right|-R\right)_{+} \|_{\phi}=0$, then $\mathbb{E} \phi\left(a\left|Z\right|\right)<+\infty$ for every $a>0$.
\end{lemma}
\begin{proof}
Fix $a>0$. Since $\displaystyle\lim_{R \to +\infty}\|\left(\left|Z\right|-R\right)_{+} \|_{\phi}=0$, we can choose $R$ large enough such that
\begin{equation*}
\|\left(\left|Z\right|-R\right)_{+}\|_{\phi}<1/(2a).
\end{equation*}
Since $\mathbb{R}$ is densely ordered, we may choose
$t\in\mathbb{R}$ such that
\begin{equation*}
\|\left(\left|Z\right|-R\right)_{+}\|_{\phi}<t<1/(2a).
\end{equation*}
We first justify that $t$ is an admissible Luxemburg scale. Since $t>\|\left(\left|Z\right|-R\right)_{+}\|_{\phi}$, there exists  $0<\lambda_0<t$ such that
\begin{equation*}
\mathbb{E}\phi\left(\frac{\left(\left|Z\right|-R\right)_{+}}{\lambda_0}\right) \le 1.
\end{equation*}
Since $t> \lambda_0$, the monotonicity of $\phi$ yields
\begin{equation*}
\mathbb{E}\phi\left(\frac{\left(\left|Z\right|-R\right)_{+}}{t}\right) \le 1.
\end{equation*}
Besides, since $2a<1/t$, again, the monotonicity of $\phi$ gives
\begin{equation}\label{eq:appendix-beyond-eq3}
\mathbb{E}\phi\left(2a\left(\left|Z\right|-R\right)_{+}\right) \le 1.
\end{equation}
Finally, since $\left|Z\right| \le R+\left(\left|Z\right|-R\right)_{+}$, using monotonicity and convexity of $\phi$, we have 
\begin{equation*}
\phi\left(a\left|Z\right|\right) \le \frac{1}{2}\phi\left(2aR\right)+\frac{1}{2}\phi\left[2a\left(\left|Z\right|-R\right)_{+}\right].
\end{equation*}
Taking expectation and using \eqref{eq:appendix-beyond-eq3},
\begin{equation*}
\mathbb{E}\phi\left(a\left|Z\right|\right)\le \frac{1}{2}\phi\left(2aR\right)+\frac{1}{2}<+\infty.
\end{equation*}
Since $a>0$ is arbitrary, we complete the proof.
\end{proof}
With this ingredient in hand, we now state a necessary condition for vanishing expected empirical $\operatorname{SOW}^p_{\phi,p}$ loss.
\begin{proposition}
Let $p \in [1,+\infty)$ and let $\mu \in \mathcal{P}_{\phi}(\mathbb{R}^d)$. Let $\widehat{\mu}_n$ denote the empirical measure of $n$ i.i.d. samples drawn from $\mu$. If $\displaystyle \lim_{n \to +\infty}\mathbb{E}\operatorname{SOW}^p_{\phi,p}\left(\widehat{\mu}_n,\mu\right)=0$, then $\mu \in \mathcal{P}^{\mathrm{all}}_{\phi}\left(\mathbb{R}^d\right)$, where
\begin{equation*}
\mathcal{P}^{\mathrm{all}}_{\phi}\left(\mathbb{R}^d\right)\coloneqq \left\{P \in \mathcal{P}\left(\mathbb{R}^d\right) :\int_{\mathbb{R}^d}\phi\left(\frac{\|x\|}{\lambda}\right)\mathrm{d}P(x)<+\infty\quad \forall \lambda >0\right\}.
\end{equation*}
\end{proposition}
\begin{proof}
Since $\displaystyle \lim_{n \to +\infty}\mathbb{E}\operatorname{SOW}^p_{\phi,p}\left(\widehat{\mu}_n,\mu\right)=0$, we can choose a subsequence $\left(n_k\right)_{k \in \mathbb{N}}$ such that
\begin{equation*}
\mathbb{E}\operatorname{SOW}^p_{\phi,p}\left(\widehat{\mu}_{n_k},\mu\right) \le 2^{-2kp}.
\end{equation*}
By Markov's inequality,
\begin{equation*}
\mathbb{P}\left(\operatorname{SOW}^p_{\phi,p}\left(\widehat{\mu}_{n_k},\mu\right)>2^{-k}\right)\le 2^k\mathbb{E}
\operatorname{SOW}^p_{\phi,p}(\widehat{\mu}_{n_k},\mu) \le 2^{-kp}.
\end{equation*}
Since $p\ge 1$, $\displaystyle \sum_{k=1}^{\infty}2^{-kp}<+\infty$. Therefore, the Borel–Cantelli lemma gives
\begin{equation}\label{eq:eq-appendix-beyond-eq4}
\lim_{k \to +\infty}\operatorname{SOW}^p_{\phi,p}\left(\widehat{\mu}_{n_k},\mu\right)=0 \quad \text{almost surely.}
\end{equation}
We now fix a sample path on which \eqref{eq:eq-appendix-beyond-eq4} holds. Using Proposition~\ref{prop:SOW-maxSOW}, for every coordinate vector $e_j$, we have
\begin{equation}\label{eq:eq-appendix-beyond-eq5}
\operatorname{W}_{\phi}\left(\Pi^{e_j}_{\#}\widehat{\mu}_{n_k},\Pi^{e_j}_{\#}\mu\right) \to 0 \quad \text{ as } k \to +\infty. 
\end{equation}
Write $X=(X_1,\dots,X_d)\sim \mu$. For every $k$, the projected measure $\Pi^{e_j}_{\#}\widehat{\mu}_{n_k}$ has finite support. Combining the argument in Section~\ref{appendix:beyond-insights} with \eqref{eq:eq-appendix-beyond-eq5}, we deduce
\begin{equation*}
\|\left(\left|X_j\right|-R\right)_{+}\|_{\phi} \to 0 \quad \text{ as } R \to +\infty.
\end{equation*}
Applying Lemma~\ref{appendix:beyond-usefullemma}, we deduce
\begin{equation*}
\mathbb{E}\phi\left(b\left|X_j\right|\right)< +\infty \quad \text{ for every } b>0 \text{ and every } j=1,\dots,d. 
\end{equation*}
Finally, for a fixed $a>0$, the monotonicity and convexity of $\phi$ give
\begin{equation*}
\phi\left(a\|X\|\right)\le \phi\left(a\sum_{j=1}^{d}\left|X_j\right|\right)=\phi\left(\frac{1}{d}\sum_{j=1}^{d}da\left|X_j\right|\right)\le \frac{1}{d}\sum_{j=1}^{d}\phi\left(da\left|X_j\right|\right).
\end{equation*}
By the preceding conclusion, each term on the right has finite expectation. Hence,
\begin{equation*}
\mathbb{E}\phi\left(a\|X\|\right)<+\infty.
\end{equation*}
Since $a>0$ is arbitrary, $\mu \in \mathcal{P}_{\phi}^{\text{all}}\left(\mathbb{R}^d\right)$.
\end{proof}
The preceding proposition identifies the tail restriction imposed by Orlicz geometry. However, it does not yet exploit the additional integrability information carried by the $p$-th power in the expected loss. In fact, the same assumption also forces the target distribution to have a finite $p$-th moment. Combining these two restrictions yields the following stronger necessary condition
\begin{theorem}
Let $p\in[1,+\infty)$ and let $\mu\in \mathcal{P}_{\phi}\left(\mathbb{R}^d\right)$. Let $\widehat{\mu}_n$ denote the empirical measure of $n$ i.i.d. samples drawn from $\mu$. If $\displaystyle \lim_{n \to +\infty}\mathbb{E}\operatorname{SOW}^p_{\phi,p}\left(\widehat{\mu}_n,\mu\right)=0$, then $\mu \in \mathcal{P}^{\mathrm{all}}_{\phi}\left(\mathbb{R}^d\right)\cap \mathcal{P}_p\left(\mathbb{R}^d\right)$.
\end{theorem}
\begin{proof}
By the preceding proposition, it remains only to prove that $\mu \in \mathcal{P}_{p}\left(\mathbb{R}^d\right)$.
Let $X_1,X_2,\dots,X_n$ be i.i.d. with law $\mu$ and define
\begin{equation*}
\bar{X}_n\coloneqq \frac{1}{n}\sum_{i=1}^n X_i,\quad m\coloneqq \mathbb{E}X_1.
\end{equation*}
The vector $m$ is well defined. Indeed, since $\mu \in \mathcal{P}_{\phi}\left(\mathbb{R}^d\right)$, Lemma~\ref{lem:W1-Wphi} gives
\begin{equation*}
\operatorname{W}_1\left(\mu,\delta_0\right) \le \phi^{-1}(1)\operatorname{W}_{\phi}\left(\mu,\delta_0\right)<+\infty.
\end{equation*}
Therefore,
\begin{equation*}
\mathbb{E}\|X_1\| = \operatorname{W}_1\left(\mu,\delta_0 \right) <+\infty.
\end{equation*}
For any $\theta \in \mathbb{S}^{d-1}$, let $(U,V)$ be an arbitrary coupling of $\Pi^{\theta}_{\#}\widehat{\mu}_n$ and $\Pi^{\theta}_{\#} \mu$. Then,
\begin{equation*}
\mathbb{E}\left|U-V\right|\ge \left|\mathbb{E}U-\mathbb{E}V\right|=\left|\langle \theta,\bar{X}_n\rangle -\langle \theta,m\rangle \right|=\left|\langle \theta , \bar{X}_n-m\rangle \right|.
\end{equation*}
Taking the infimum over all couplings yields
\begin{equation*}
\operatorname{W}_1\left(\Pi^{\theta}_{\#}\widehat{\mu}_n,\Pi^{\theta}_{\#}\mu\right)\ge \left|\langle \theta,\bar{X}_n-m\rangle \right|.
\end{equation*}
By Lemma~\ref{lem:W1-Wphi}, we have
\begin{equation*}
\operatorname{W}_{\phi}\left(\Pi^{\theta}_{\#}\widehat{\mu}_n,\Pi^{\theta}_{\#}\mu\right) \ge \frac{\left|\langle \theta,\bar{X}_n-m\rangle \right|}{\phi^{-1}(1)}.
\end{equation*}
Hence,
\begin{equation*}
\operatorname{SOW}^p_{\phi,p}\left(\widehat{\mu}_n,\mu\right) \ge \frac{1}{\left[\phi^{-1}(1)\right]^p}\int_{\mathbb{S}^{d-1}}\left|\langle \theta,\bar{X}_n-m\rangle\right|^p \mathrm{d}\sigma(\theta)=\frac{\kappa_{p,d}}{\left[\phi^{-1}(1)\right]^p}\|\bar{X}_n-m\|^p,
\end{equation*}
where we used \eqref{eq:sphereuseful} in the last equality. Therefore,
\begin{equation*}
\mathbb{E}\|\bar{X}_n-m\|^p \le \frac{\left[\phi^{-1}(1)\right]^p}{\kappa_{p,d}}\mathbb{E}\operatorname{SOW}^p_{\phi,p}\left(\widehat{\mu}_n,\mu\right) \xrightarrow[]{n\to +\infty}0.
\end{equation*}
As a consequence, there exists $n_0 \in \mathbb{N}$ with $n_0\ge 2$, we have
\begin{equation*}
\mathbb{E}\|\bar{X}_{n_0}-m\|^p<+\infty.
\end{equation*}
We set $Z_i\coloneqq X_i-m$ and $R \coloneqq \sum_{i=2}^{n_0}Z_i$. Then, $Z_1$ and $R$ are independent and $Z_1+R=n_0\left(\bar{X}_{n_0}-m\right)$. So, $\mathbb{E}\|Z_1+R\|^p<+\infty$. The independence of $Z_1$ and $R$, together with Tonelli's theorem gives
\begin{equation*}
\int_{\mathbb{R}^d}\int_{\mathbb{R}^d}\|z+r\|^p\mathrm{d}\rho(z)d\eth(r) <+\infty,
\end{equation*}
where $\rho$ and $\eth$ denote the laws of $Z_1$ and $R$ respectively. As a result,
\begin{equation*}
\int_{\mathbb{R}^d}\|z+r\|^pd\rho(z) <+\infty \quad \eth\text{-almost every $r$}.
\end{equation*}
Fix one such $r_0$. Since $p \ge 1$, for all $z \in \mathbb{R}^d$ we have
\begin{equation*}
\|z\|^p \le 2^{p-1}\left(\|z+r_0\|^p+\|r_0\|^p\right).
\end{equation*}
Taking the integral with respect to $\rho$ yields
\begin{equation*}
\mathbb{E}\|Z_1\|^p<+\infty.
\end{equation*}
Finally, since $X_1=Z_1+m$, we have
\begin{equation*}
\mathbb{E}\|X_1\|^p \le 2^{p-1}\left(\mathbb{E}\|Z_1\|^p+\|m\|^p\right)<+\infty.
\end{equation*}
Consequently, $\mu \in \mathcal{P}_{p}\left(\mathbb{R}^d\right)$.
\end{proof}

The theorem above gives a necessary condition for $\mathbb{E}[\operatorname{SOW}^p_{\phi,p}\left(\widehat{\mu}_n,\mu\right)]$ converges to $0$ as $n$ goes to $+\infty$. It shows that any class of unbounded distributions on which the expected empirical loss vanishes must be contained in $\mathcal P_\phi^{\mathrm{all}}(\mathbb R^d)\cap \mathcal{P}_{p}\left(\mathbb{R}^d\right)$. Determining whether $\mathcal P_\phi^{\mathrm{all}}(\mathbb R^d)\cap \mathcal{P}_{p}\left(\mathbb{R}^d\right)$ is sufficient for vanishing expected empirical loss or whether one must impose a smaller tail class depending on both $\phi$ and $p$ is left for future work.
\section{Technical contributions and proof roadmap}\label{app:technicalcontribution}
In this appendix, we provide a proof roadmap for the main statistical results established in this paper. Our goal is to make explicit where the classical sliced Wasserstein analysis relies on the power-law structure  $\phi(t)=t^p$, and to highlight the additional technical ingredients required to extend these arguments to a general Orlicz function
\subsection{Empirical convergence of the $\operatorname{SOW}$ distance}
We first consider the one-sample problem in Theorem~\ref{thm:SOW-one-sample}. Fix a direction $\theta \in \mathbb{S}^{d-1}$. Since both projected measures are supported on an interval of length at most $D>0$, Lemma~\ref{lem:Wphi-W1-compact} relates the projected $\operatorname{OW}$ distance to the ordinary $\operatorname{W}_1$ distance
\begin{equation*}
\operatorname{W}_{\phi}\left(\Pi^{\theta}_{\#}\widehat{\mu}_n,\Pi^{\theta}_{\#}\mu\right) \le D\omega_{\phi}\left(\frac{\operatorname{W}_1\left(\Pi^{\theta}_{\#}\widehat{\mu}_n,\Pi^{\theta}_{\#}\mu\right)}{D}\right).
\end{equation*}
The advantage of this comparison is that $\operatorname{W}_1$ has a simple one-dimensional CDF representation (see, for example, \cite[Theorem~2.9]{bobkov2019one}). If $\widehat{F}_{n,\theta}$ and $F_{\theta}$ denote the CDFs of the two projected measures, then
\begin{equation*}
\operatorname{W}_1\left(\Pi^{\theta}_{\#}\widehat{\mu}_n,\Pi^{\theta}_{\#}\mu\right)=\int_{\mathbb{R}}\left|\widehat{F}_{n,\theta}(x)-F_{\theta}(x)\right|\mathrm{d}x \le D\|\widehat{F}_{n,\theta}-F_{\theta}\|_{\infty}.
\end{equation*}
Consequently, we have
\begin{equation*}
\operatorname{W}_{\phi}\left(\Pi^{\theta}_{\#}\widehat{\mu}_n,\Pi^{\theta}_{\#}\mu\right) \le D\omega_{\phi}\left(\Delta_{n,\theta}\right),
\end{equation*}
where $\Delta_{n,\theta}=\|\widehat{F}_{n,\theta}-F_{\theta}\|_{\infty}$. At this point, the estimation problem is reduced to a single quantity $\mathbb{E}\left[\omega_{\phi}\left(\Delta_{n,\theta}\right)^p\right]$. On the other hand, by the Dvoretzky-Kiefer-Wolfowitz (DKW) inequality, we have
\begin{equation*}
\mathbb{P}\left(\Delta_{n,\theta}>u\right) \le 2\exp \left(-2nu^2\right).
\end{equation*}
In the classical sliced Wasserstein case, $\phi(t)=t^p$ and therefore $\omega_{\phi}(u)^p=u$. Hence, the remaining quantity becomes simply $\mathbb{E}\left[\Delta_{n,\theta}\right]$, which can be directly controlled by using DKW inequality. Indeed, we have
\begin{equation*}
\mathbb{E}\left[\Delta_{n,\theta}\right]=\int_{0}^{+\infty}\mathbb{P}\left(\Delta_{n,\theta}>u\right)\mathrm{d}u \lesssim n^{-1/2}.
\end{equation*}
For a general Orlicz function, however, the identity $\omega_{\phi}(u)^p=u$ is no longer available. This is precisely where an additional Orlicz-specific argument is needed. Lemma~{} addresses this gap by decomposing the range of $\Delta_{n,\theta}$ into suitable scales. This decomposition exploits the growth property of $\omega_{\phi}$ established in Lemma~\ref{lem:properties-omega-psi}, allowing the DKW tail bound to be applied properly and yielding the desired bound.
\subsection{Estimation error of the powered SOW functional}
We next consider the two-sample estimation problem for the powered functional $\operatorname{SOW}^p_{\phi,p}$ in Theorem~\ref{thm:powerSOW-two-sample}. Fix a direction $\theta \in \mathbb{S}^{d-1}$, and define the normalized empirical and population quantile-gap profiles by
\begin{equation*}
\widehat{h}_{\theta}(s) =\frac{\left|F^{-1}_{\Pi^{\theta}_{\#}\widehat{\mu}_n}(s)-F^{-1}_{\Pi^{\theta}_{\#}\widehat{\nu}_m}(s)\right|}{D}, \quad h_{\theta}(s)=\frac{\left|F^{-1}_{\Pi^{\theta}_{\#}\mu}(s)-F^{-1}_{\Pi^{\theta}_{\#}\nu}(s)\right|}{D}, \quad s\in(0,1).
\end{equation*}
Both $\widehat{h}_{\theta}$ and $h_{\theta}$ take values in $[0,1]$. By the one-dimensional quantile representation and the homogeneity of the Luxemburg norm, we have
\begin{equation*}
\left|\operatorname{W}^p_{\phi}\left(\Pi^{\theta}_{\#}\widehat{\mu}_n,\Pi^{\theta}_{\#}\widehat{\nu}_m\right)-\operatorname{W}^p_{\phi}\left(\Pi^{\theta}_{\#}\mu,\Pi^{\theta}_{\#}\nu\right)\right|=D^p\left|\|\widehat{h}_{\theta}\|_{\phi}^p-\|h_{\theta}\|^p_{\phi}\right|.
\end{equation*}
In the classical sliced Wasserstein setting, a key step in obtaining the parametric rate for the powered functional is to reduce its variation to projected $\operatorname{W}_1$ errors (see, for example, the proof of Proposition~1 in \cite{nietert2022statistical}). In view of the representation above, obtaining an analogous reduction in the Orlicz setting amounts to controlling $\left|\|\widehat{h}_{\theta}\|^p_{\phi}-\|h_{\theta}\|^p_{\phi}\right|$ in terms of the $L^1$ discrepancy $\|\widehat{h}_{\theta}-h_{\theta}\|_{L^1}$. When $\phi(t)=t^p$, this control is immediate. Indeed, since $\|\widehat{h}_{\theta}\|^p_{\phi}=\int_{0}^{1}|\widehat{h}_{\theta}(s)|^p\mathrm{d}s$,$\|h_{\theta}\|^p_{\phi}=\int_{0}^{1}\left|h_{\theta}(s)\right|^p\mathrm{d}s$ and $0 \le \widehat{h}_{\theta},h_{\theta} \le 1$, we have
\begin{equation*}
\left|\|\widehat{h}_{\theta}\|^p_{\phi}-\|h_{\theta}\|^p_{\phi}\right| \le p\|\widehat{h}_{\theta}-h_{\theta}\|_{L^1}. 
\end{equation*}
Then, using the triangle inequality and the one-dimensional quantile representation of $\operatorname{W}_1$, we have
\begin{equation*}
\|\widehat{h}_{\theta}-h_{\theta}\|_{L^1} \le \frac{1}{D}\left[\operatorname{W}_1\left(\Pi^{\theta}_{\#}\widehat{\mu}_n,\Pi^{\theta}_{\#}\mu\right)+\operatorname{W}_1\left(\Pi^{\theta}_{\#}\widehat{\nu}_m,\Pi^{\theta}_{\#}\nu\right)\right],
\end{equation*}
where the expectation of the right-hand side can be controlled by using \cite[Theorem~3.2]{bobkov2019one}.

For a general Orlicz function, however, the preceding $L^1$-perturbation argument is no longer available. We therefore introduce the worst-case perturbation modulus
\begin{equation*}
\Omega_{\phi,p}(u) =\sup_{\substack{0\le f,g\le1\\\|f-g\|_{L^1(0,1)}\le u}} \left|\|f\|_\phi^p-\|g\|_\phi^p\right|, \quad u \in [0,1].
\end{equation*}
The main technical step is then to identify the sharp scale of this modulus. Proposition~\ref{prop:luxemburg-L1-modulus} shows that $\Omega_{\phi,p}(u) \asymp_{p} \psi_{\phi,p}^{\mathrm{cav}}(u)$. Consequently,
\begin{equation*}
\left|\|\widehat{h}_{\theta}\|^p_{\phi}-\|h_{\theta}\|^p_{\phi}\right| \lesssim_{p}\psi_{\phi,p}^{\mathrm{cav}}(\|\widehat{h}_{\theta}-h_{\theta}\|_{L^1(0,1)}).
\end{equation*}
The concavity of $\psi^{\mathrm{cav}}_{\phi,p}$ and Jensen's inequality yield
\begin{equation*}
\mathbb{E}\left[\psi^{\mathrm{cav}}_{\phi,p}\left(\|\widehat{h}_{\theta}-h_{\theta}\|_{L^1(0,1)}\right)\right] \le \psi^{\mathrm{cav}}_{\phi,p}\left(\mathbb{E}\|\widehat{h}_{\theta}-h_{\theta}\|_{L^1(0,1)}\right).
\end{equation*}
Using the monotonicity of $\psi_{\phi,p}^{\mathrm{cav}}$, the right-hand side then can be controlled by the quantile representation of $\operatorname{W}_1$ together with \cite[Theorem~3.2]{bobkov2019one}. Combining these estimates and integrating over the projection directions yields Theorem~\ref{thm:powerSOW-two-sample}
\subsection{Minimax lower bound of the powered SOW functional}
The minimax lower bounds are established using Le Cam's two-point method. For the $\operatorname{SOW}$ distance, a direct two-point perturbation at scale $n^{-1/2}$ already yields the matching lower bound, so we focus here on the additional difficulty arising for the powered functional. 

For the powered functional, reusing the two-point construction from Section~\ref{app:minimax-two-sample} only gives a separation of order $D^p\psi_{\phi,p}(n^{-1/2})$. However, the upper bound in Theorem~\ref{thm:powerSOW-two-sample} is governed by $D^p\psi^{\mathrm{cav}}_{\phi,p}\left(n^{-1/2}\right)$, which can be larger than $D^p\psi_{\phi,p}(n^{-1/2})$. To overcome this problem, Lemma~\ref{lem:parameter-powerSOW} shows that for every sufficiently small $h$, there exist two nearby mass levels $0 \le q_0<q_1\le 1/2$, $q_1-q_0 \le h$ such that $\psi_{\phi,p}(q_1)-\psi_{\phi,p}(q_0) \gtrsim_p \psi^{\mathrm{cav}}_{\phi,p}(h).$ This allows us to construct two distributions with small Kullback--Leibler divergence, while their powered $\operatorname{SOW}$ values are separated at the desired scale. Le Cam's method then gives the matching lower bound.

\bibliography{example_paper}
\bibliographystyle{abbrv}

\end{document}